\documentclass[11pt,letterpaper]{article}
\usepackage{graphicx}
\usepackage{epsfig}
\usepackage{amssymb}
\usepackage{amsthm}
\usepackage{amsfonts}
\usepackage{amsmath}
\usepackage[margin=1.4in]{geometry}
\RequirePackage{natbib}

\usepackage{mkolar_definitions}
\usepackage{amsmath}
\usepackage{mathabx}
\usepackage{caption}
\usepackage{subcaption}

\let\tilde\widetilde
\let\bar\overline
\let\hat\widehat
\let\check\widecheck
\newcommand{\Omegabt}{\Omega^{*}}
\newcommand{\ia}{{\bar a}}
\newcommand{\aia}{{a,\ia}}
\newcommand{\iaa}{{\ia,a}}
\newcommand{\iaia}{{\ia,\ia}}
\newcommand{\nab}{\neg{ab}}

\usepackage{color}

\title{\huge Graph Estimation From Multi-attribute Data}
\date{}

\author{
Mladen Kolar\thanks{Machine Learning Department, Carnegie Mellon University, Pittsburgh, PA 15217, USA; e-mail: {\tt
mladenk@cs.cmu.edu}.},~~Han Liu\thanks{Department of Operations Research and Financial Engineering, Princeton University, Princeton, NJ 08544, USA; e-mail: {\tt
hanliu@princeton.edu} }~~and~~Eric P. Xing\thanks{Machine Learning Department, Carnegie Mellon University, Pittsburgh, PA 15217, USA; e-mail: {\tt
epxing@cs.cmu.edu}.}
}

\bibpunct{(}{)}{,}{a}{,}{,}

\begin{document}

\maketitle

\begin{abstract}
  Many real world network problems often concern multivariate nodal
  attributes such as image, textual, and multi-view feature vectors on
  nodes, rather than simple univariate nodal attributes. The existing
  graph estimation methods built on Gaussian graphical models and
  covariance selection algorithms can not handle such data, neither
  can the theories developed around such methods be directly
  applied. In this paper, we propose a new principled framework for
  estimating graphs from multi-attribute data. Instead of estimating
  the partial correlation as in current literature, our method
  estimates the partial canonical correlations that naturally
  accommodate complex nodal features.  Computationally, we provide an
  efficient algorithm which utilizes the multi-attribute
  structure. Theoretically, we provide sufficient conditions which
  guarantee consistent graph recovery. Extensive simulation studies
  demonstrate performance of our method under various
  conditions. Furthermore, we provide illustrative applications to
  uncovering gene regulatory networks from gene and protein profiles,
  and uncovering brain connectivity graph from functional magnetic
  resonance imaging data.
\end{abstract}

\textbf{Keywords}: Graphical model selection; Multi-attribute data; Network analysis;
  Partial canonical correlation.

\section{Introduction}

In many modern problems, we are interested in studying a network of
entities with multiple attributes rather than a simple univariate
attribute. For example, when an entity represents a person in a social
network, it is widely accepted that the nodal attribute is most
naturally a vector with many personal information including
demographics, interests, and other features, rather than merely a
single attribute, such as a binary vote as assumed in the current
literature of social graph estimation based on Markov random fields
\citep{Banerjee08,kolar08estimating}. In another example, when an
entity represents a gene in a gene regulation network, modern data
acquisition technologies allow researchers to measure the activities
of a single gene in a high-dimensional space, such as an image of the
spatial distribution of the gene expression, or a multi-view snapshot
of the gene activity such as mRNA and protein abundances, rather than
merely a single attribute such as an expression level, which is
assumed in the current literature on gene graph estimation based on
Gaussian graphical models~\citep{peng09partial}. Indeed, it is
somewhat surprising that existing research on graph estimation remains
largely blinded to the analysis of multi-attribute data that are
prevalent and widely studied in the network community. Existing
algorithms and theoretical analysis relies heavily on covariance
selection using graphical lasso, or penalized pseudo-likelihood. They
can not be easily extended to graphs with multi-variate nodal
attributes.

In this paper, we present a study on graph estimation from
multi-attribute data, in an attempt to fill the gap between the
practical needs and existing methodologies from the literature. Under
a Gaussian graphical model, one assumes that a $p$-dimensional random
vector $X \in \RR^p$ follows a multivariate Gaussian distribution with
the mean $\mu$ and covariance matrix $\Sigma$, with each component of
the vector corresponding to a node in the graph.  Based on $n$
independent and identically distributed observations, one can estimate
an undirected graph $G=(V,E)$, where the node set $V$ corresponds to
the $p$ variables, and the edge set $E$ describes the conditional
independence relationships among the variables, that is, variables
$X_a$ and $X_b$ are conditionally independent given all the remaining
variables if $(a,b)\notin E$. Given multi-attribute data, this
approach is clearly invalid, because it naively translates to
estimating one graph per attribute. A subsequent integration of all
such graphs to a summary graph on the entire dataset may lead to
unclear statistical interpretation.

We consider the following new setting for estimating a multi-attribute
graph.  In this setting, we consider a ``stacked'' long random vector
$X=(X^T_1,...,X^T_p)^T$ where $X_1 \in \RR^{k_1}, \ldots, X_p
\in \RR^{k_p}$ are themselves random vectors that jointly follow a
multivariate Gaussian distribution with mean $\mu = (\mu_1, \ldots,
\mu_p)^T$ and covariance matrix $\Sigma^*$, which is partitioned as 
\begin{equation}
  \label{eq:vector_model}
\Sigma^* =   \left(
    \begin{array}{cccc}
     \Sigma_{11}^*  & \cdots & \Sigma_{1p}^* \\
     \vdots &   \ddots & \vdots \\
     \Sigma_{p1}^* & \cdots   & \Sigma_{pp}^* \\
    \end{array}
  \right),
\end{equation}
with $\Sigma^*_{ij} = \Cov(X_i, X_j)$.  Without loss of generality, we
assume $\mu=0$.  Let $G = (V,E)$ be a graph with the vertex set $V =
\{1,\ldots, p\}$ and the set of edges $E \subseteq V \times V$ that
encodes the conditional independence relationships among $(X_a)_{a\in
  V}$. That is, each node $a \in V$ of the graph $G$ corresponds to
the random vector $X_a$ and there is no edge between nodes $a$ and $b$
in the graph if and only if $X_a$ is conditionally independent of
$X_b$ given all the vectors corresponding to the remaining nodes,
$X_{\nab} = \{X_c\ :\ c \in V\bks\{a,b\}\}$. Such a graph is also
known as a Markov network (of Markov graph), which we shall emphasize
in this paper to compare with an alternative graph over $V$ known as
the association network, which is based on pairwise marginal
independence. Conditional independence can be read from the inverse of
the covariance matrix, as the block corresponding to $X_a$ and $X_b$
will be equal to zero. Let $\Dcal_n = \left\{x_i\right\}_{i=1}^n$ be a
sample of $n$ independent and identically distributed vectors drawn
from $\mathrm{N}(0, \Sigma)$. For a vector $x_i$, we denote $x_{i,a}
\in \RR^{k_a}$ the component corresponding to the node $a \in V$. Our
goal is to estimate the structure of the graph $G$ from the sample
$\Dcal_n$. Note that we allow for different nodes to have different
number of attributes, which is useful in many applications, e.g., when
a node represents a gene pathway in a regulatory network.

Using the standard Gaussian graphical model for univariate nodal
observations, one can estimate the Markov graph for each attribute
individually by estimating the sparsity pattern of the precision
matrix $\Omega = \Sigma^{-1}$.  This is also known as the covariance
selection problem \citep{Dempster72}. For high dimensional problems,
\citet{Meinshausen06} propose a parallel Lasso approach for estimating
Gaussian graphical models by solving a collection of sparse regression
problems. This procedure can be viewed as a pseudo-likelihood based
method. In contrast, \citet{Banerjee08}, \citet{Yuan07}, and
\citet{Friedman08} take a penalized likelihood approach to estimate
the sparse precision matrix $\Omega$. To reduce estimation bias,
\citet{Lam09}, \citet{Jalali12}, and \citet{Shen12} developed the
non-concave penalties to penalize the likelihood function. More
recently, \citet{Yuan10} and \citet{Cai11a} proposed the graphical
Dantzig selector and CLIME, which can be solved by linear programming
and are more amenable to theoretical analysis than the penalized
likelihood approach. Under certain regularity conditions, these
methods have proven to be graph estimation consistent
\citep{Ravikumar11,Yuan10,Cai11a} and scalable software packages, such
as $\mathrm{glasso}$ and $\mathrm{huge}$, were developed to implement
these algorithms \citep{Zhao12b}. However, in the case of
multi-attribute data, it is not clear how to combine estimated graphs
to obtain a single Markov graph reflecting the structure of the
underlying complex system. This is especially the case when nodes in
the graph contain different number of attributes.

In a previous work, \citet{katenka2011multi} proposed a method for
estimating association networks from multi-attribute data using
canonical correlation as a dependence measure between two groups of
attributes. However, association networks are known to confound the
direct interactions with indirect ones as they only represent marginal
associations.  In contrast, we develop a method based on partial
canonical correlation, which give rise to a Markov graph that is
better suited for separating direct interactions from indirect
confounders.  Our work is related to the literature on simultaneous
estimation of multiple Gaussian graphical models under a multi-task
setting
\citep{Guo:09,varoquaux2010brain,honorio2010multi,Chiquet2011inferring,danaher2011joint}.
However, the model given in \eqref{eq:vector_model} is different from
models considered in various multi-task settings and the optimization
algorithms developed in the multi-task literature do not extend to
handle the optimization problem given in our setting.

Unlike the standard procedures for estimating the structure of
Gaussian graphical models (e.g., neighborhood selection
\citep{Meinshausen06} or glasso \citep{Friedman08}), which infer the
partial correlations between pairs of multi-attribute nodes, our
proposed method estimates the {partial canonical correlations} between
pairs of nodes, which leads to a graph estimator over multi-attribute
nodes that bears the same probabilistic independence interpretations
as that of the graph from Gaussian graphical model over univariate
nodes. Under this new framework, the contributions of this paper
include: (i) Computationally, an efficient algorithm is provided to
estimate the multi-attribute Markov graphs; (ii) Theoretically, we
provide sufficient conditions which guarantee consistent graph
recovery; and (iii) Empirically, we apply our procedure to uncover
gene regulatory networks from gene and protein profiles, and to
uncover brain connectivity graph from functional magnetic resonance
imaging data. 

\section{Methodology}

In this section, we propose to estimate the graph by estimating
non-zero partial canonical correlation between the nodes. This leads
to a penalized maximum likelihood objective, for which we develop an
efficient optimization procedure. 

\subsection{Preliminaries}
\label{sec:prelim}

Let $X_a$ and $X_b$ be two multivariate random vectors.  Canonical
correlation is defined between $X_a$ and $X_b$ as
\begin{equation*}
\rho_c(X_a, X_b) = \max_{\ub \in \RR^{k_a}, \vb\in\RR^{k_b}} \Cor(\ub^TX_a, \vb^TX_b).
\end{equation*}
That is, computing canonical correlation between $X_a$ and $X_b$ is
equivalent to maximizing the correlation between two linear
combinations $\ub^TX_a$ and $\vb^TX_b$ with respect to vectors $\ub$
and $\vb$. Canonical correlation can be used to measure association
strength between two nodes with multi-attribute observations. For
example, in \citet{katenka2011multi}, a graph is estimated from
multi-attribute nodal observations by elementwise thresholding the
canonical correlation matrix between nodes, but such a graph estimator
may confound the direct interactions with indirect ones.

In this paper, we exploit the partial canonical correlation to
estimate a graph from multi-attribute nodal observations. A graph
is going to be formed by connecting nodes with non-zero partial
canonical correlation. Let $\hat\Ab = \argmin\ \EE\rbr{\norm{X_a -
  \Ab X_{\nab}}_2^2}$ and $\hat\Bb = \argmin\ \EE\rbr{\norm{X_b -
  \Bb X_{\nab}}_2^2}$, then the partial canonical correlation
between $X_a$ and $X_b$ is defined as 
\begin{equation}
\label{eq:partial_correlation}
  \begin{aligned}
\rho_c(X_a, X_b; X_{\nab}) 
= \max_{\ub\in\RR^{k_a}, \vb\in\RR^{k_b}} \Cor\{\ub^T(X_a
- \hat\Ab X_{\nab}), \vb^T(X_b - \hat\Bb X_{\nab})\},    
  \end{aligned}
\end{equation}
 that is, the partial canonical correlation between $X_a$ and $X_b$
is equal to the canonical correlation between the residual vectors of
$X_a$ and $X_b$ after the effect of $X_{\nab}$ is
removed \citep{rao_partial_1969}.

Let $\Omega^*$ denote the precision matrix under the model in
\eqref{eq:vector_model}. Using standard results for the multivariate
Gaussian distribution (see also Equation (7) in
\citet{rao_partial_1969}), a straightforward calculation shows
that\footnote{Calculation given in Appendix~\ref{sec:appendix_eq_3}}
\begin{equation}
  \label{eq:comp_cov}
  \rho_c(X_a, X_b; X_{\nab}) \neq 0
  \quad 
  \text{if and only if}
  \quad
   \max_{\ub\in\RR^{k_a}, \vb\in\RR^{k_b}}\ \ub^T\Omega_{ab}^*\vb \neq 0.
\end{equation}
This implies that estimating whether the partial canonical correlation
is zero or not can be done by estimating whether a block of the
precision matrix is zero or not. Furthermore, under the model in
\eqref{eq:vector_model}, vectors $X_a$ and $X_b$ are conditionally
independent given $X_{\nab}$ if and only if partial canonical
correlation is zero. A graph built on this type of inter-nodal
relationship is known as a {Markov graph}, as it captures both
local and global Markov properties over all arbitrary subsets of nodes
in the graph, even though the graph is built based on pairwise
conditional independence properties. In
$\S$\ref{sec:estimation}, we use the above observations to design
an algorithm that estimates the non-zero partial canonical correlation
between nodes from data $\Dcal_n$ using the penalized maximum
likelihood estimation of the precision matrix.

Based on the relationship given in   \eqref{eq:comp_cov}, we can
motivate an alternative method for estimating the non-zero partial
canonical correlation. Let $\bar{a} = \{b\ :\ b\in V\bks\{a\}\}$
denote the set of all nodes minus the node $a$. Then
\begin{equation*}
  \label{eq:population_block_regression}
  \EE\rbr{X_a\mid X_\ia = x_\ia} =
  \Sigma_\aia^*\Sigma_\iaia^{*,-1}
  x_\ia.
\end{equation*}
Since $\Omegabt_\aia =
-(\Sigma_{aa}^*-\Sigma_\aia^*\Sigma_\iaia^{*,-1}\Sigma_\iaa^*)^{-1}
\Sigma_\aia^*\Sigma_\iaia^{*,-1}$, we observe that a zero block
$\Omega_{ab}$ can be identified from the regression coefficients when
each component of $X_a$ is regressed on $X_\ia$. We do not build
an estimation procedure around this observation, however, we note that
this relationship shows how one would develop a regression
based analogue of the work presented in \citet{katenka2011multi}. 

\subsection{Penalized Log-Likelihood Optimization}
\label{sec:estimation}

Based on the data $\Dcal_n$, we propose to minimize the penalized
negative Gaussian log-likelihood under the model in
\eqref{eq:vector_model},
\begin{equation}
  \label{eq:max_ll_opt}
  \min_{\Omegab \succ \zero}\ \Bigl\{\tr \Sbb \Omegab - \log|\Omegab|
    + \lambda \sum_{a, b} \norm{\Omegab_{ab}}_F\Bigr\}
\end{equation}
where $\Sbb = n^{-1}\sum_{i=1}^n\xb_i\xb_i^T$ is the sample covariance
matrix, $\norm{\Omegab_{ab}}_F)$ denotes the Frobenius norm of
$\Omegab_{ab}$ and $\lambda$ is a user defined parameter that controls
the sparsity of the solution $\hat\Omega$. The Frobenius norm penalty
encourages blocks of the precision matrix to be equal to zero, similar
to the way that the $\ell_2$ penalty is used in the group Lasso
\citep{Yuan06model}.  Here we assume that the same number of samples
is available per attribute. However, the same method can be used in
cases when some samples are obtained on a subset of
attributes. Indeed, we can simply estimate each element of the matrix
$\Sbb$ from available samples, treating non-measured attributes as
missing completely at random \citep[for more details 
see][]{kolar12missing}.

The dual problem to \eqref{eq:max_ll_opt} is
\begin{equation}
  \label{eq:max_ll_opt_dual}
    \max_{\Sigmab} \sum_{j \in V} k_j + \log|\Sigmab| \qquad
      \text{subject to } \qquad
      \max_{a,b}\ \norm{\Sbb_{ab} - \Sigmab_{ab}}_F \leq \lambda,
\end{equation}
where $\Sigmab$ is the dual variable to $\Omegab$ and $|\Sigma|$
denotes the determinant of $\Sigma$. Note that the primal problem
gives us an estimate of the precision matrix, while the dual problem
estimates the covariance matrix. The proposed optimization procedure,
described below, will simultaneously estimate the precision matrix and
covariance matrix, without explicitly performing an expensive matrix
inversion.

We propose to optimize the objective function in \eqref{eq:max_ll_opt}
using an inexact block coordinate descent procedure, inspired by
\citet{mazumder11flexible}. The block coordinate descent is an
iterative procedure that operates on a block of rows and columns while
keeping the other rows and columns fixed. We write
\begin{equation*}
  \Omegab = \left(
  \begin{array}{cc}
    \Omegab_{aa} & \Omegab_\aia \\
    \Omegab_\iaa & \Omegab_\iaia\\
  \end{array}
  \right),
\quad
  \Sigmab = \left(
  \begin{array}{cc}
    \Sigmab_{aa} & \Sigmab_\aia \\
    \Sigmab_\iaa & \Sigmab_\iaia\\
  \end{array}
  \right),
\quad
  \Sbb = \left(
  \begin{array}{cc}
    \Sbb_{aa} & \Sbb_\aia \\
    \Sbb_\iaa & \Sbb_\iaia\\
  \end{array}
  \right)
\end{equation*}
and suppose that $(\tilde\Omegab, \tilde\Sigmab)$ are the current
estimates of the precision matrix and covariance matrix. With the
above block partition, we have $\log|\Omegab| =
\log(\Omegab_\iaia) + \log(\Omegab_{aa} -
\Omegab_\aia(\Omegab_\iaia)^{-1}\Omegab_\iaa)$.  In the next
iteration, $\hat \Omegab$ is of the form
\begin{equation*}
  \hat\Omegab = \tilde\Omegab + \left(
  \begin{array}{cc}
    \Deltab_{aa} & \Deltab_\aia \\
    \Deltab_\iaa & \zero
  \end{array}
\right) =
\left(
  \begin{array}{cc}
    \hat\Omegab_{aa} & \hat\Omegab_\aia \\
    \hat\Omegab_\iaa & \tilde\Omegab_\iaia \\
  \end{array}
\right)
\end{equation*}
and is obtained by minimizing
\begin{equation}
  \label{eq:partial_min}
  \tr\Sbb_{aa}\Omegab_{aa}
\! + \! 2\tr\Sbb_\aia\Omegab_\iaa
  - \log|\Omegab_{aa} - \Omegab_\aia(\tilde\Omegab_\iaia)^{-1}\Omegab_\iaa|
  + \lambda\norm{\Omegab_{aa}}_F + 2\lambda\sum_{b \neq
    a}\norm{\Omegab_{ab}}_F.
\end{equation}
Exact minimization over the variables $\Omegab_{aa}$ and
$\Omegab_\aia$ at each iteration of the block coordinate descent
procedure can be computationally expensive. Therefore, we propose to
update $\Omegab_{aa}$ and $\Omegab_\aia$ using one generalized
gradient step update (see \citet{beck09fast}) in each iteration.  Note
that the objective function in \eqref{eq:partial_min} is a sum of a
smooth convex function and a non-smooth convex penalty so that the
gradient descent method cannot be directly applied. Given a step size
$t$, generalized gradient descent optimizes a quadratic approximation
of the objective at the current iterate $\tilde\Omegab$, which results
in the following two updates
\begin{align}
  \label{eq:update_aa}
  \hat\Omegab_{aa} &= \argmin_{\Omegab_{aa}} \Big\{
  \tr(\Sbb_{aa} - \tilde\Sigmab_{aa})\Omegab_{aa} +
  \frac{1}{2t}\norm{\Omegab_{aa} - \tilde\Omegab_{aa}}_F^2 +
  \lambda \norm{\Omegab_{aa}}_F\Bigr\},
  \quad\text{ and} \\
  \label{eq:update_ab}
  \hat\Omegab_{ab} &= \argmin_{\Omegab_{ab}}
  \Bigl\{ \tr(\Sbb_{ab} - \tilde\Sigmab_{ab})\Omegab_{ba} +
  \frac{1}{2t}\norm{\Omegab_{ab} - \tilde\Omegab_{ab}}_F^2 +
  \lambda \norm{\Omegab_{ab}}_F\Bigr\},
  \quad \forall b \in \ia.
\end{align}
If the resulting estimator $\hat\Omegab$ is not positive definite or
the update does not decrease the objective, we halve the step size $t$
and find a new update.  Once the update of the precision matrix
$\hat\Omegab$ is obtained, we update the covariance matrix
$\hat\Sigmab$.  Updates to the precision and covariance matrices can
be found efficiently, without performing expensive matrix inversion,
as we show in Appendix~\ref{appendix:a}
(see~\eqref{eq:update_aa_final}--\eqref{eq:update_sigma}).  Combining
all three steps we get the following algorithm:
\begin{enumerate}
\item
Set the initial estimator $\tilde\Omegab = \diag(\Sbb)$ and
$\tilde\Sigmab = \tilde\Omegab^{-1}$. Set the step size $t=1$.
\item
For each $a \in V$ perform the following:
\begin{tabbing}
\qquad Update $\hat\Omegab$ using 
\eqref{eq:update_aa_final} and \eqref{eq:update_ab_final}.\\
\qquad If $\hat\Omegab$ is not positive definite,
  set $t \leftarrow t/2$ and repeat the update.\\
\qquad Update $\hat\Sigmab$ using \eqref{eq:update_sigma}.
\end{tabbing}
\item
  Repeat Step 2 until the duality gap 
  \[
 \Bigl| \tr (\Sbb \hat \Omegab) - \log|\hat\Omegab|
  + \lambda \sum_{a, b} \norm{\hat\Omegab_{ab}}_F -
  \sum_{j \in V} k_j - \log|\Sigmab|\Bigr| \leq \epsilon,  
  \]
  where $\epsilon$ is a prefixed  precision parameter (for example, $\epsilon=10^{-3}$).
\end{enumerate}
Finally, we form a graph $\hat G = (V, \hat E)$ by connecting nodes
with $\norm{\hat\Omegab_{ab}}_F \neq 0$.

Computational complexity of the procedure is given in
Appendix~\ref{appendix:b}.  Convergence of the above described
procedure to the unique minimum of the objective function in
\eqref{eq:max_ll_opt} does not follow from the standard results on the
block coordinate descent algorithm \citep{tseng01convergence} for two
reasons. First, the minimization problem in \eqref{eq:partial_min} is
not solved exactly at each iteration, since we only update
$\Omegab_{aa}$ and $\Omegab_\aia$ using one generalized gradient step
update in each iteration. Second, the blocks of variables, over which
the optimization is done at each iteration, are not completely
separable between iterations due to the symmetry of the problem. The
proof of the following convergence result is given in
Appendix~\ref{appendix:c}.

\begin{lemma}
  \label{lem:convergence}
  For every value of $\lambda > 0$, the above described algorithm
  produces a sequence of estimates $\cbr{\tilde\Omegab^{(t)}}_{t \geq
    1}$ of the precision matrix that monotonically decrease the
  objective values given in \eqref{eq:max_ll_opt}. Every element of
  this sequence is positive definite and the sequence converges to the
  unique minimizer $\hat \Omegab$ of \eqref{eq:max_ll_opt}.
\end{lemma}

\subsection{Efficient Identification of Connected Components}

When the target graph $\hat G$ is composed of smaller, disconnected
components, the solution to the problem in \eqref{eq:max_ll_opt} is
block diagonal (possibly after permuting the node indices) and can be
obtained by solving smaller optimization problems. That is, the
minimizer $\hat\Omegab$ can be obtained by solving
\eqref{eq:max_ll_opt} for each connected component independently,
resulting in massive computational gains. We give necessary and
sufficient condition for the solution $\hat\Omegab$ of
\eqref{eq:max_ll_opt} to be block-diagonal, which can be easily
checked by inspecting the empirical covariance matrix $\Sbb$.  

Our first result follows immediately from the Karush-Kuhn-Tucker
conditions for the optimization problem \eqref{eq:max_ll_opt} and
states that if $\hat\Omegab$ is block-diagonal, then it can be
obtained by solving a sequence of smaller optimization problems.
\begin{lemma}
  If the solution to \eqref{eq:max_ll_opt} takes the form 
  $\hat\Omega={\rm diag}(\hat\Omega_1,
  \hat\Omega_2,\ldots,\hat\Omega_l)$, that is, $\hat\Omega$ is a block
  diagonal matrix with the diagonal blocks $\hat\Omega_1, \ldots,
  \hat\Omega_l$, then it can be obtained by solving
  \begin{equation*}
  \min_{\Omegab_{l'} \succ \zero}\ \Bigl\{ \tr \Sbb_{l'} \Omegab_{l'} - \log|\Omegab_{l'}|
    + \lambda \sum_{a, b} \norm{\Omegab_{ab}}_F\Bigr\}
  \end{equation*}
  separately for each $l'=1,\ldots,l$, where $\Sbb_{l'}$ are
  submatrices of $\Sbb$ corresponding to $\Omegab_{l'}$.
\end{lemma}
Next, we describe how to identify diagonal blocks of
$\hat\Omegab$. Let $\Pcal = \{P_1, P_2, \ldots, P_l\}$ be a partition
of the set $V$ and assume that the nodes of the graph are ordered in
a way that if $a \in P_j$, $b \in P_{j'}$, $j < j'$, then $a < b$.
The following lemma states that the blocks of $\hat\Omega$ can be
obtained from the blocks of a thresholded sample covariance matrix.
\begin{lemma}
  \label{lem:nec_suff_block_diag}
  A necessary and sufficient conditions for $\hat\Omegab$ to be block
  diagonal with blocks $P_1, P_2, \ldots, P_l$ is that
  $\norm{\Sbb_{ab}}_F \leq \lambda$ for all $a \in P_j$, $b \in
  P_{j'}$, $j \neq j'$.
\end{lemma}

Blocks $P_1, P_2, \ldots, P_l$ can be identified by forming a $p\times
p$ matrix $\Qb$ with elements $q_{ab} = \ind\{\norm{\Sbb_{ab}}_F >
\lambda\}$ and computing connected components of the graph with
adjacency matrix $\Qb$. The lemma states also that given two penalty
parameters $\lambda_1 < \lambda_2$, the set of unconnected nodes with
penalty parameter $\lambda_1$ is a subset of unconnected nodes with
penalty parameter $\lambda_2$. The simple check above allows us to
estimate graphs on datasets with large number of nodes, if we are
interested in graphs with small number of edges. However, this is
often the case when the graphs are used for exploration and
interpretation of complex systems.
Lemma~\ref{lem:nec_suff_block_diag} is related to existing results
established for speeding-up computation when learning single and
multiple Gaussian graphical models \citep{witten2011new,
  mazumder2011exact, danaher2011joint}. Each condition is different,
since the methods optimize different objective functions.

\section{Consistent Graph Identification}
\label{sec:theory}

In this section, we provide theoretical analysis of the estimator
described in $\S$\ref{sec:estimation}. In particular, we provide
sufficient conditions for consistent graph recovery.  For simplicity
of presentation, we assume that $k_a = k$, for all $a \in V$, that is,
we assume that the same number of attributes is observed for each
node. For each $a=1,\ldots,kp$, we assume that
$(\sigma_{aa}^*)^{-1/2}X_a$ is sub-Gaussian with parameter $\gamma$,
where $\sigma_{aa}^*$ is the $a$th diagonal element of
$\Sigmab^*$. Recall that $Z$ is a sub-Gaussian random variable if
there exists a constant $\sigma \in (0, \infty)$ such that
\begin{equation*}
  \EE\rbr{\exp(tZ)} \leq \exp(\sigma^2t^2),\ \text{for all
  } t \in \RR.
\end{equation*}

Our assumptions involve the Hessian of the function $f(\Ab) =
\tr\Sbb\Ab - \log|\Ab|$ evaluated at the true $\Omegab^*$, $\Hcal =
\Hcal(\Omegabt) = (\Omegabt)^{-1}\otimes(\Omegabt)^{-1} \in
\RR^{(pk)^2\times(pk)^2}$, with $\otimes$ denoting the Kronecker
product, and the true covariance matrix $\Sigmab^*$. The Hessian and
the covariance matrix can be thought of as block matrices with blocks
of size $k^2 \times k^2$ and $k \times k$, respectively. We will make
use of the operator $\Ccal(\cdot)$ that operates on these block
matrices and outputs a smaller matrix with elements that equal to the
Frobenius norm of the original blocks. For example, $\Ccal(\Sigmab^*)
\in \RR^{p\times p}$ with elements $\Ccal(\Sigmab^*)_{ab} =
\norm{\Sigmab^*_{ab}}_F$. Let $\Tcal = \{(a,b): \norm{\Omegab_{ab}}_F
\neq 0\}$ and $\Ncal = \{(a,b): \norm{\Omegab_{ab}}_F = 0\}$. With
this notation introduced, we assume that the following
{irrepresentable} condition holds. There exists a constant $\alpha \in
[0, 1)$ such that
\begin{equation}
\label{eq:assum:irrepresentable}
\opnorm{
\Ccal\left(
  \Hcal_{\Ncal\Tcal}(\Hcal_{\Tcal\Tcal})^{-1}
\right)}{\infty} \leq 1 - \alpha,
\end{equation}
where $\opnorm{A}{\infty} = \max_i \sum_j|A_{ij}|$.
We will also need the following quantities to specify the results $
\kappa_{\Sigmab^*} = \opnorm{\Ccal(\Sigmab^*)}{\infty}$ and
$\kappa_{\Hcal} =
\opnorm{\Ccal(\Hcal_{\Tcal\Tcal}^{-1})}{\infty}$. These conditions
extend the conditions specified in \citet{Ravikumar11} needed for
estimating graphs from single attribute observations. 

We have the following result that provides sufficient conditions for
the exact recovery of the graph.
\begin{proposition}
\label{prop:sparsistency}
Let $\tau > 2$. We set
the penalty parameter $\lambda$ in \eqref{eq:max_ll_opt} as 
\[
\lambda = 8k\alpha^{-1}
\rbr{128(1+4\gamma^2)^2(\max_a(\sigma_{aa}^*)^2)n^{-1}(2\log(2k) +
  \tau\log(p))}^{1/2}.
\]
If $
  n > C_1s^2k^2(1+8\alpha^{-1})^2(\tau\log p + \log 4 + 2\log k),$
where $s$ is the maximal degree of nodes in $G$, 
$
C_1 =
(48\sqrt{2}(1+4\gamma^2)(\max_a \sigma_{aa}^*)
\max(\kappa_{\Sigmab^*}\kappa_{\Hcal},
\kappa_{\Sigmab^*}^3\kappa_{\Hcal}^2))^2
$
 and
\[
\min_{(a,b)\in\Tcal, a\neq b} \norm{\Omegab_{ab}}_F > 
16\sqrt{2}(1+4\gamma^2)(\max_a \sigma_{aa}^*)
          (1+8\alpha^{-1})\kappa_\Hcal k
          \rbr{\frac{\tau\log p + \log 4 + 2\log k}{n}}^{1/2},
\]
then $ \PP\rbr{\hat G = G} \geq 1 - p^{2-\tau}$.
\end{proposition}

The proof of Proposition~\ref{prop:sparsistency} is given in
Appendix~\ref{appendix:c}. We extend the proof of \citet{Ravikumar11}
to accommodate the Frobenius norm penalty on blocks of the precision
matrix. This proposition specifies the sufficient sample size and a
lower bound on the Frobenius norm of the off-diagonal blocks needed
for recovery of the unknown graph.  Under these conditions and
correctly specified tuning parameter $\lambda$, the solution to the
optimization problem in \eqref{eq:max_ll_opt} correctly recovers the
graph with high probability. In practice, one needs to choose the
tuning parameter in a data dependent way. For example, using the
Bayesian information criterion.  Even though our theoretical analysis
obtains the same rate of convergence as that of \citet{Ravikumar11},
our method has a significantly improved finite-sample performance
(More details will be provided in $\S$\ref{sec:simulation}.). It
remains an open question whether the sample size requirement can be
improved as in the case of group Lasso \citep[see, for
example,][]{lounici10oracle}.  The analysis of \cite{lounici10oracle}
relies heavily on the special structure of the least squares
regression. Hence, their method does not carry over to the more
complicated objective function as in \eqref{eq:max_ll_opt}.

\section{Interpreting Edges}
\label{sec:interpretation}

We propose a post-processing step that will allow us to quantify the
strength of links identified by the method proposed in
$\S$\ref{sec:estimation}, as well as identify important attributes that
contribute to the existence of links.

For any two nodes $a$ and $b$ for which $\Omegab_{ab}\neq0$, we define
$\Ncal(a,b) = \{ c \in V \bks \{a,b\}\ :\ \Omegab_{ac} \neq 0 \text{
  or } \Omegab_{bc} \neq 0 \}$, which is the Markov blanket for the set
of nodes $\{ \Xb_a, \Xb_b \}$. Note that the conditional distribution
of $(\Xb_a^T, \Xb_b^T)^T$ given $\Xb_{\nab}$ is equal to the
conditional distribution of $(\Xb_a^T, \Xb_b^T)^T$ given
$\Xb_{\Ncal(a,b)}$.  Now,
\begin{equation*}
  \begin{aligned}
\rho_c(\Xb_a, \Xb_b; \Xb_{\nab}) 
&= \rho_c(\Xb_a, \Xb_b; \Xb_{\Ncal(a,b)})\\
&= \max_{\wb_a\in\RR^{k_a}, \wb_b\in\RR^{k_b}} \Cor(\ub^T(\Xb_a
- \tilde\Ab\Xb_{\Ncal(a,b)}), \vb^T(\Xb_b - \tilde\Bb\Xb_{\Ncal(a,b)})),     
  \end{aligned}
\end{equation*}
where $\tilde\Ab = \argmin\ \EE\rbr{\norm{\Xb_a -
  \Ab\Xb_{\Ncal(a,b)}}_2^2}$ and $\tilde\Bb = \argmin\ \EE\rbr{\norm{\Xb_b
  - \Bb\Xb_{\Ncal(a,b)}}_2^2}$. Let $\bar\Sigmab(a,b) = \Var(\Xb_a,
\Xb_b\mid\Xb_{\Ncal(a,b)})$. Now we can express the partial canonical
correlation as
\[
\rho_c(\Xb_a, \Xb_b; \Xb_{\Ncal(a,b)}) = 
\max_{\wb_a\in\RR^{k_a}, \wb_b\in\RR^{k_a} }\ 
\frac{\wb_a^T\bar\Sigma_{ab}\wb_b}
{\rbr{\wb_a^T\bar\Sigma_{aa}\wb_a}^{1/2}\rbr{\wb_b^T\bar\Sigma_{bb}\wb_b}^{1/2}}
\]
where 
\[
\bar\Sigmab(a,b) = \left(
\begin{array}{cc}
\bar\Sigmab_{aa} & \bar\Sigmab_{ab} \\
\bar\Sigmab_{ba} & \bar\Sigmab_{bb}
\end{array}
\right).
\]
The weight vectors $\wb_a$ and $\wb_b$ can be easily found by solving 
the system of eigenvalue equations
\begin{equation}
  \label{eq:eigenvalue_system}
\left\{
\begin{array}{c}
  \bar\Sigmab_{aa}^{-1}\bar\Sigmab_{ab}\bar\Sigmab_{bb}^{-1}\bar\Sigmab_{ba}\wb_a
  = \phi^2\wb_a \\
  \bar\Sigmab_{bb}^{-1}\bar\Sigmab_{ba}\bar\Sigmab_{aa}^{-1}\bar\Sigmab_{ab}\wb_b
  = \phi^2\wb_b
\end{array}
\right.
\end{equation}
with $\wb_a$ and $\wb_b$ being the vectors that correspond to the
maximum eigenvalue $\phi^2$. Furthermore, we have $\rho_c(\Xb_a,
\Xb_b; \Xb_{\Ncal(a,b)}) = \phi$. Following \citet{katenka2011multi},
the weights $\wb_a$, $\wb_b$ can be used to access the relative
contribution of each attribute to the edge between the nodes $a$ and
$b$. In particular, the weight $(w_{a,i})^2$ characterizes the
relative contribution of the $i$th attribute of node $a$ to
$\rho_c(\Xb_a, \Xb_b; \Xb_{\Ncal(a,b)})$.

Given an estimate $\hat\Ncal(a, b) = \{ c \in V \bks \{a,b\}\ :\
\hat\Omegab_{ac} \neq 0 \text{ or } \hat\Omegab_{bc} \neq 0 \}$ of the
Markov blanket $\Ncal(a, b)$, we form the residual vectors
\[
\rb_{i, a} = \xb_{i,a} - \check{\Ab}\xb_{i, \hat \Ncal(a,b)},
\qquad
\rb_{i, b} = \xb_{i,b} - \check{\Bb}\xb_{i, \hat \Ncal(a,b)},
\]
where $\check{\Ab}$ and $\check{\Bb}$ are the least square estimators
of ${\tilde{\Ab}}$ and ${\tilde{\Bb}}$. Given the residuals, we form
$\check\Sigmab(a,b)$, the empirical version of the matrix $\bar
\Sigmab(a,b)$, by setting
\[
{\check \Sigma}_{aa} = {\Cor}\rbr{\{\rb_{i,a}\}_{i \in [n]}},\quad
{\check \Sigma}_{bb} = {\Cor}\rbr{\{\rb_{i,b}\}_{i \in [n]}},\quad
{\check \Sigma}_{ab} = {\Cor}\rbr{\{\rb_{i,a} \}_{i \in [n]}, \{\rb_{i,a} \}_{i \in [n]}}.
\]
Now, solving the eigenvalue system in \eqref{eq:eigenvalue_system}
will give us estimates of the vectors $\wb_a$, $\wb_b$ and the partial
canonical correlation.

Note that we have described a way to interpret the elements of the
off-diagonal blocks in the estimated precision matrix. The elements of
the diagonal blocks, which correspond to coefficients between
attributes of the same node, can still be interpreted by their
relationship to the partial correlation coefficients.

\section{Simulation Studies}
\label{sec:simulation}

In this section, we perform a set of simulation studies to illustrate
finite sample performance of our method. We demonstrate that the
scalings of $(n,p,s)$ predicted by the theory are sharp. Furthermore,
we compare against three other methods: 1) a method that uses
the ${\rm glasso}$ first to estimate one graph over each of the $k$
individual attributes and then creates an edge in the resulting graph
if an edge appears in at least one of the single attribute graphs, 2)
the method of \cite{Guo:09} and 3) the method of
\cite{danaher2011joint}.  We have also tried applying the ${\rm
  glasso}$ to estimate the precision matrix for the model in
\eqref{eq:vector_model} and then post-processing it, so that an edge
appears in the resulting graph if the corresponding block of the
estimated precision matrix is non-zero. The result of this method is
worse compared to the first baseline, so we do not report it here.

All the methods above require setting one or two tuning parameters
that control the sparsity of the estimated graph. We select these
tuning parameters by minimizing the Bayesian information criterion,
which balances the goodness of fit of the model and its complexity,
over a grid of parameter values. For our multi-attribute method, the
Bayesian information criterion takes the following form
\[
{\rm BIC}(\lambda) = \tr(\Sbb\hat\Omegab) - \log|\hat\Omegab| 
+ \sum_{a < b} \ind\{\hat\Omegab_{ab} \neq \zero\}k_ak_b\log(n).
\]
Other methods for selecting tuning parameters are possible, like
minimization of cross-validation or Akaike information criterion.
However, these methods tend to select models that are too dense.

Theoretical results given in $\S$\ref{sec:theory} characterize the
sample size needed for consistent recovery of the underlying graph.
In particular, Proposition~\ref{prop:sparsistency} suggests that we
need $n = \theta s^2k^2\log(pk)$ samples to estimate the graph
structure consistently, for some $\theta > 0$. Therefore, if we plot
the hamming distance between the true and recovered graph
against $\theta$, we expect the curves to reach zero distance for
different problem sizes at a same point. We verify this on randomly
generated chain and nearest-neighbors graphs.

\begin{figure}[p]
  \centering

  \begin{subfigure}[b]{0.95\textwidth}
    \centering
    \includegraphics[width=\textwidth]{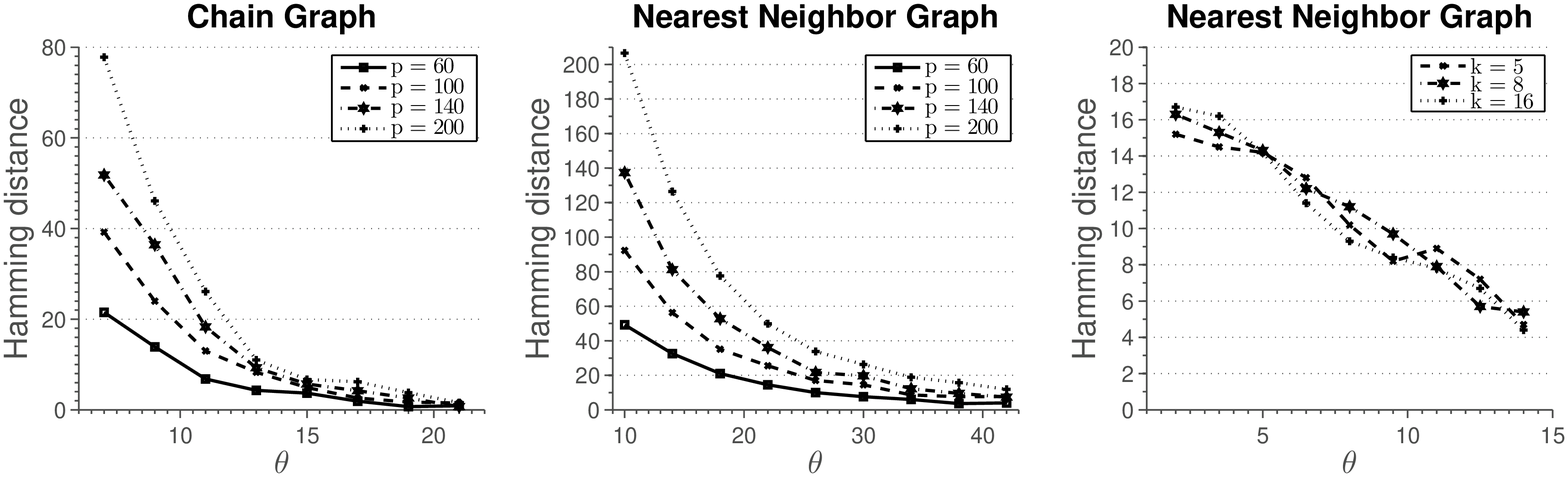}
    \caption{glasso procedure}
    \label{fig:normal:glasso}
  \end{subfigure}%

 \begin{subfigure}[b]{0.95\textwidth}
    \centering
    \includegraphics[width=\textwidth]{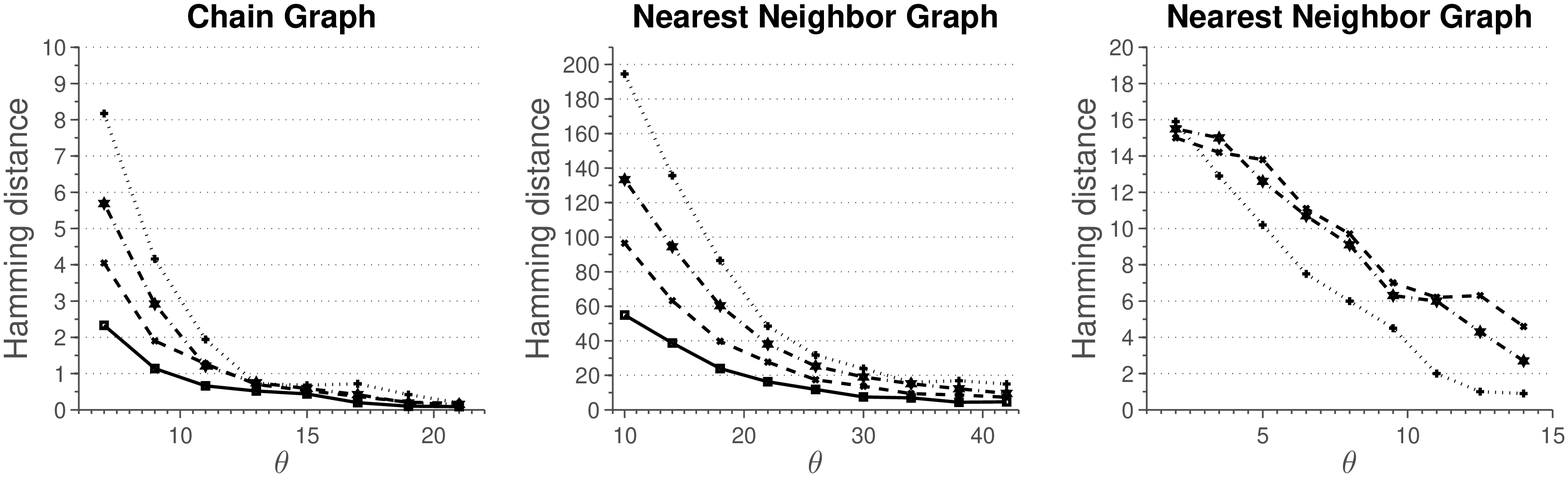}
    \caption{Procedure of \cite{danaher2011joint}}
    \label{fig:normal:mt}
  \end{subfigure}%

 \begin{subfigure}[b]{0.95\textwidth}
    \centering
    \includegraphics[width=\textwidth]{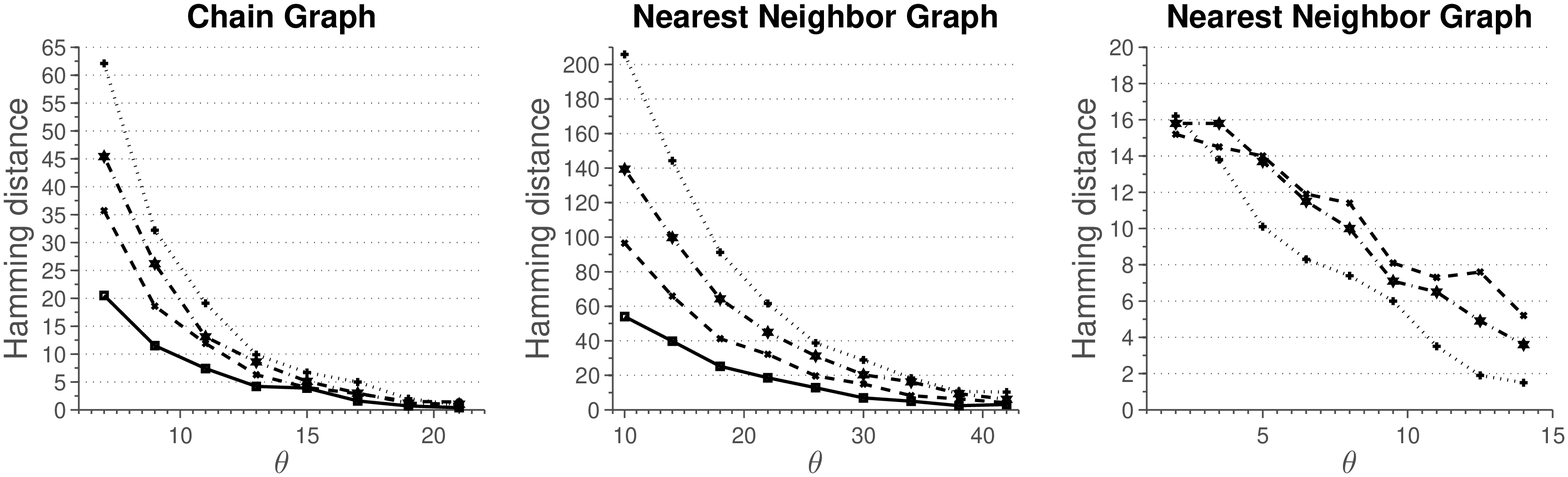}
    \caption{Procedure of \cite{Guo:09}}
    \label{fig:normal:guo}
  \end{subfigure}%

 \begin{subfigure}[b]{0.95\textwidth}
    \centering
    \includegraphics[width=\textwidth]{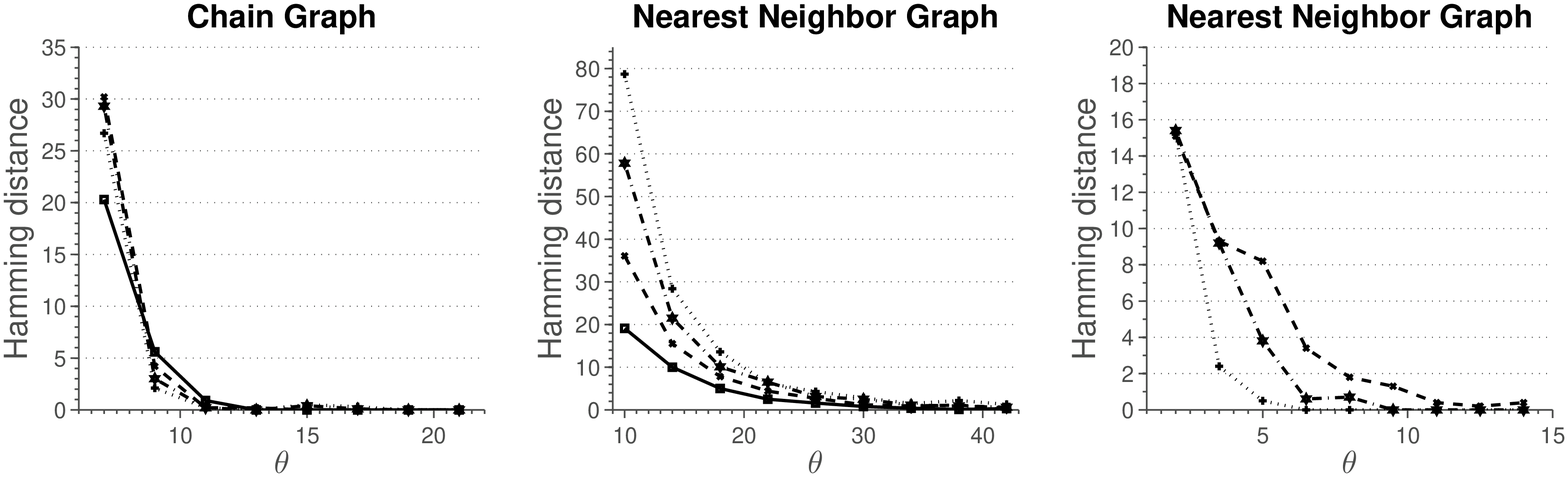}
    \caption{Multi-attribute procedure}
    \label{fig:normal:ma}
  \end{subfigure}%
  \caption{Average hamming distance plotted against the rescaled
    sample size. Results are averaged over 100 independent
    runs. Off-diagonal blocks are full matrices.}
  \label{fig:simulation:normal}
\end{figure}

We generate data as follows. A random graph with $p$ nodes is created
by first partitioning nodes into $p/20$ connected components, each
with $20$ nodes, and then forming a random graph over these $20$
nodes. A chain graph is formed by permuting the nodes and connecting
them in succession, while a nearest-neighbor graph is constructed
following the procedure outlined in \citet{li06gradient}. That is, for
each node, we draw a point uniformly at random on a unit square and
compute the pairwise distances between nodes. Each node is then
connected to $s=4$ closest neighbors.  Since some of nodes will have
more than $4$ adjacent edges, we  randomly remove edges from nodes that
have degree larger than $4$ until the maximum degree of a node in a
network is $4$. Once the graph is created, we construct a
precision matrix, with non-zero blocks corresponding to edges in the
graph. Elements of diagonal blocks are set as $0.5^{|a-b|}$, $0\leq
a,b\leq k$, while off-diagonal blocks have elements with the same
value, $0.2$ for chain graphs and $0.3/k$ for nearest-neighbor
networks. Finally, we add $\rho I$ to the precision matrix, so that
its minimum eigenvalue is equal to $0.5$.  Note that $s=2$ for the
chain graph and $s=4$ for the nearest-neighbor graph. Simulation
results are averaged over 100 replicates.

Figure~\ref{fig:simulation:normal} shows simulation results.  Each row
in the figure reports results for one method, while each column in the
figure represents a different simulation setting. For the first two
columns, we set $k=3$ and vary the total number of nodes in the
graph. The third simulation setting sets the total number of nodes
$p=20$ and changes the number of attributes $k$. In the case of the
chain graph, we observe that for small sample sizes the method of
\cite{danaher2011joint} outperforms all the other methods. We note
that the multi-attribute method is estimating many more parameters,
which require large sample size in order to achieve high
accuracy. However, as the sample size increases, we observe that
multi-attribute method starts to outperform the other methods. In
particular, for the sample size indexed by $\theta = 13$ all the graph
are correctly recovered, while other methods fail to recover the
graph consistently at the same sample size. In the case of
nearest-neighbor graph, none of the methods recover the graph well for
small sample sizes. However, for moderate sample sizes,
multi-attribute method outperforms the other methods.  Furthermore,
as the sample size increases none of the other methods recover the
graph exactly. This suggests that the conditions for consistent
graph recovery may be weaker in the multi-attribute setting.

\begin{figure}[p]
  \centering

  \begin{subfigure}[b]{0.9\textwidth}
    \centering
    \begin{subfigure}[b]{0.4\textwidth}
      \centering
      \includegraphics[width=\textwidth]{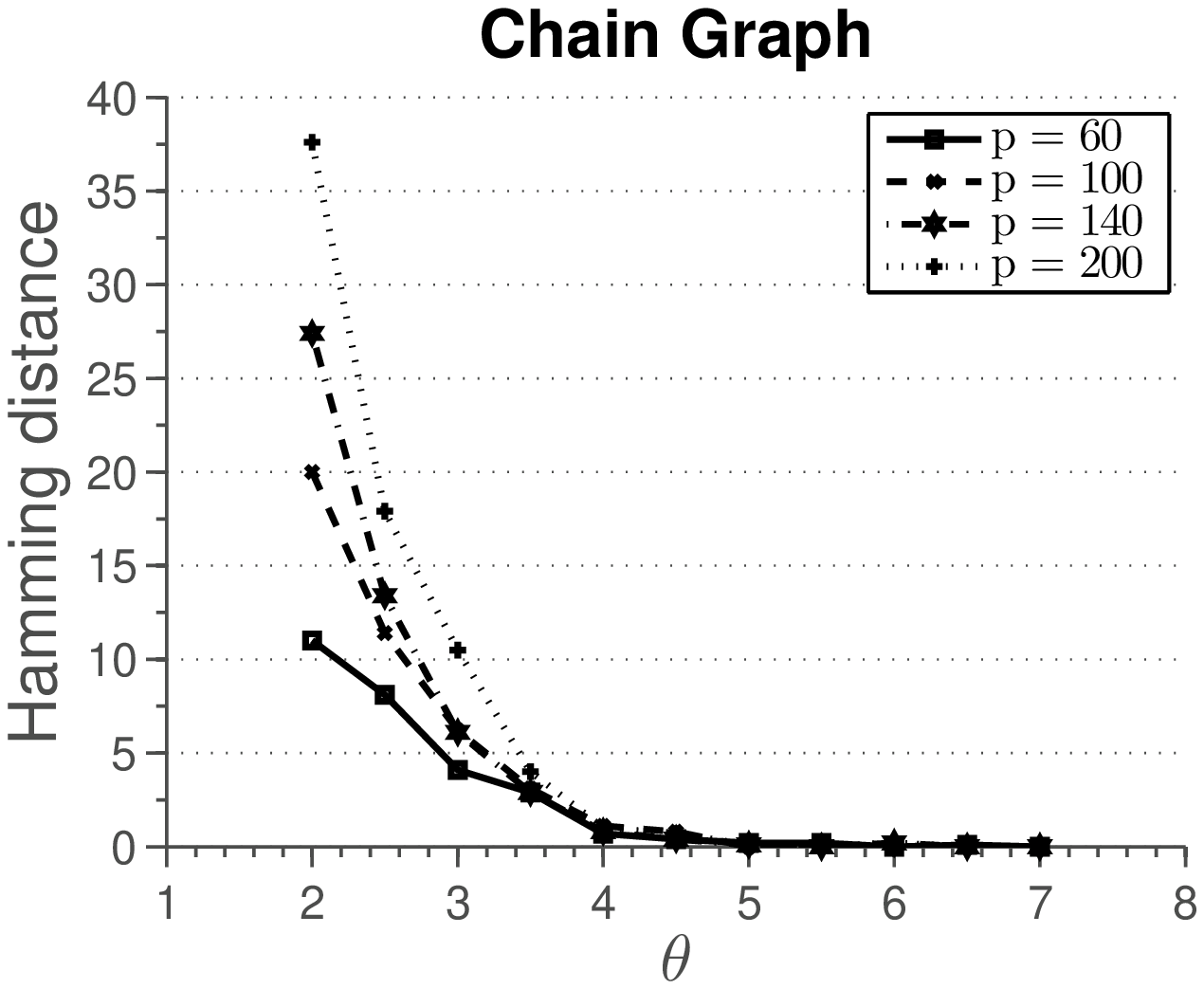}
    \end{subfigure}%
    \hfill
    \begin{subfigure}[b]{0.4\textwidth}
      \centering
      \includegraphics[width=\textwidth]{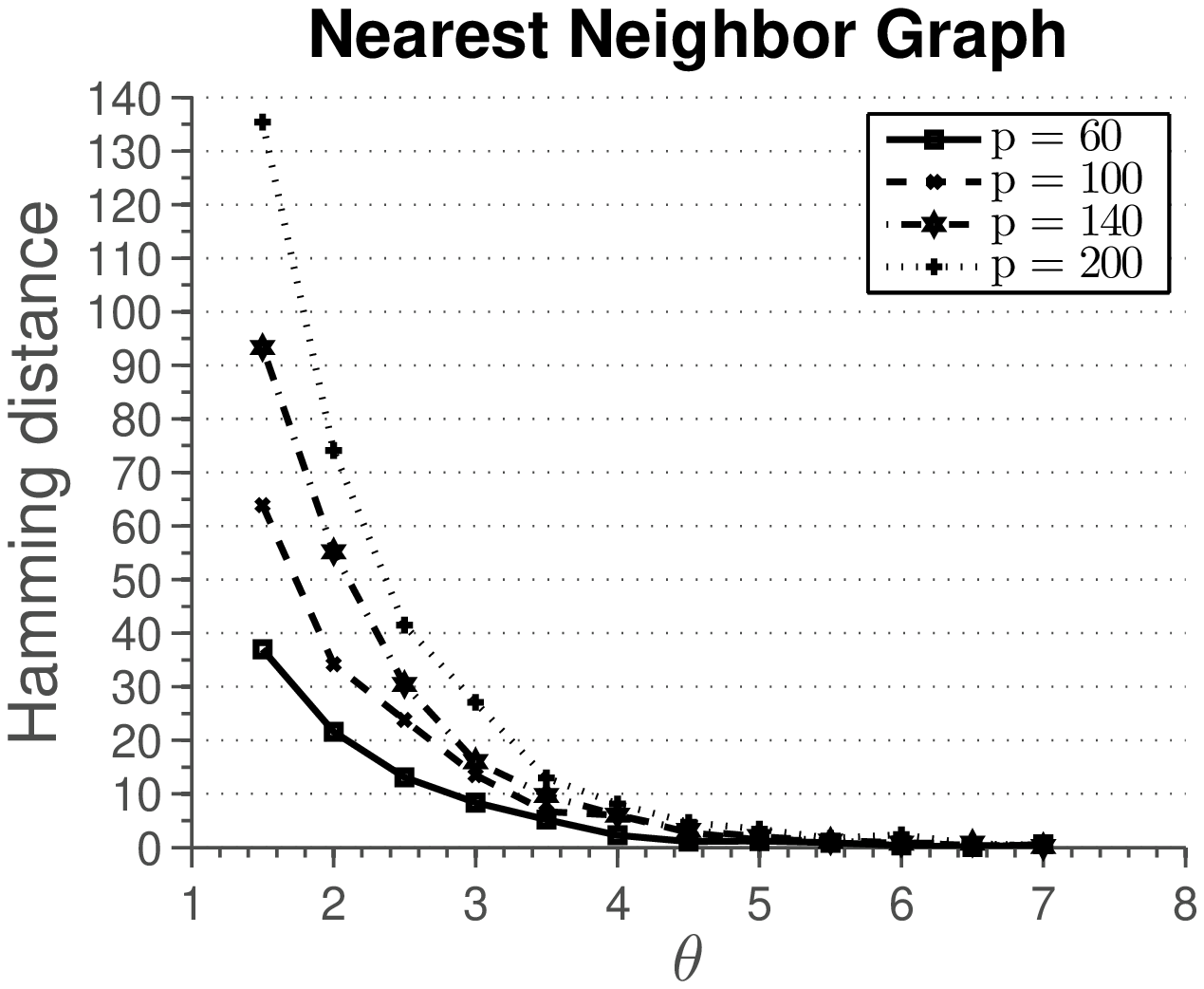}
    \end{subfigure}%
    \caption{glasso procedure}
    \label{fig:easy:glasso}
  \end{subfigure}%

 \begin{subfigure}[b]{0.9\textwidth}
    \centering
    \begin{subfigure}[b]{0.4\textwidth}
      \centering
      \includegraphics[width=\textwidth]{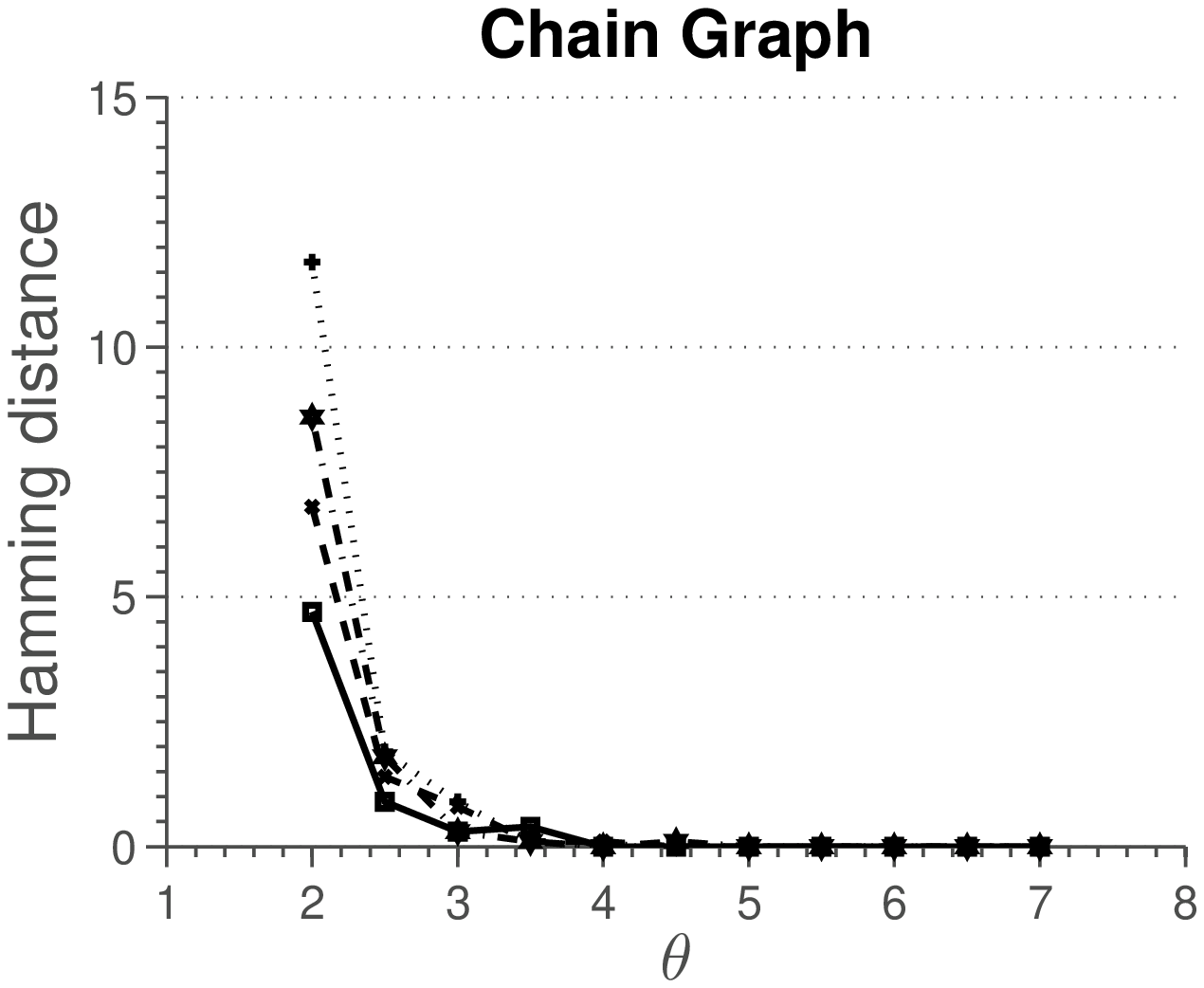}
    \end{subfigure}%
    \hfill
    \begin{subfigure}[b]{0.4\textwidth}
      \centering
      \includegraphics[width=\textwidth]{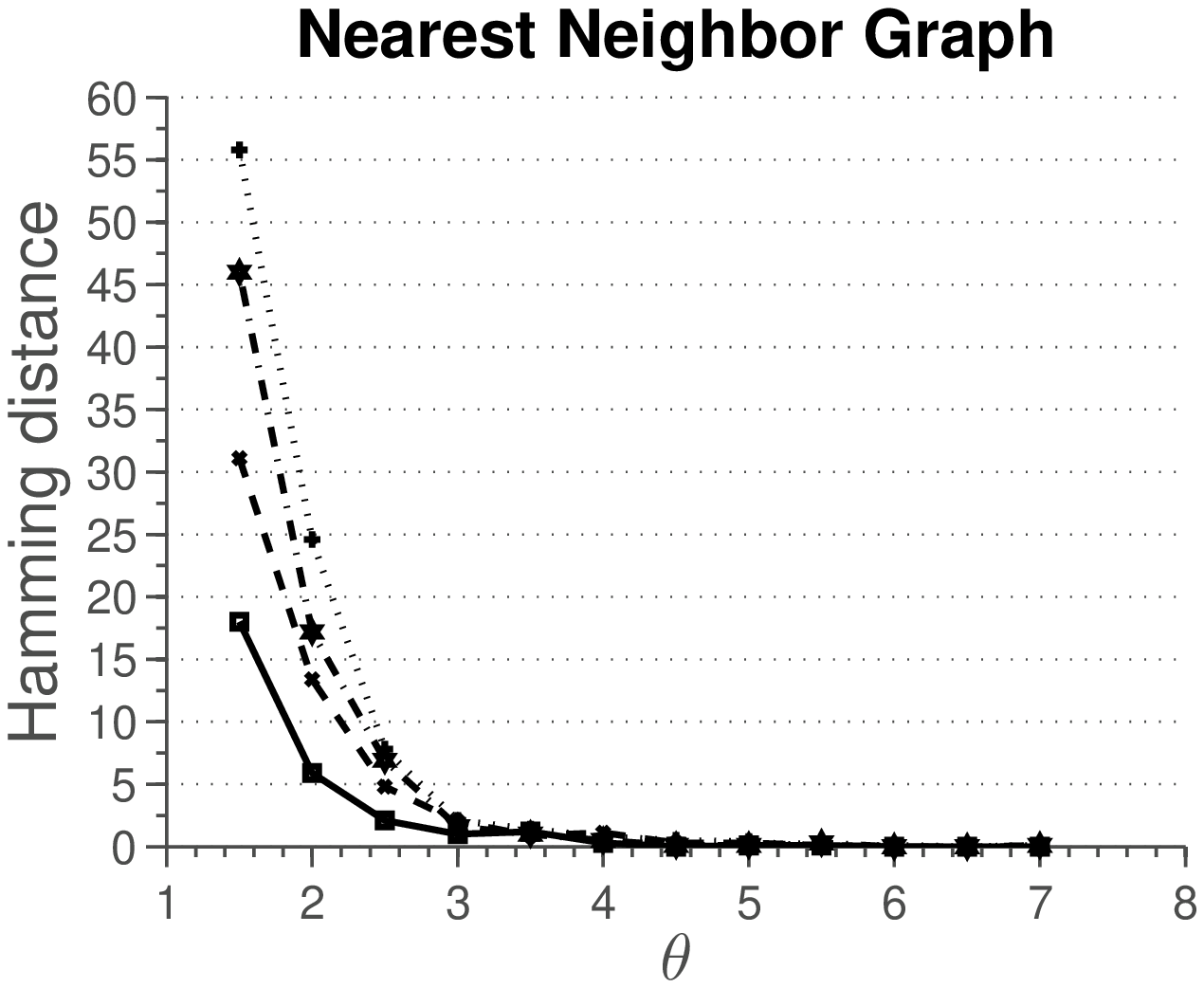}
    \end{subfigure}%
    \caption{Procedure of \cite{danaher2011joint}}
    \label{fig:easy:mt}
  \end{subfigure}%

 \begin{subfigure}[b]{0.9\textwidth}
    \centering
    \begin{subfigure}[b]{0.4\textwidth}
      \centering
      \includegraphics[width=\textwidth]{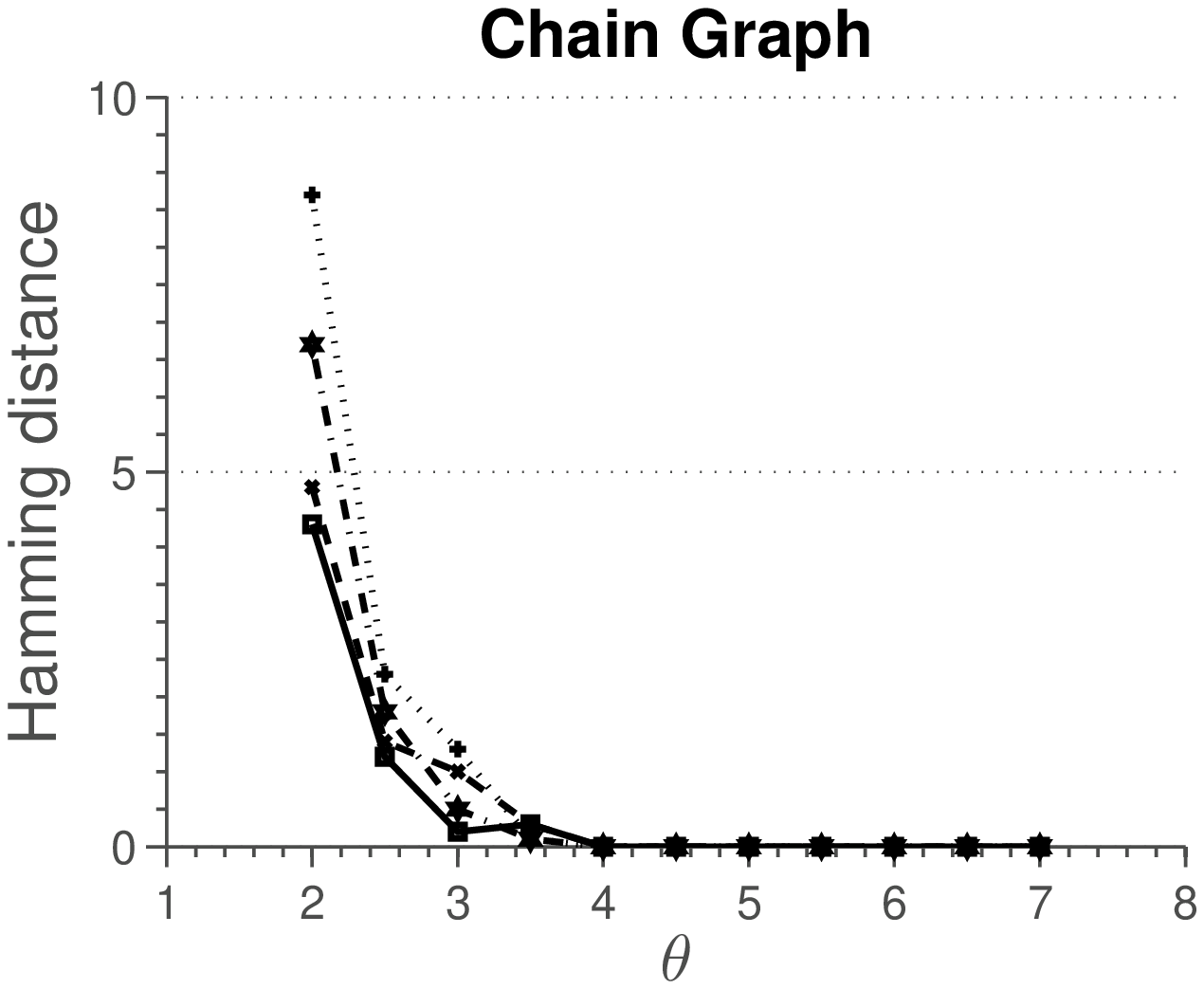}
    \end{subfigure}%
    \hfill
    \begin{subfigure}[b]{0.4\textwidth}
      \centering
      \includegraphics[width=\textwidth]{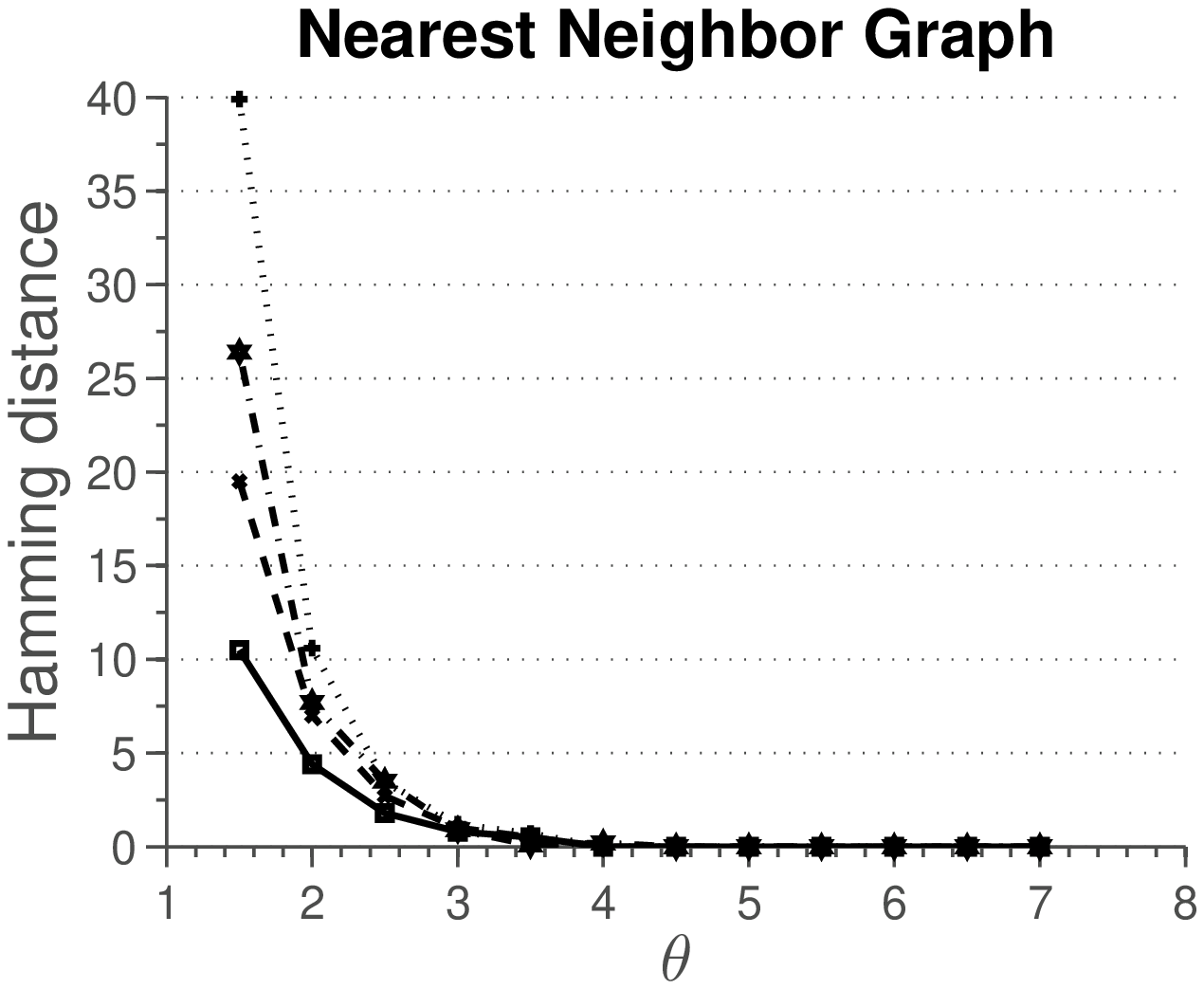}
    \end{subfigure}%
    \caption{Procedure of \cite{Guo:09}}
    \label{fig:easy:guo}
  \end{subfigure}%

 \begin{subfigure}[b]{0.9\textwidth}
    \centering
    \begin{subfigure}[b]{0.4\textwidth}
      \centering
      \includegraphics[width=\textwidth]{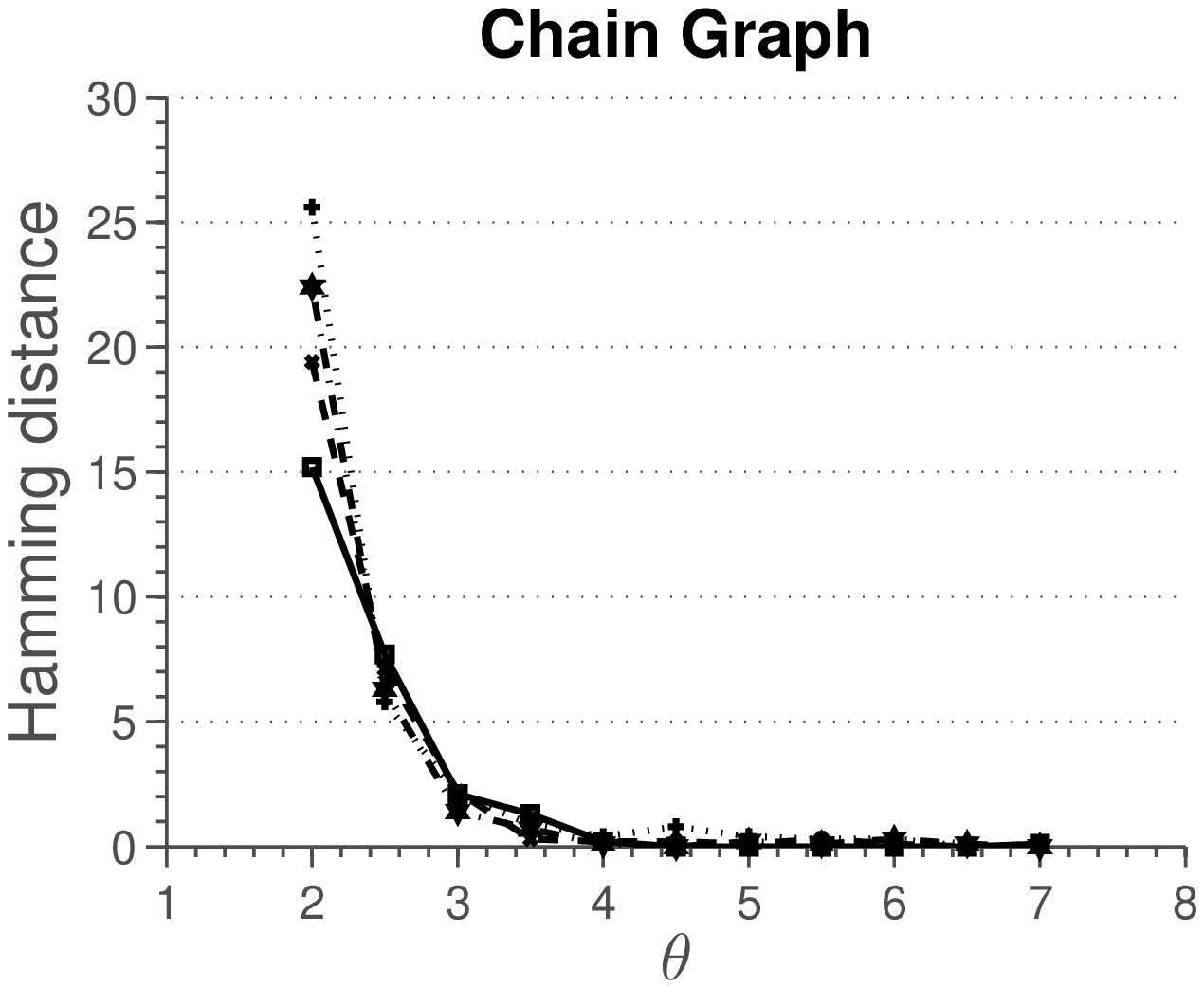}
    \end{subfigure}%
    \hfill
    \begin{subfigure}[b]{0.4\textwidth}
      \centering
      \includegraphics[width=\textwidth]{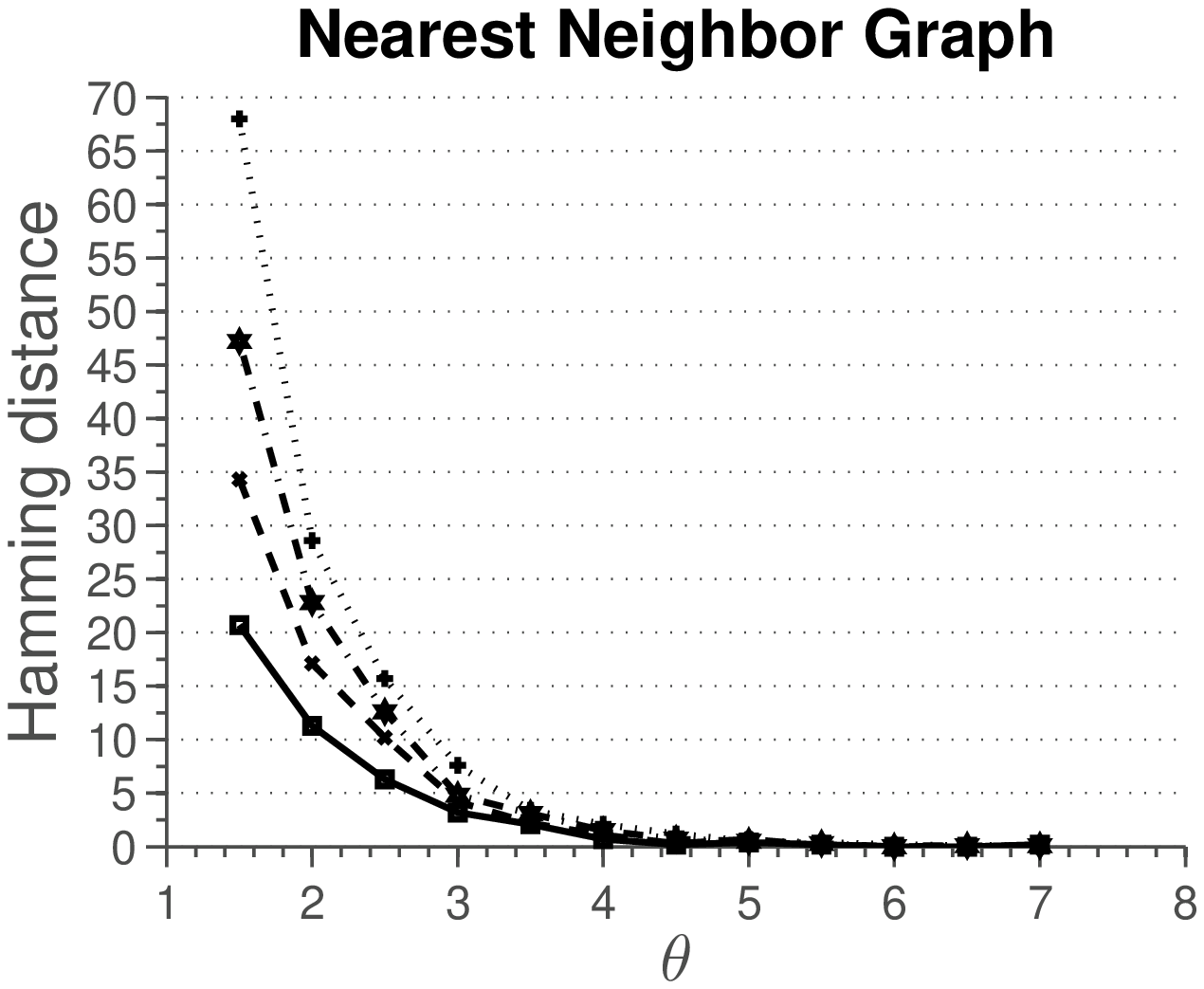}
    \end{subfigure}%
    \caption{Multi-attribute procedure}
    \label{fig:easy:ma}
  \end{subfigure}%
  \caption{Average hamming distance plotted against the rescaled
    sample size. Results are averaged over 100 independent
    runs. Blocks $\Omegab_{ab}$ of the precision matrix $\Omegab$ are
    diagonal matrices. }
  \label{fig:simulation:easy}
\end{figure}

\begin{figure}[p]
  \centering

  \begin{subfigure}[b]{0.9\textwidth}
    \centering
    \begin{subfigure}[b]{0.4\textwidth}
      \centering
      \includegraphics[width=\textwidth]{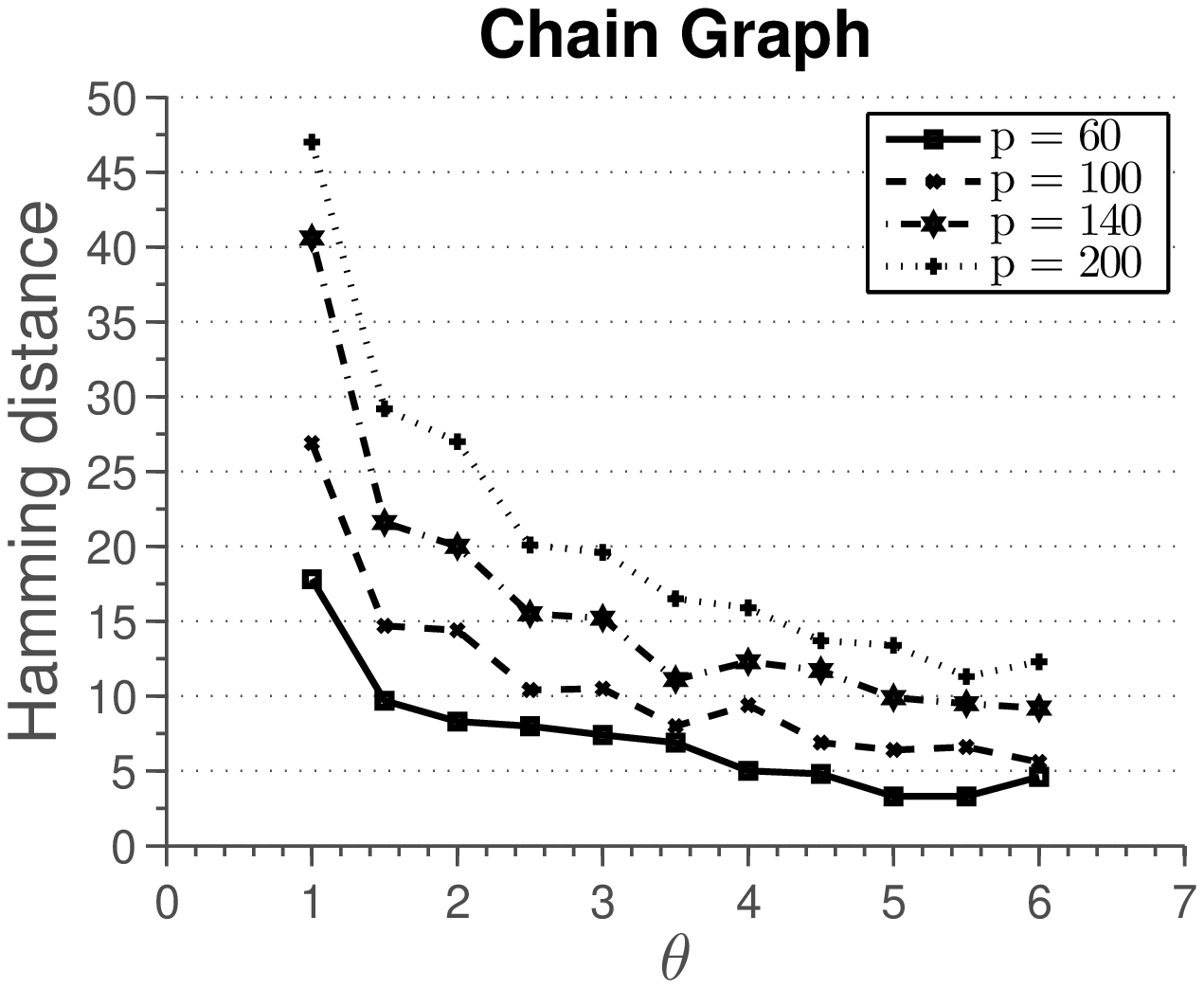}
    \end{subfigure}%
    \hfill
    \begin{subfigure}[b]{0.4\textwidth}
      \centering
      \includegraphics[width=\textwidth]{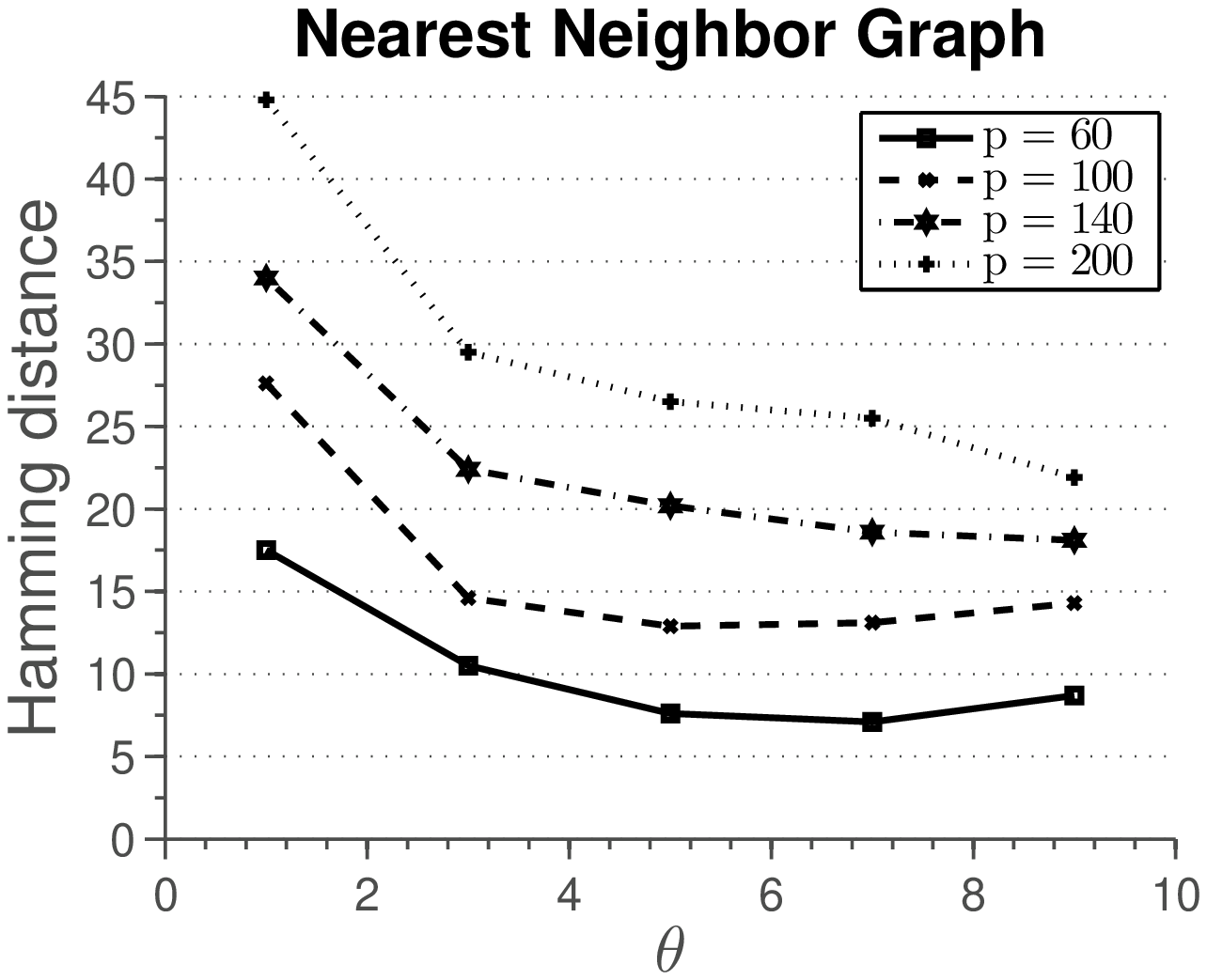}
    \end{subfigure}%
    \caption{glasso procedure}
    \label{fig:hard:glasso}
  \end{subfigure}%

 \begin{subfigure}[b]{0.9\textwidth}
    \centering
    \begin{subfigure}[b]{0.4\textwidth}
      \centering
      \includegraphics[width=\textwidth]{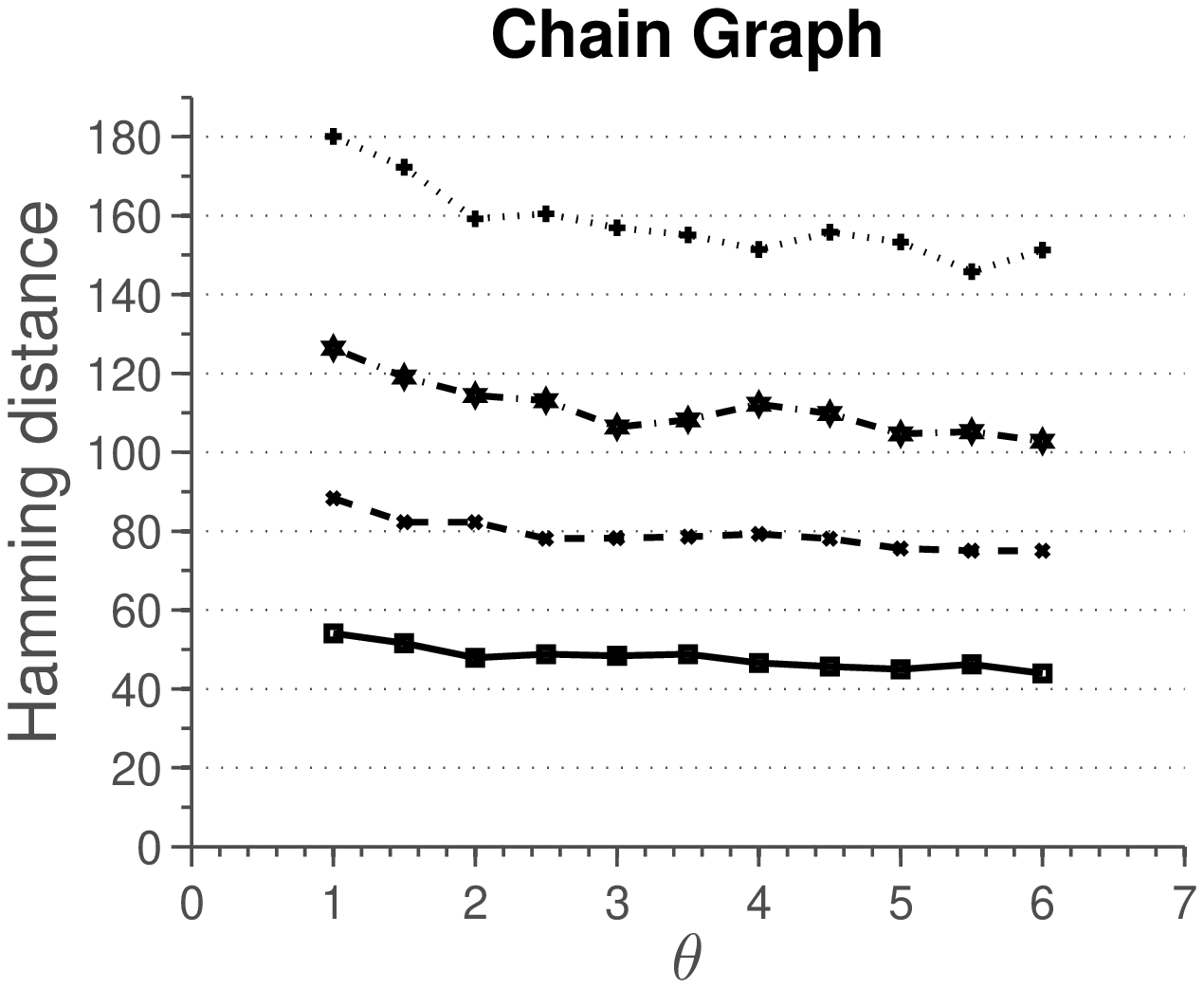}
    \end{subfigure}%
    \hfill
    \begin{subfigure}[b]{0.4\textwidth}
      \centering
      \includegraphics[width=\textwidth]{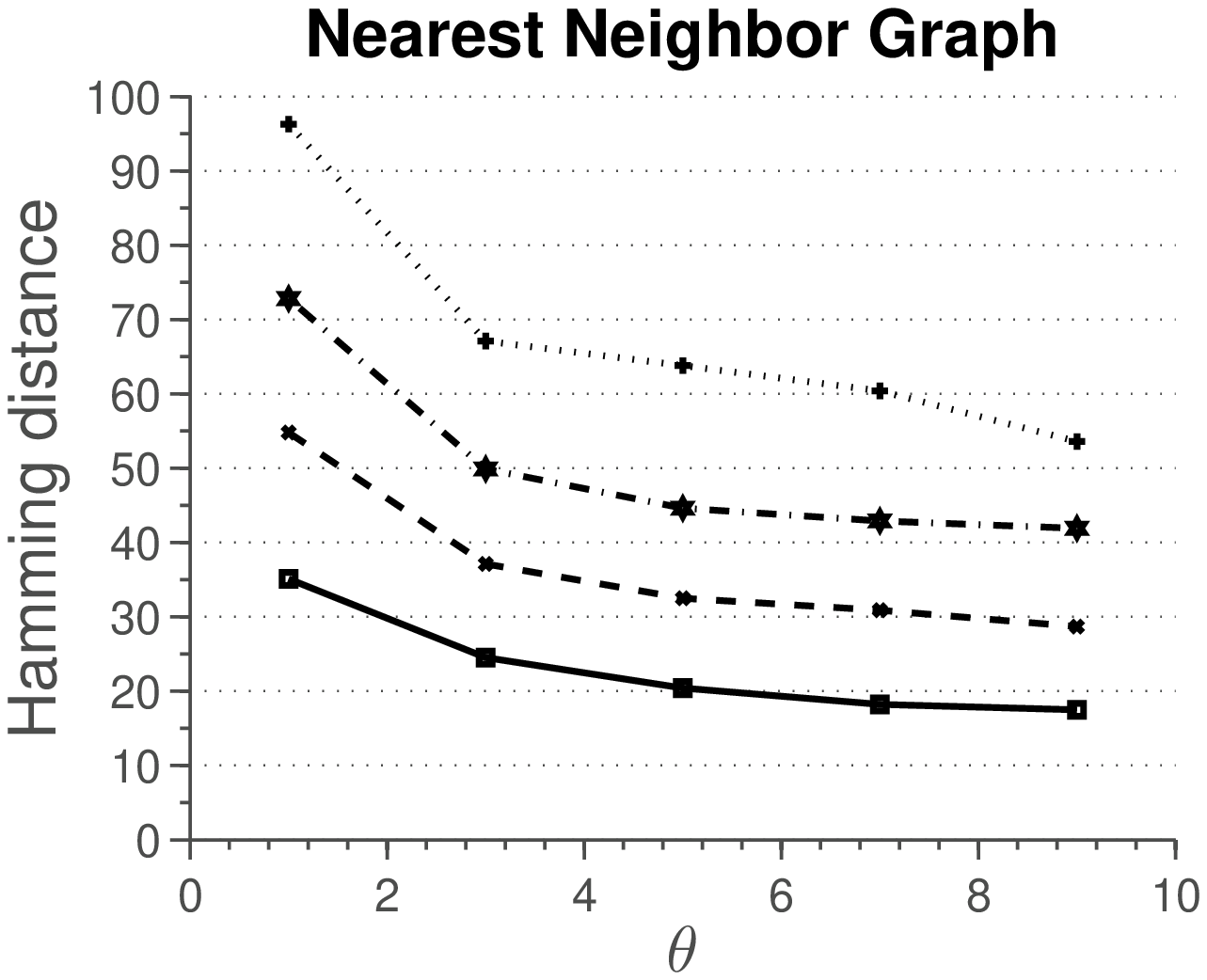}
    \end{subfigure}%
    \caption{Procedure of \cite{danaher2011joint}}
    \label{fig:hard:mt}
  \end{subfigure}%

 \begin{subfigure}[b]{0.9\textwidth}
    \centering
    \begin{subfigure}[b]{0.4\textwidth}
      \centering
      \includegraphics[width=\textwidth]{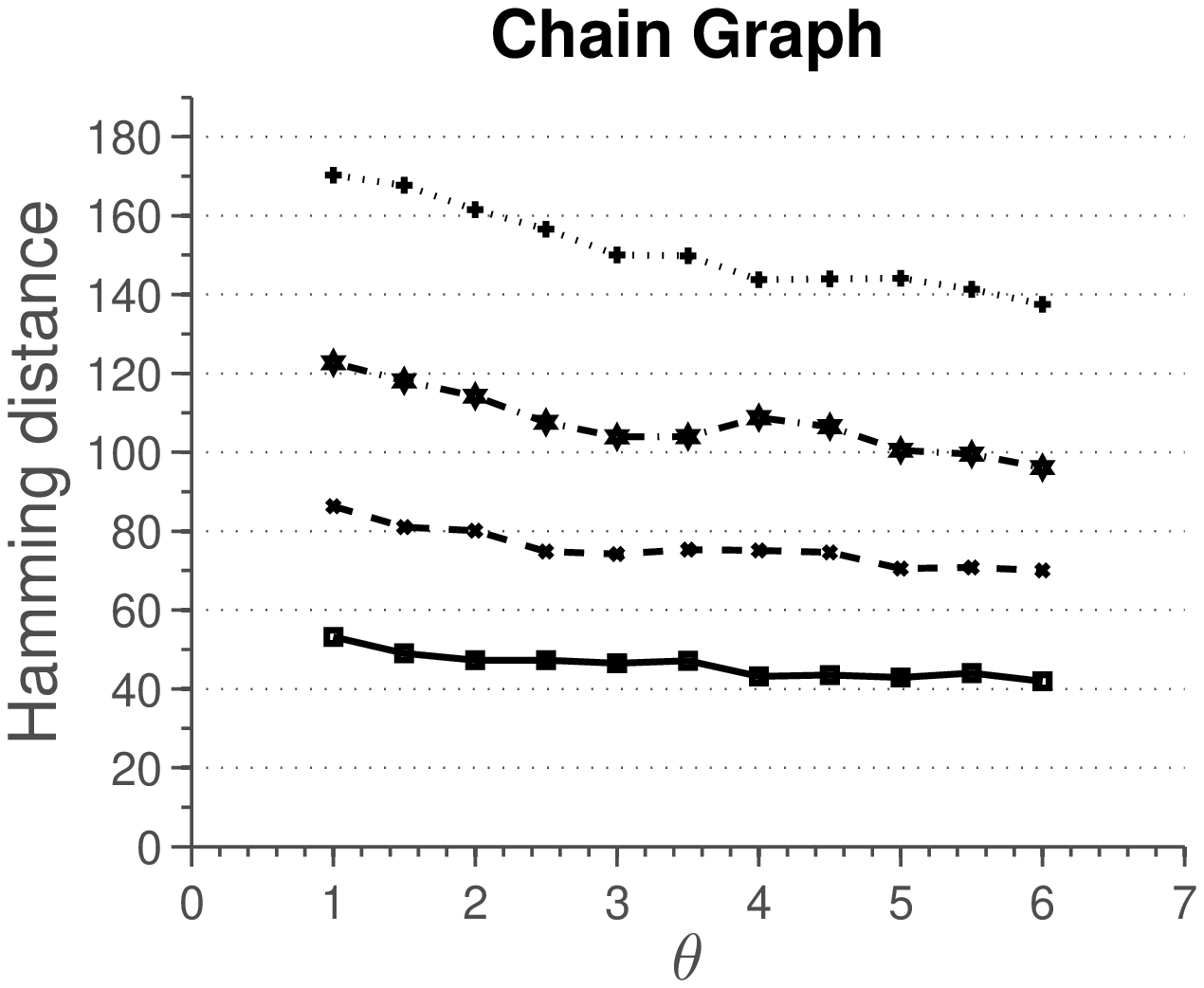}
    \end{subfigure}%
    \hfill
    \begin{subfigure}[b]{0.4\textwidth}
      \centering
      \includegraphics[width=\textwidth]{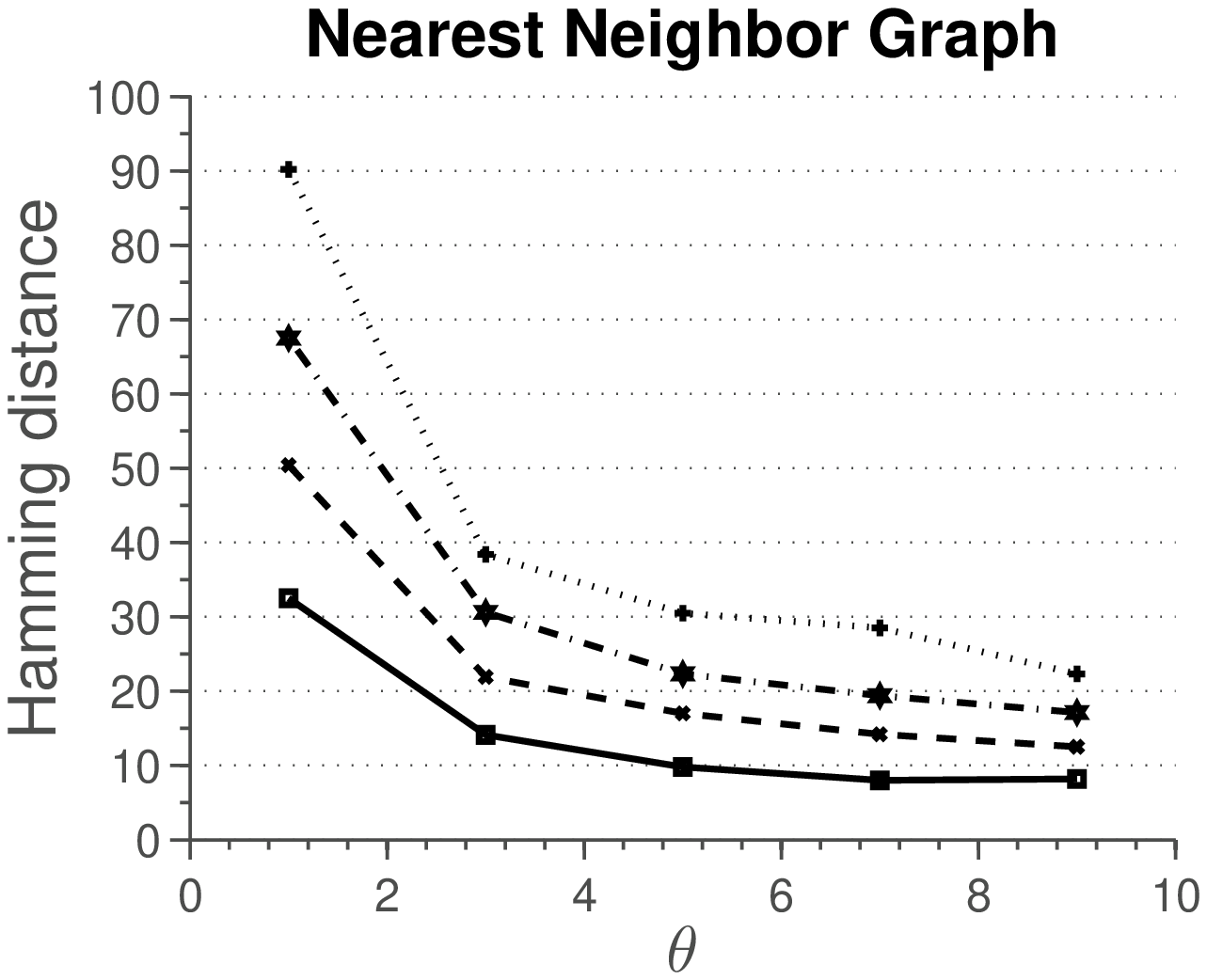}
    \end{subfigure}%
    \caption{Procedure of \cite{Guo:09}}
    \label{fig:hard:guo}
  \end{subfigure}%

 \begin{subfigure}[b]{0.9\textwidth}
    \centering
    \begin{subfigure}[b]{0.4\textwidth}
      \centering
      \includegraphics[width=\textwidth]{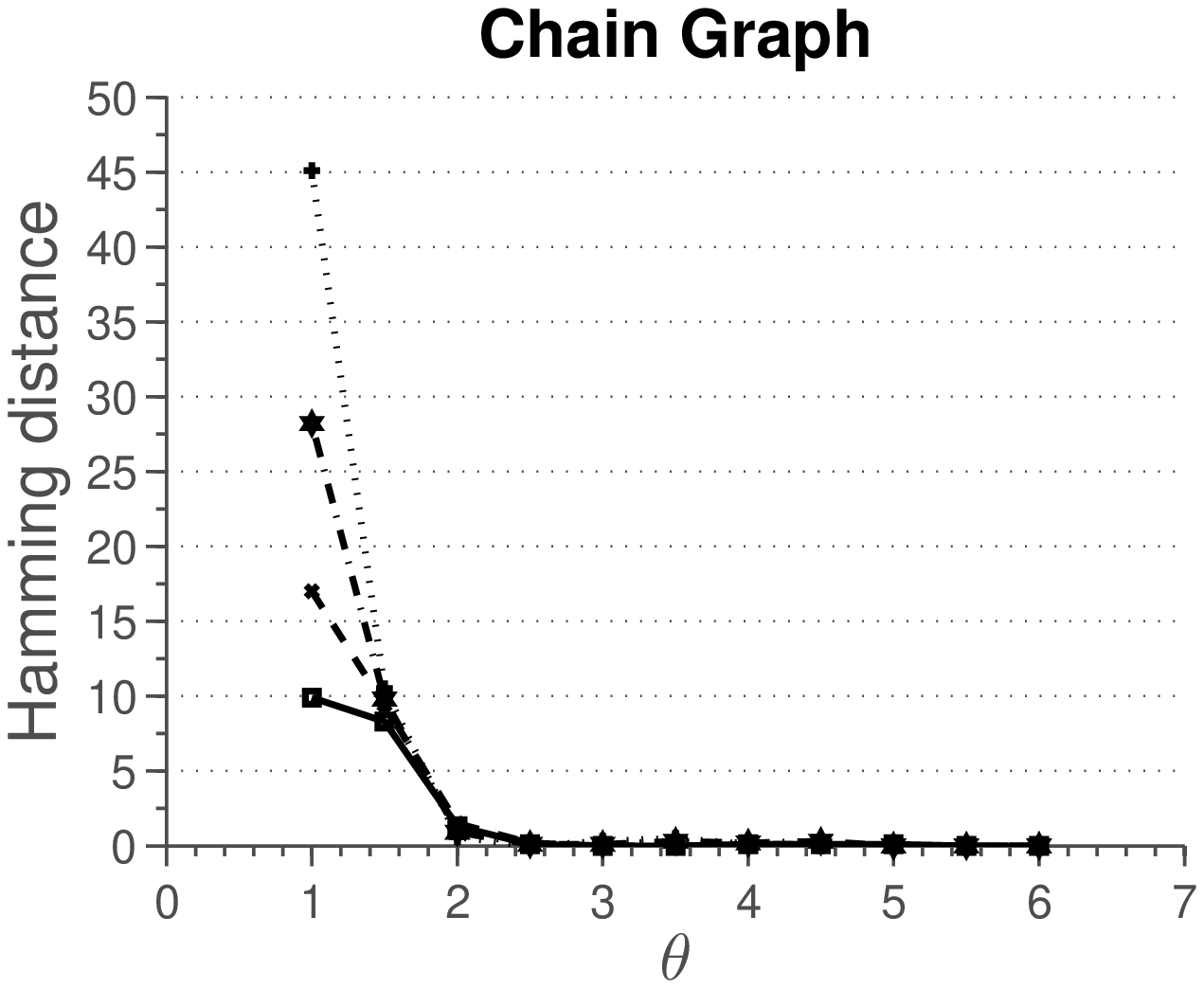}
    \end{subfigure}%
    \hfill
    \begin{subfigure}[b]{0.4\textwidth}
      \centering
      \includegraphics[width=\textwidth]{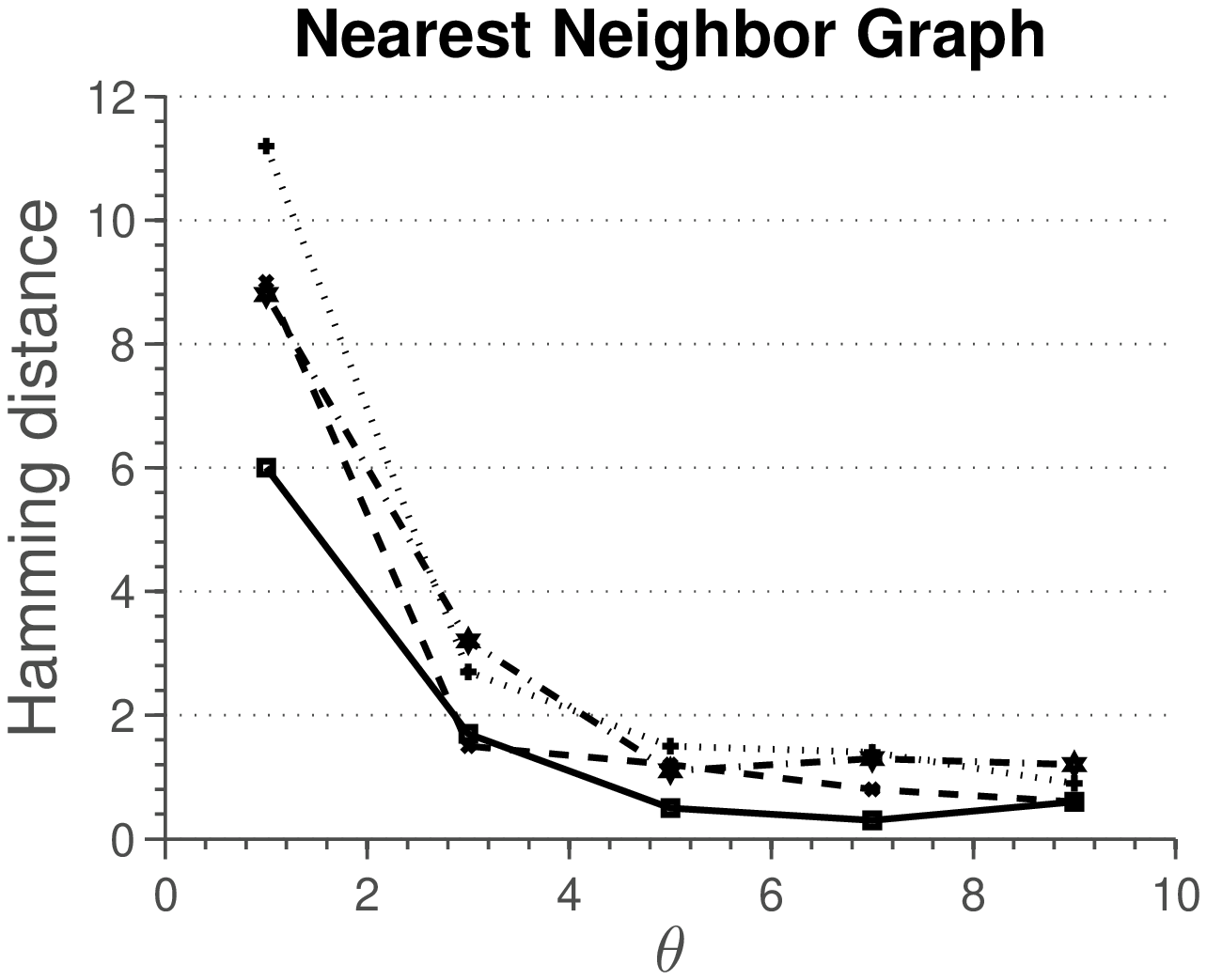}
    \end{subfigure}%
    \caption{Multi-attribute procedure}
    \label{fig:hard:ma}
  \end{subfigure}%
  \caption{Average hamming distance plotted against the rescaled
    sample size. Results are averaged over 100 independent runs.
    Off-diagonal blocks $\Omegab_{ab}$ of the precision matrix
    $\Omegab$ have zeros as diagonal elements.}
  \label{fig:simulation:hard}
\end{figure}

\begin{figure}[p]
  \centering

  \begin{subfigure}[b]{0.9\textwidth}
    \centering
    \begin{subfigure}[b]{0.4\textwidth}
      \centering
      \includegraphics[width=\textwidth]{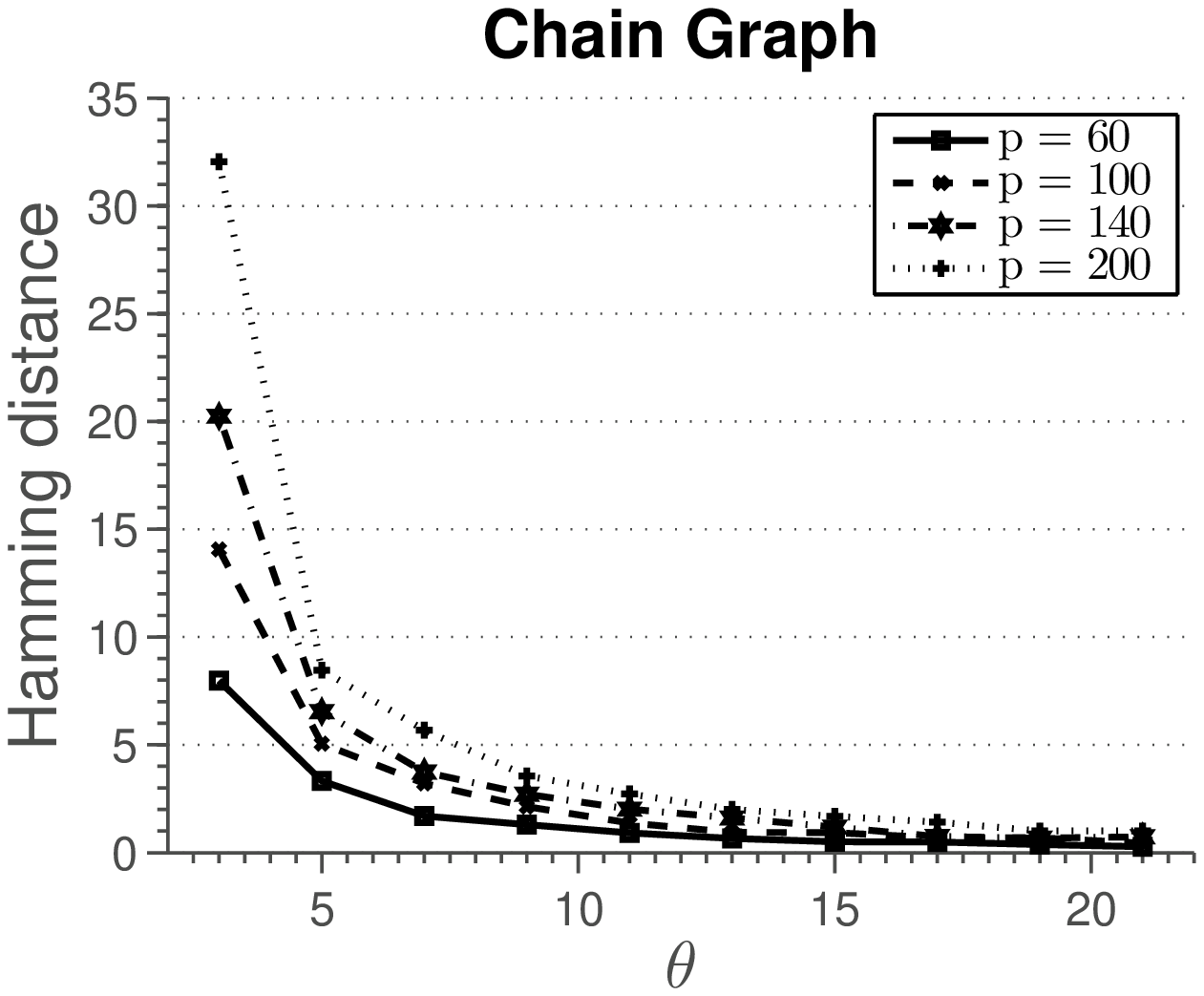}
    \end{subfigure}%
    \hfill
    \begin{subfigure}[b]{0.4\textwidth}
      \centering
      \includegraphics[width=\textwidth]{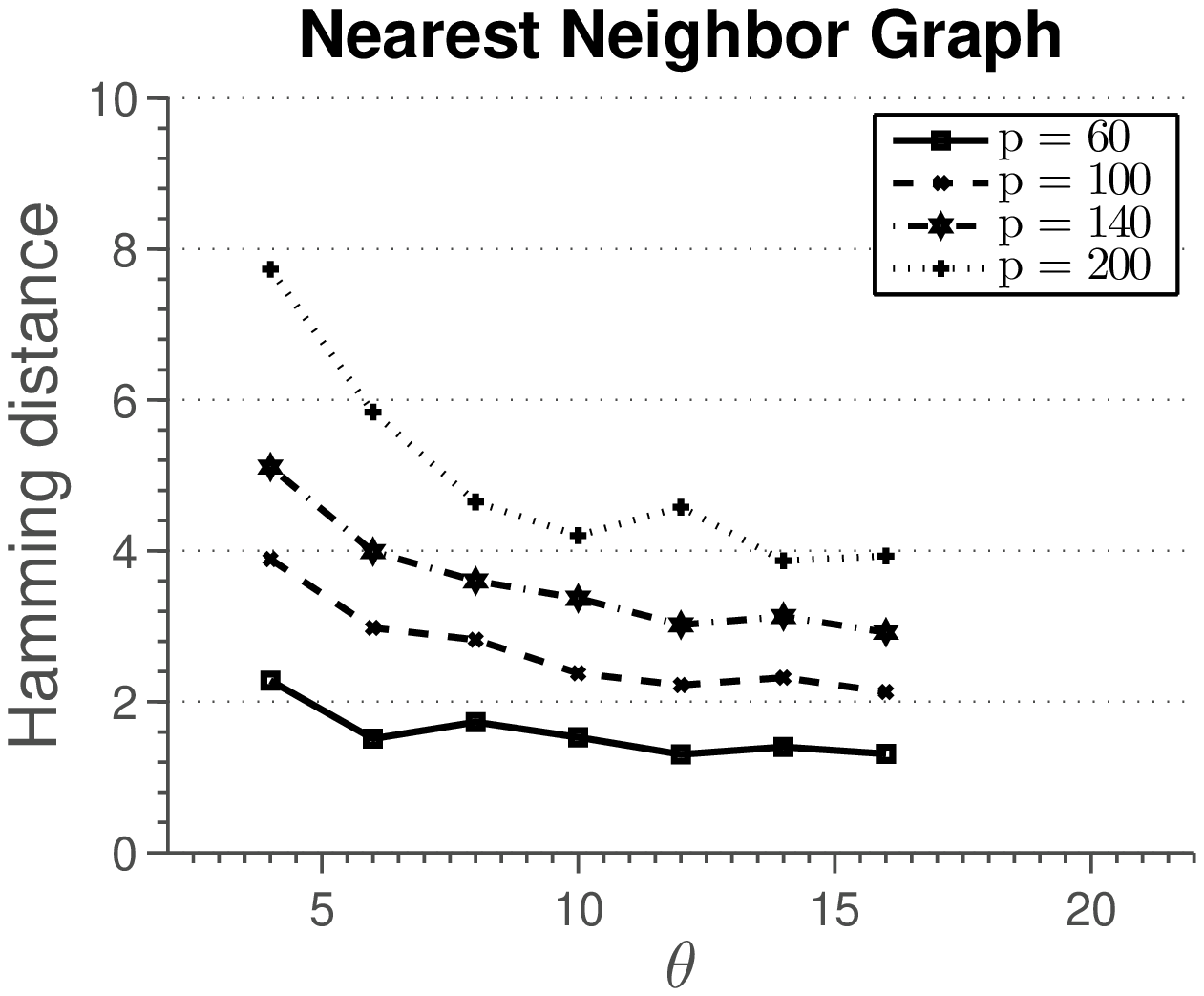}
    \end{subfigure}%
    \caption{glasso procedure}
    \label{fig:rs::glasso}
  \end{subfigure}%

 \begin{subfigure}[b]{0.9\textwidth}
    \centering
    \begin{subfigure}[b]{0.4\textwidth}
      \centering
      \includegraphics[width=\textwidth]{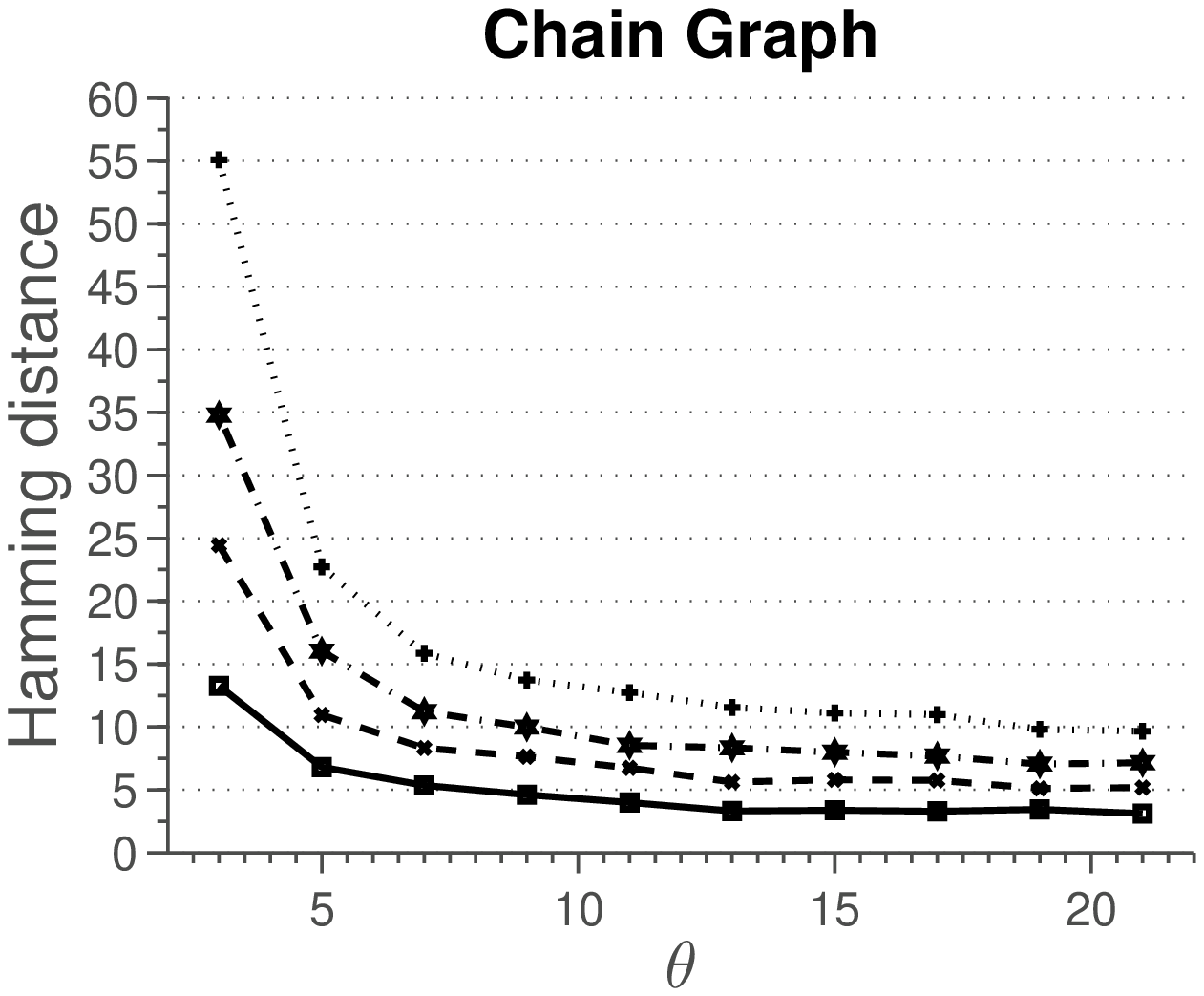}
    \end{subfigure}%
    \hfill
    \begin{subfigure}[b]{0.4\textwidth}
      \centering
      \includegraphics[width=\textwidth]{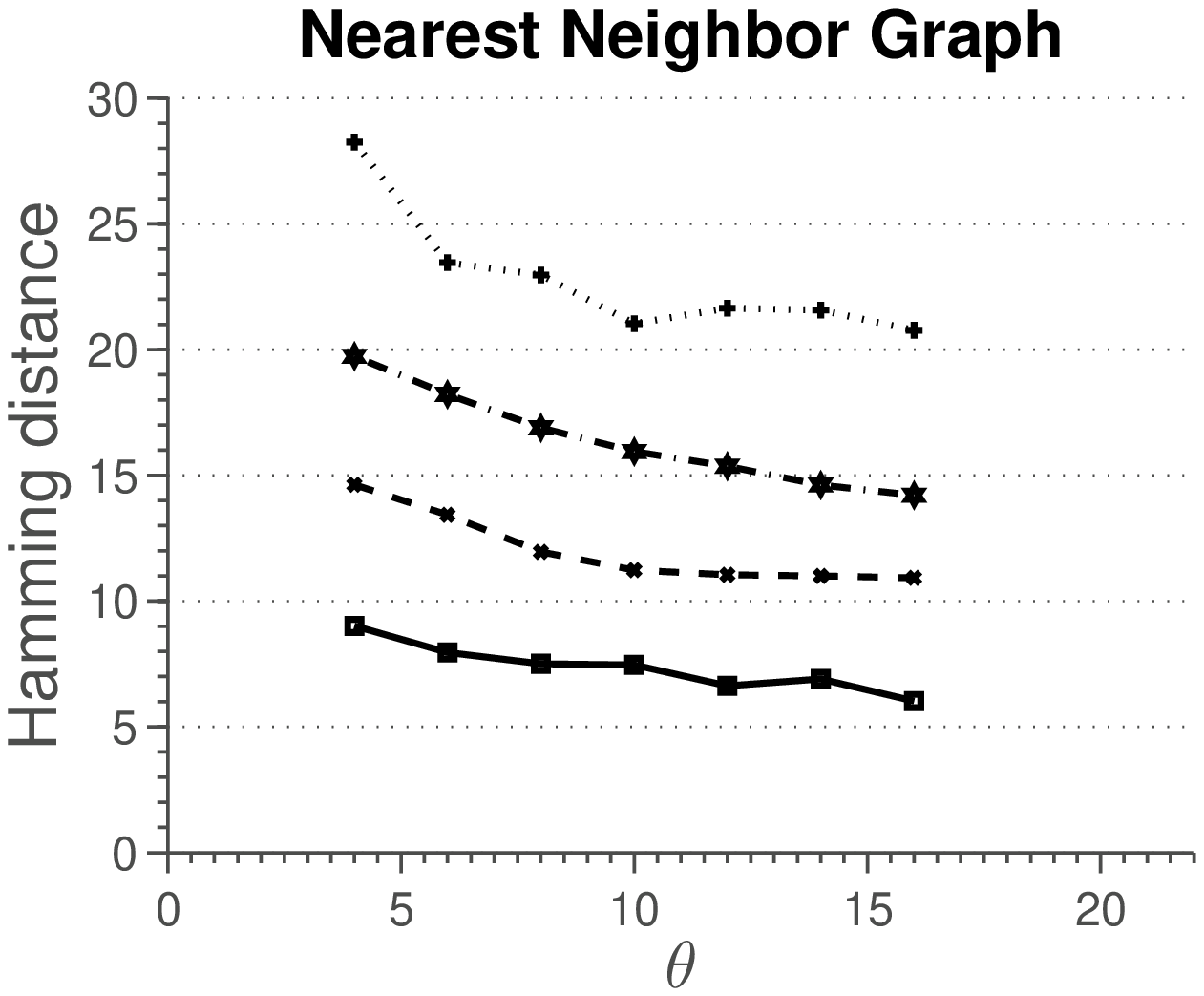}
    \end{subfigure}%
    \caption{Procedure of \cite{danaher2011joint}}
    \label{fig:rs::mt}
  \end{subfigure}%

 \begin{subfigure}[b]{0.9\textwidth}
    \centering
    \begin{subfigure}[b]{0.4\textwidth}
      \centering
      \includegraphics[width=\textwidth]{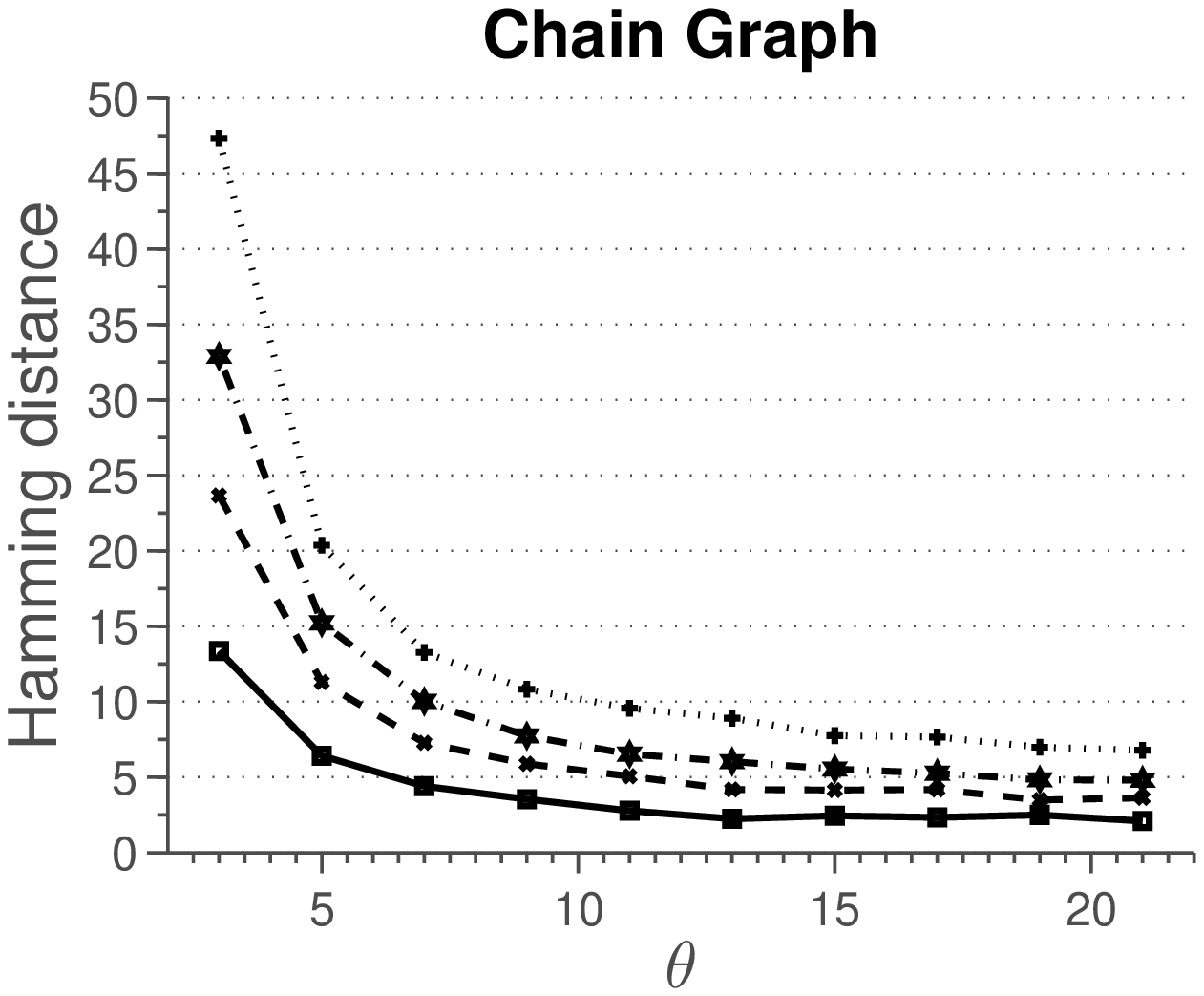}
    \end{subfigure}%
    \hfill
    \begin{subfigure}[b]{0.4\textwidth}
      \centering
      \includegraphics[width=\textwidth]{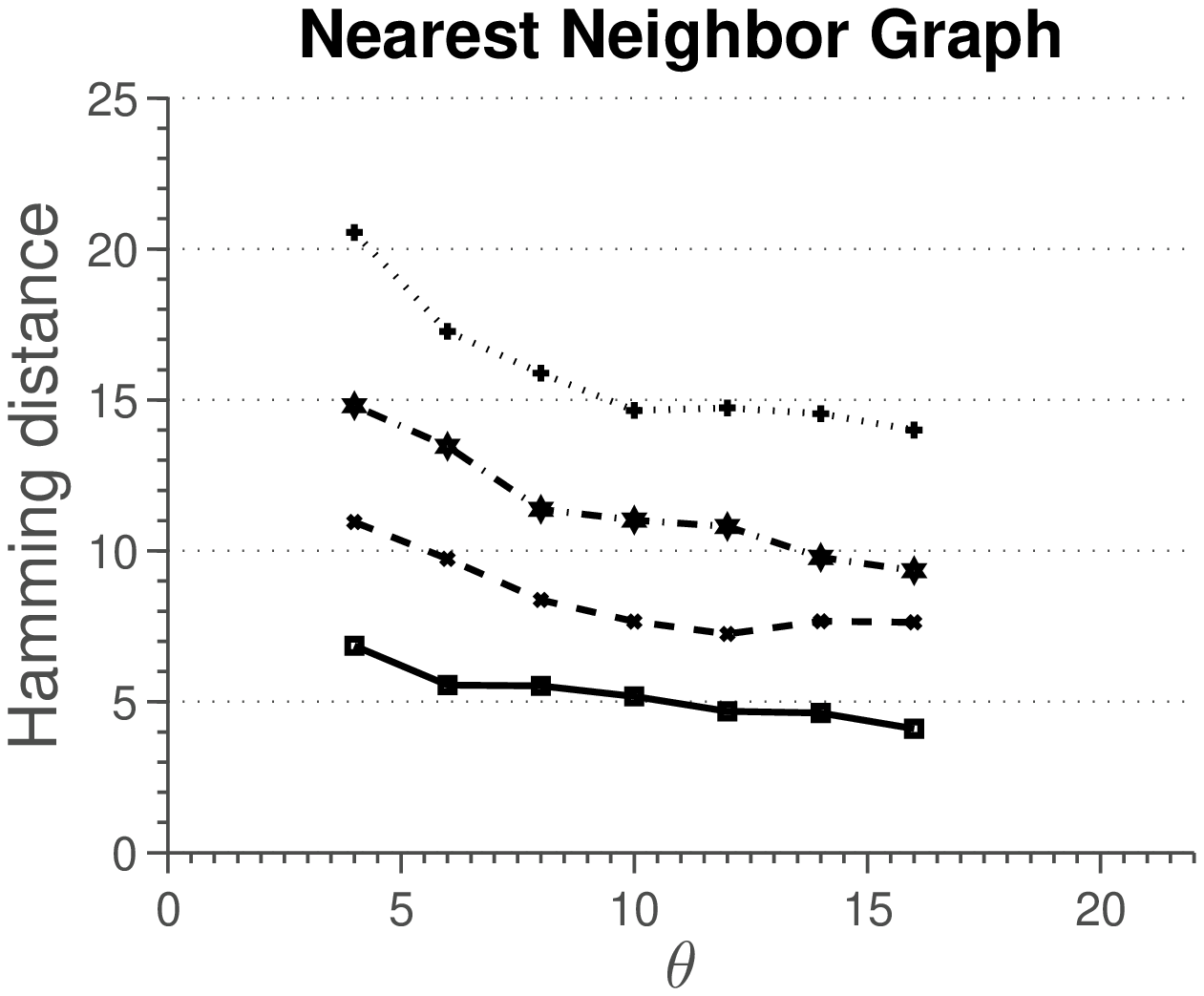}
    \end{subfigure}%
    \caption{Procedure of \cite{Guo:09}}
    \label{fig:rs::guo}
  \end{subfigure}%

 \begin{subfigure}[b]{0.9\textwidth}
    \centering
    \begin{subfigure}[b]{0.4\textwidth}
      \centering
      \includegraphics[width=\textwidth]{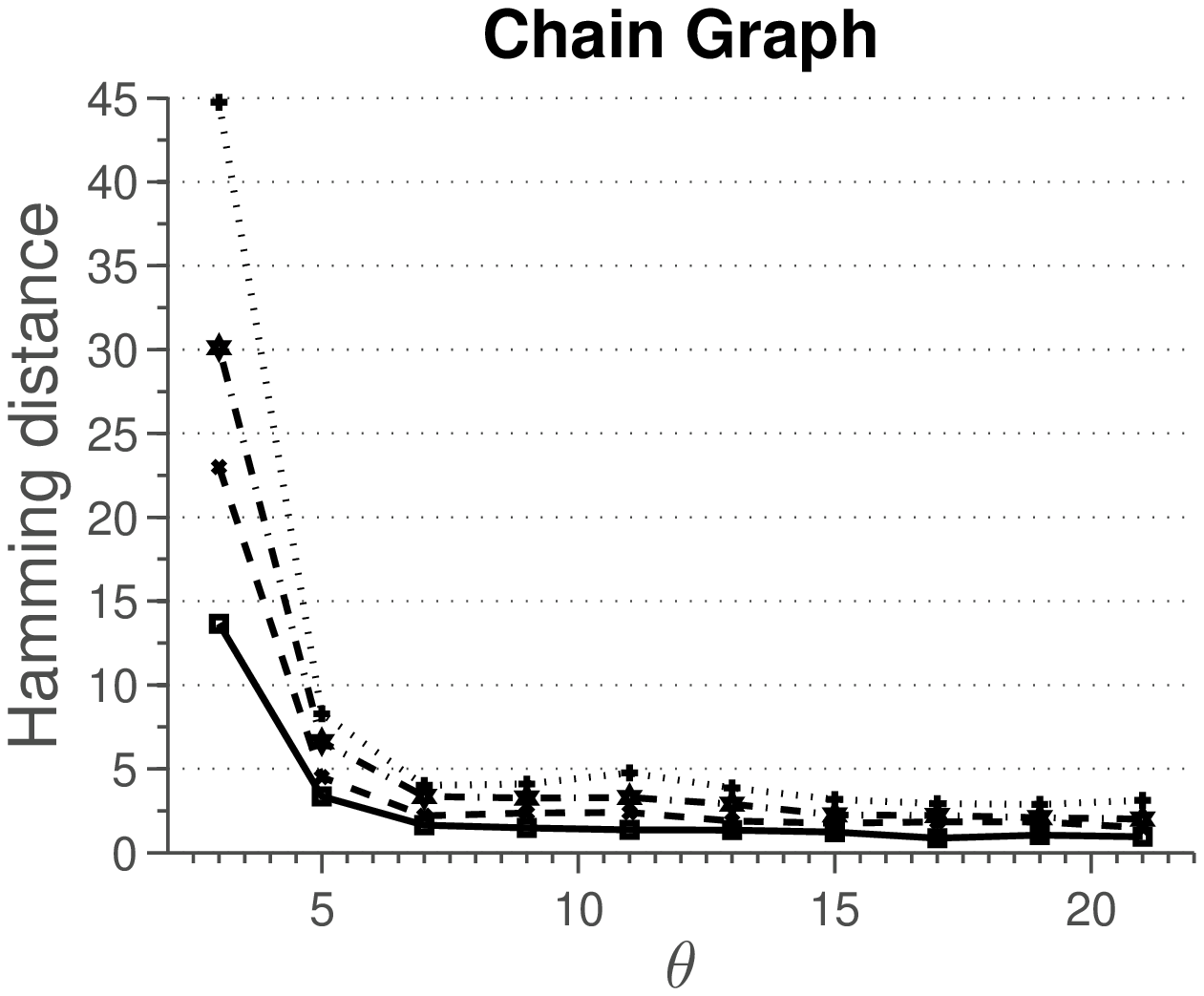}
    \end{subfigure}%
    \hfill
    \begin{subfigure}[b]{0.4\textwidth}
      \centering
      \includegraphics[width=\textwidth]{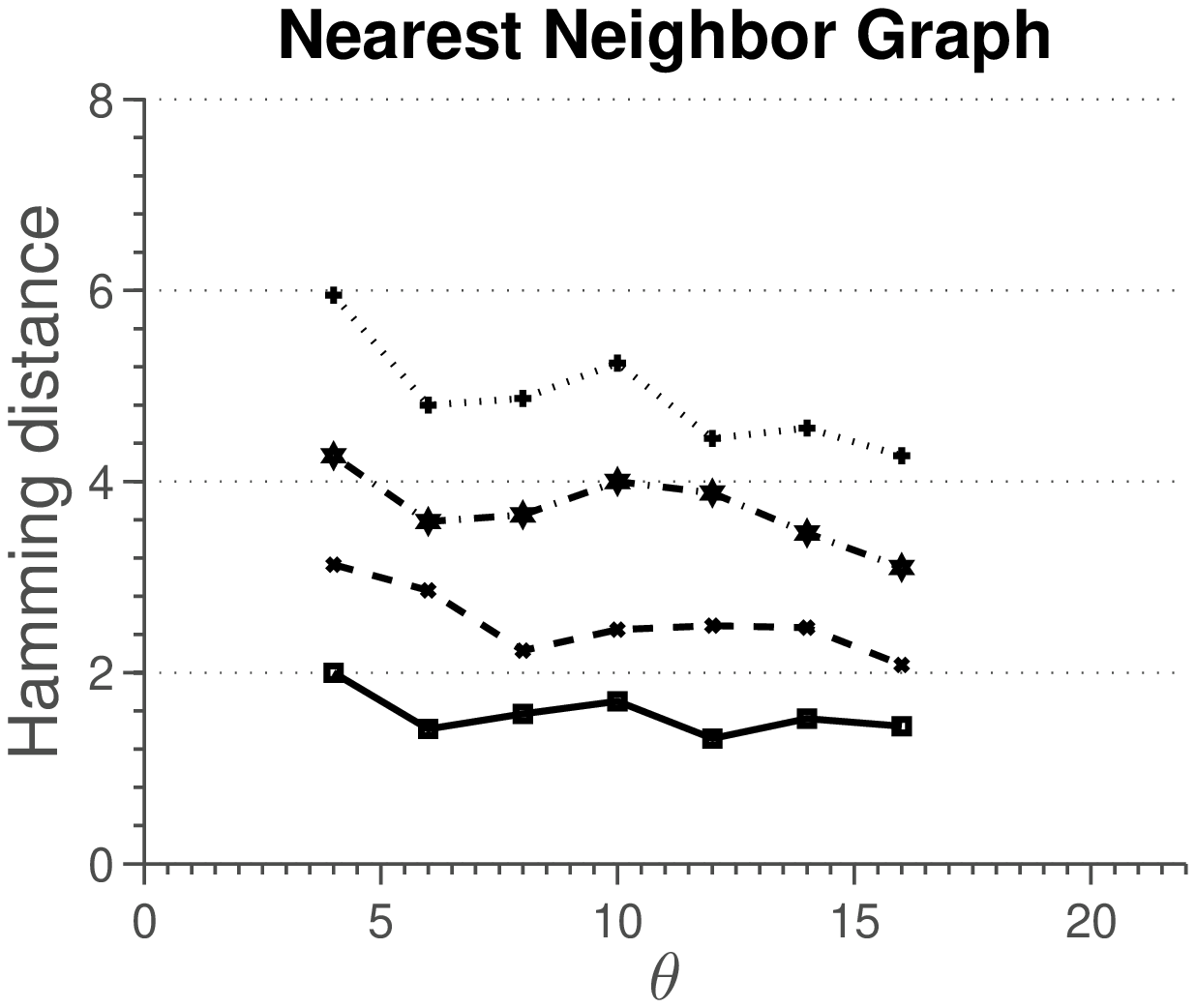}
    \end{subfigure}%
    \caption{Multi-attribute procedure}
    \label{fig:rs::ma}
  \end{subfigure}%
  \caption{Average hamming distance plotted against the rescaled
    sample size. Results are averaged over 100 independent runs.
    Off-diagonal blocks $\Omegab_{ab}$ of the precision matrix
    $\Omegab$ have elements uniformly sampled from $[-0.3, -0.1]
    \bigcup [0.1,0.3]$.  }
  \label{fig:simulation:RS}
\end{figure}

\subsection{Alternative Structure of Off-diagonal Blocks}

In this section, we investigate performance of different estimation
procedures under different assumptions on the elements of the
off-diagonal blocks of the precision matrix.

First, we investigate a situation where the multi-attribute method does
not perform as well as the methods that estimate multiple graphical
models. One such situation arises when different attributes are
conditionally independent. To simulate this situation, we use the data
generating approach as before, however, we make each block
$\Omegab_{ab}$ of the precision matrix $\Omega$ a diagonal
matrix. Figure \ref{fig:simulation:easy} summarizes results of the
simulation. We see that the methods of \cite{danaher2011joint} and
\cite{Guo:09} perform better, since they are estimating much fewer
parameters than the multi-attribute method. ${\rm glasso}$ does not
utilize any structural information underlying the estimation problem
and requires larger sample size to correctly estimate the graph than
other methods.

A completely different situation arises when the edges between nodes
can be inferred only based on inter-attribute data, that is, when a
graph based on any individual attribute is empty. To generate data
under this situation, we follow the procedure as before, but with the
diagonal elements of the off-diagonal blocks $\Omegab_{ab}$ set to
zero. Figure~\ref{fig:simulation:hard} summarizes results of the
simulation. In this setting, we clearly see the advantage of the
multi-attribute method, compared to other three methods. Furthermore,
we can see that ${\rm glasso}$ does better than multi-graph methods of
\cite{danaher2011joint} and \cite{Guo:09}. The reason is that ${\rm
  glasso}$ can identify edges based on inter-attribute relationships
among nodes, while multi-graph methods rely only on intra-attribute
relationships. This simulation illustrates an extreme scenario where
inter-attribute relationships are important for identifying edges.

So far, off-diagonal blocks of the precision matrix were constructed
to have constant values. Now, we use the same data generating
procedure, but generate off-diagonal blocks of a precision matrix in a
different way. Each element of the off-diagonal block $\Omega_{ab}$ is
generated independently and uniformly from the set $[-0.3, -0.1]
\bigcup [0.1,0.3]$. The results of the simulation are given in
Figure~\ref{fig:simulation:RS}. Again, qualitatively, the results are
similar to those given in Figure~\ref{fig:simulation:normal}, except
that in this setting more samples are needed to recover the graph
correctly.

\subsection{Different Number of Samples per Attribute}

In this section, we show how to deal with a case when different number
of samples is available per attribute. As noted in
$\S$\ref{sec:estimation}, we can treat non-measured attributes as
missing completely at random \citep[see][for more
details]{kolar12missing}. 

Let $R = (r_{il})_{i\in\{1,\ldots,n\},l\in\{1,\ldots,pk\}}\in \RR^{n \times pk}$ be an
indicator matrix, which denotes for each sample point $x_i$ the components
that are observed. Then the sample covariance matrix $S =
(\sigma_{lk}) \in \RR^{pk \times pk}$ is estimated as
$\sigma_{lk} = \rbr{\sum_{i=1}^n r_{i,l}r_{i,k}}^{-1}
{\sum_{i=1}^n r_{i,l}r_{i,k}x_{i,l}x_{i,k}}.$
This estimate is plugged into the objective in \eqref{eq:max_ll_opt}.

\begin{figure}[b]
  \centering

  \includegraphics[width=\textwidth]{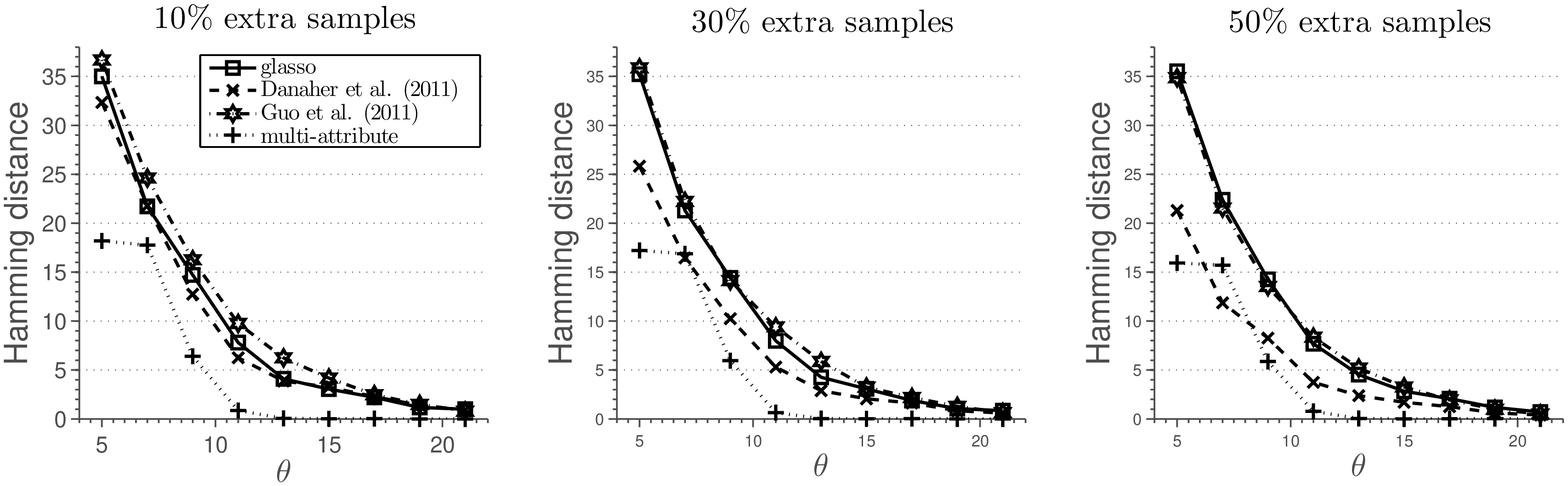}

  \caption{Average hamming distance plotted against the rescaled
    sample size. Results are averaged over 100 independent
    runs. Additional samples available for the first attribute. }
  \label{fig:simulation:diff_sample}
\end{figure}

We generate a chain graph with $p = 60$ nodes, construct a precision
matrix associated with the graph and $k=3$ attributes, and generate $n
= \theta s^2k^2\log(pk)$ samples, $\theta > 0$. Next, we generate
additional $10\%$, $30\%$ and $50\%$ samples from the same model, but
record only the values for the first attribute. Results of the
simulation are given in Figure~\ref{fig:simulation:diff_sample}.
Qualitatively, the results are similar to those presented in
Figure~\ref{fig:simulation:normal}.

\section{Illustrative Applications to Real Data}

In this section, we illustrate how to apply our method to data arising
in studies of biological regulatory networks and Alzheimer's disease.

\subsection{Analysis of a Gene/Protein Regulatory Network}
\label{sec:real_data:sup}

We provide illustrative, exploratory analysis of data from the
well-known NCI-60 database, which contains different molecular
profiles on a panel of 60 diverse human cancer cell
lines.
Data set consists of protein profiles (normalized reverse-phase lysate
arrays for 92 antibodies) and gene profiles (normalized RNA microarray
intensities from Human Genome U95 Affymetrix chip-set for $> 9000$
genes). We focus our analysis on a subset of 91 genes/proteins for
which both types of profiles are available. These profiles are
available across the same set of $60$ cancer cells. More detailed
description of the data set can be found in \citet{katenka2011multi}.

We inferred three types of networks: a network based on protein
measurements alone, a network based on gene expression profiles and a
single gene/protein network. For protein and gene networks we use the
\texttt{glasso}, while for the gene/protein network, we use our
procedure outlined in $\S$\ref{sec:estimation}. We use the stability
selection~\citep{stability:10} procedure to estimate stable
networks. In particular, we first select the penalty parameter
$\lambda$ using cross-validation, which over-selects the number of
edges in a network. Next, we use the selected $\lambda$ to estimate
100 networks based on random subsamples containing 80\% of the
data-points. Final network is composed of stable edges that appear in
at least 95 of the estimated networks. Table~\ref{tab:net_stats}
provides a few summary statistics for the estimated
networks. Furthermore, protein and gene/protein networks share $96$
edges, while gene and gene/protein networks share $104$ edges. Gene
and protein network share only $17$ edges. Finally, $66$ edges are
unique to gene/protein network. Figure~\ref{fig:node_degree} shows
node degree distributions for the three networks.  We observe that the
estimated networks are much sparser than the association networks in
\citet{katenka2011multi}, as expected due to marginal correlations
between a number of nodes.  The differences in networks require a
closer biological inspection by a domain scientist.

\begin{table}[b]

 \caption{Summary statistics for protein, gene, and gene/protein
   networks ($p = 91$).}

 \begin{tabular}{lrrr}
 ~ & { protein network} & {gene network} & {gene/protein network} \\ 

Number of edges & 122&   214 &   249 \\
Density & 0.03 & 0.05 & 0.06 \\
Largest connected component & 62 & 89 & 82 \\
Avg Node Degree  & 2.68 &  4.70 & 5.47\\
Avg Clustering Coefficient  & 0.0008 &    0.001 &   0.003\\
 \end{tabular}

\label{tab:net_stats}
\end{table}

\begin{figure}[t]
  \centering
  \includegraphics[width=\columnwidth]{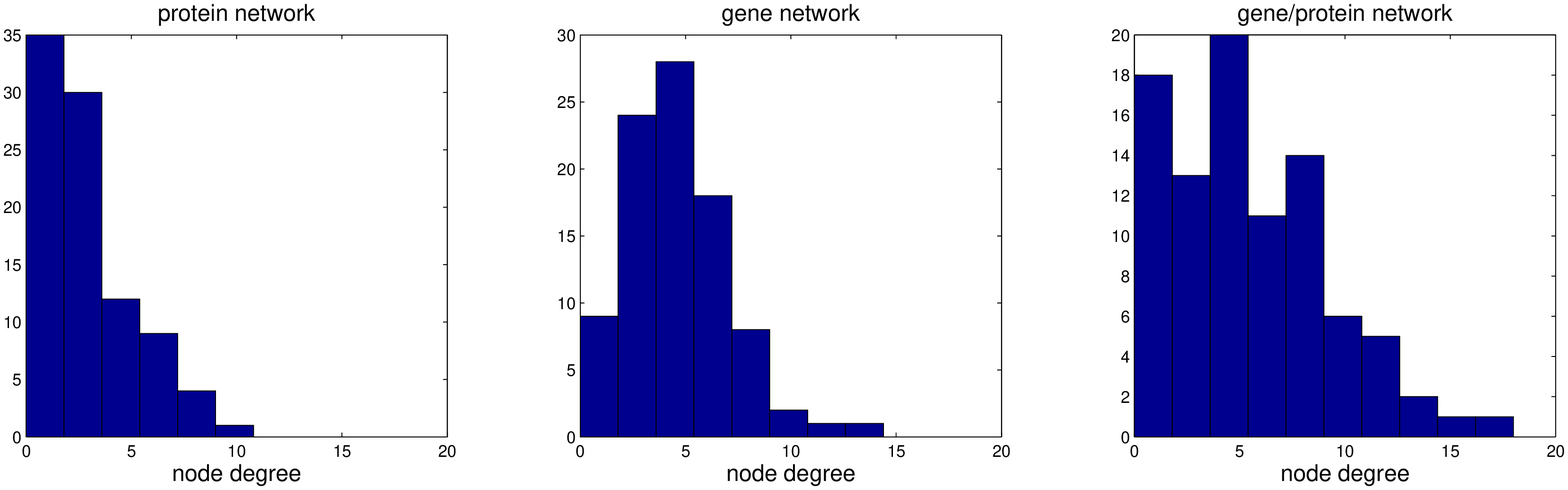}
  \caption{Node degree distributions for protein, gene and
    gene/protein networks.}
  \label{fig:node_degree}
\end{figure}

\begin{figure}[t]
  \centering
  \includegraphics[width=\columnwidth]{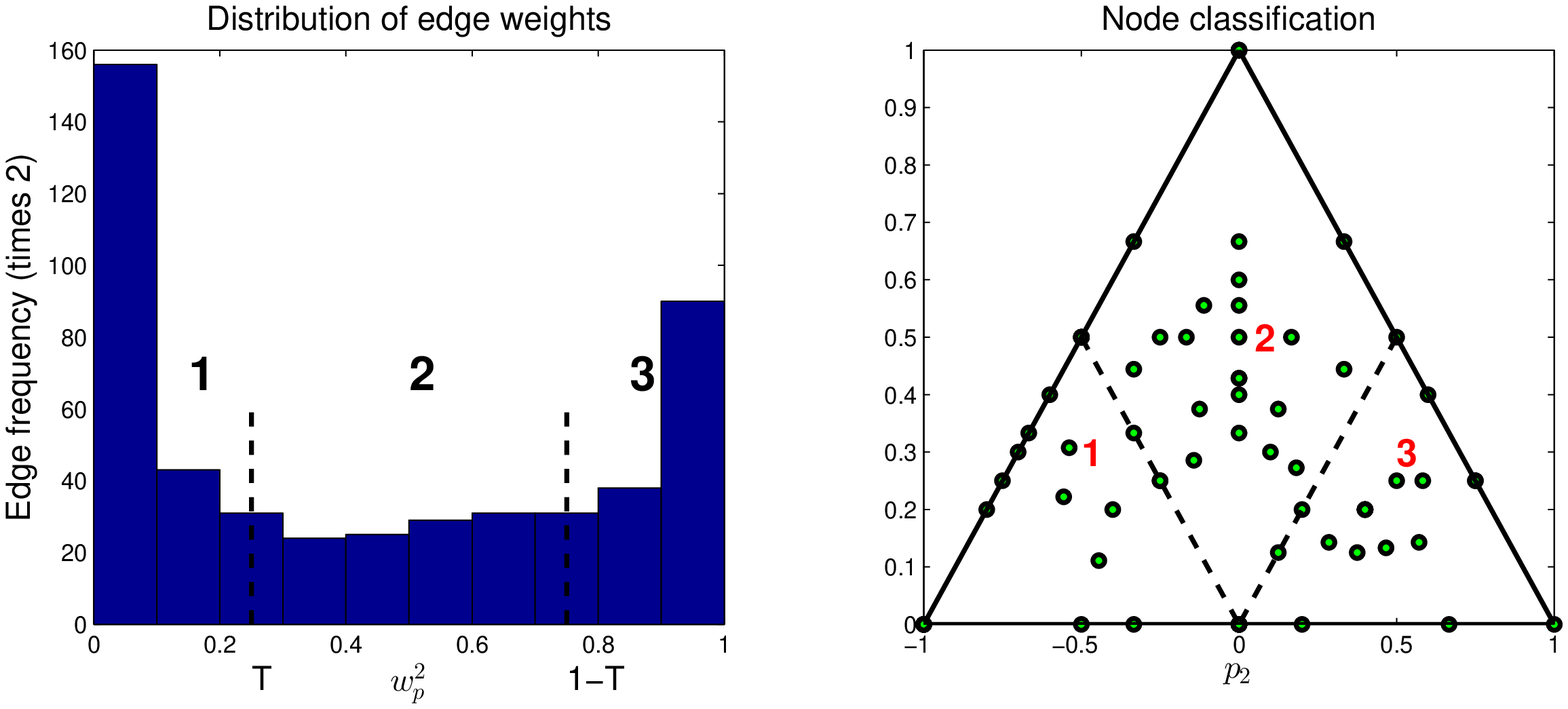}
\caption{Edge and node classification based on $w_p^2$.}
  \label{fig:node_class}
\end{figure}

We proceed with a further exploratory analysis of the gene/protein
network. We investigate the contribution of two nodal attributes to the 
existence of an edges between the nodes. Following
\citet{katenka2011multi}, we use a simple heuristic based on the weight
vectors to classify the nodes and edges into three classes. For an
edge between the nodes $a$ and $b$, we take one weight vector, say
$\wb_a$, and normalize it to have unit norm. Denote $w_{p}$ the
component corresponding to the protein attribute. Left plot in
Figure~\ref{fig:node_class} shows the values of $w_p^2$ over all
edges. The edges can be classified into three classes based on the
value of $w_p^2$. Given a threshold $T$, the edges for which $w_p^2
\in (0, T)$ are classified as gene-influenced, the edges for which
$w_p^2 \in (1-T, 1)$ are classified as protein influenced, while the
remainder of the edges are classified as mixed type. In the left plot
of Figure~\ref{fig:node_class}, the threshold is set as
$T=0.25$. Similar classification can be performed for nodes after
computing the proportion of incident edges. Let $p_1$, $p_2$ and $p_3$
denote proportions of gene, protein and mixed edges, respectively,
incident with a node. These proportions are represented in a simplex
in the right subplot of Figure~\ref{fig:node_class}. Nodes with mostly
gene edges are located in the lower left corner, while the nodes with
mostly protein edges are located in the lower right corner. Mixed
nodes are located in the center and towards the top corner of the
simplex. Further biological enrichment analysis is possible (see
\citet{katenka2011multi}), however, we do not pursue this here.

\subsection{Uncovering Functional Brain Network}
\label{sec:real_data}

We apply our method to the Positron Emission Tomography dataset, which
contains 259 subjects, of whom 72 are healthy, 132 have mild cognitive
Impairment and 55 are diagnosed as Alzheimer's \& Dementia.  Note that
mild cognitive impairment is a transition stage from normal aging to
Alzheimer's \& Dementia. The data can be obtained from {
 \it http://adni.loni.ucla.edu/}.  The preprocessing is done in the same
way as in \cite{huang2009learning}.

\begin{figure}[!b]
  \centering

    \begin{subfigure}[b]{0.33\textwidth}
      \centering
      \includegraphics[width=\textwidth]{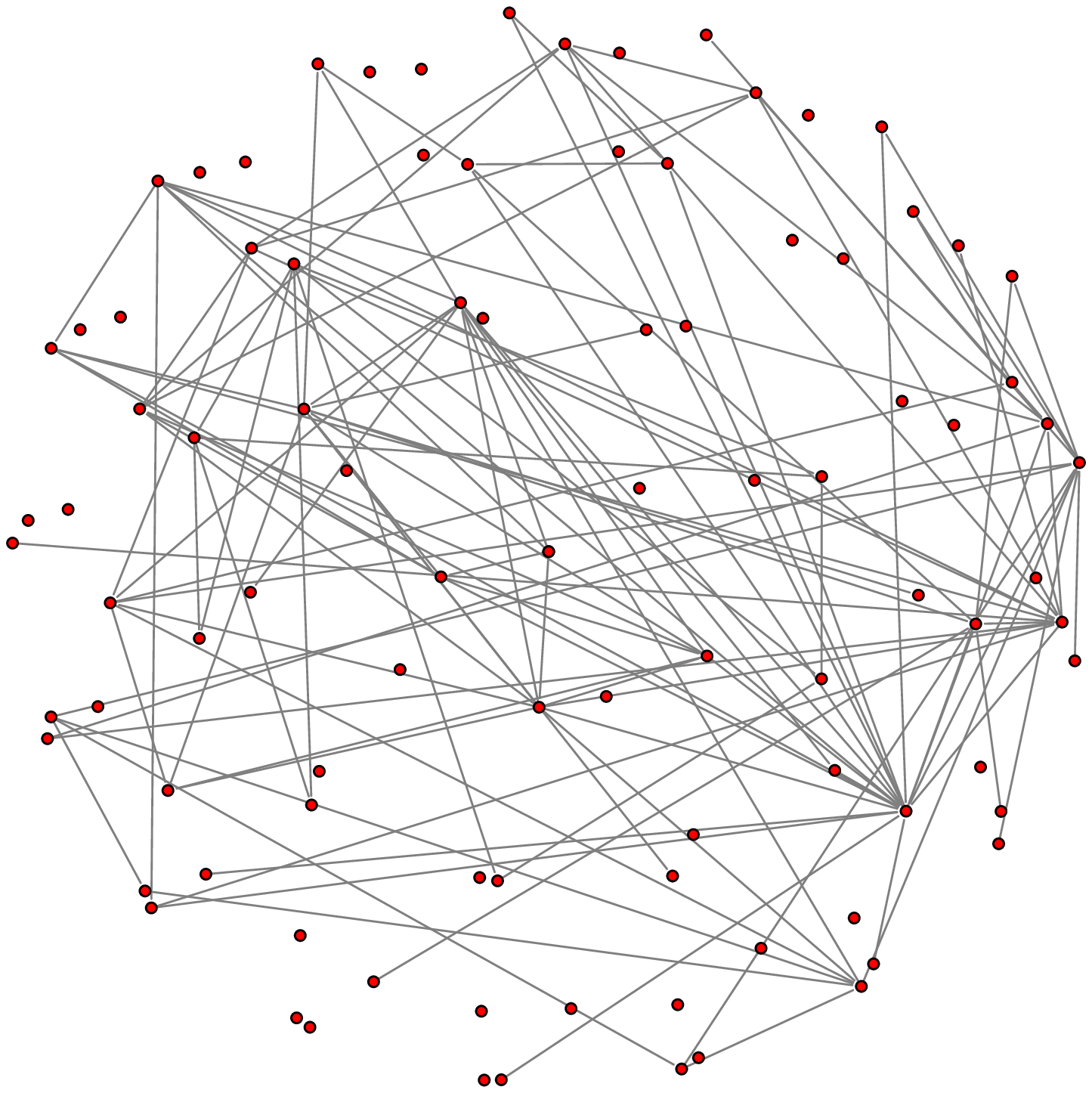}
      \caption{Healthy subjects}
    \end{subfigure}%
    \hfill
    \begin{subfigure}[b]{0.33\textwidth}
      \centering
      \includegraphics[width=\textwidth]{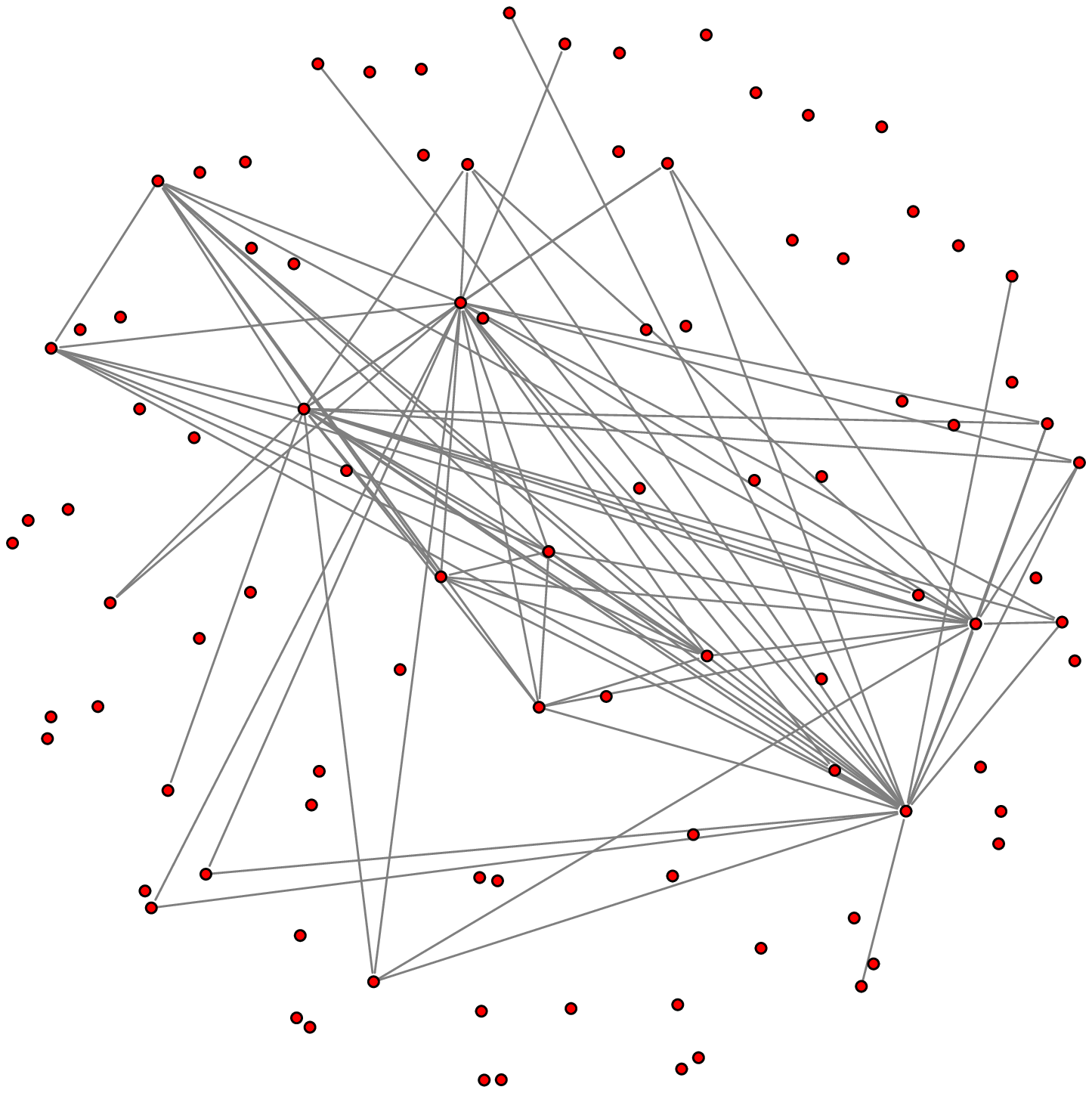}
      \caption{Mild Cognitive Impairment}
    \end{subfigure}%
    \hfill
    \begin{subfigure}[b]{0.33\textwidth}
      \centering
      \includegraphics[width=\textwidth]{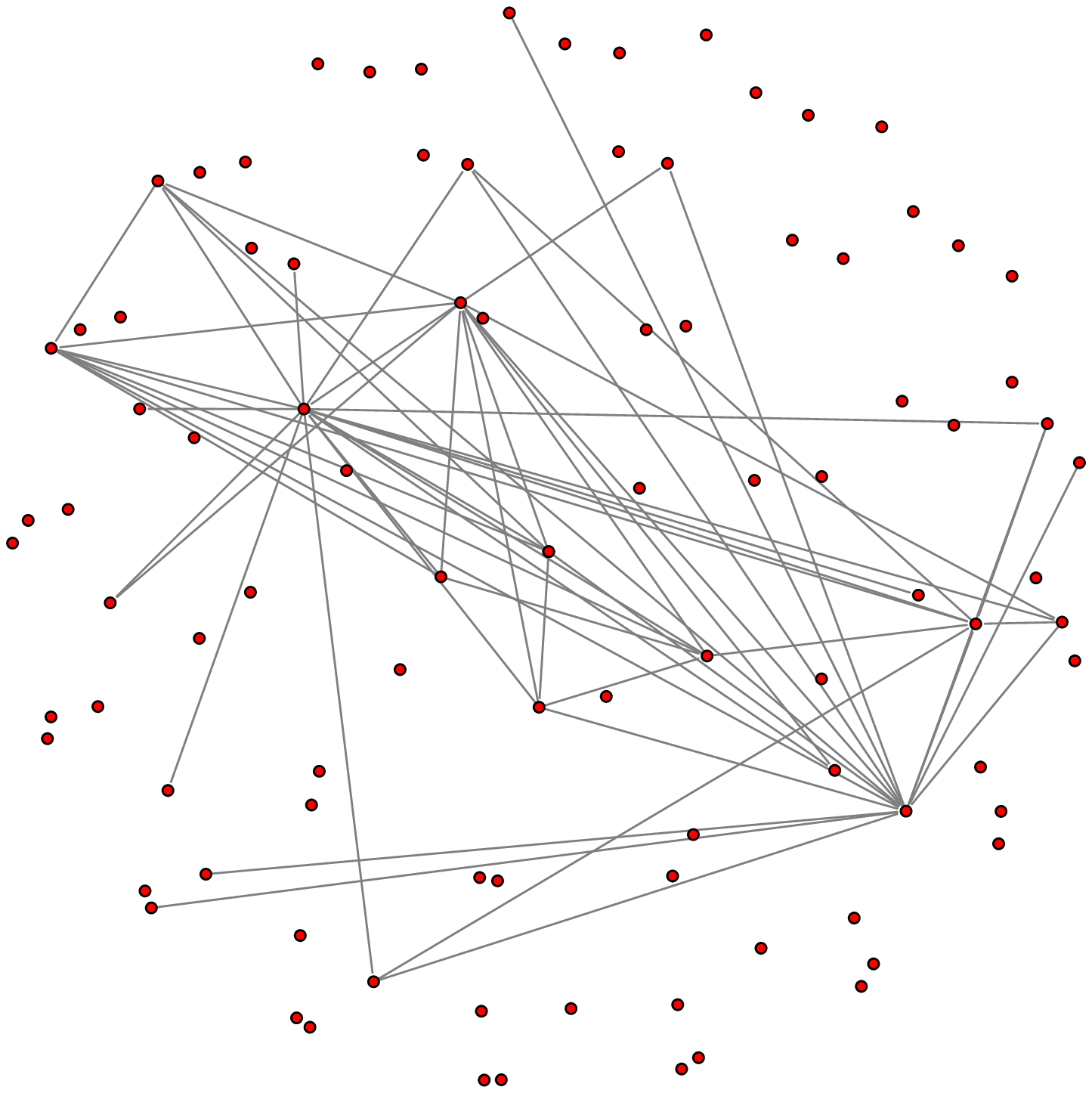}
      \caption{Alzheimer's \& Dementia}
    \end{subfigure}%
  
  \caption{Brain connectivity networks}
  \label{fig:brain}
\end{figure}

Each Positron Emission Tomography image contains $91 \times 109 \times
91 = 902,629$ voxels. The effective brain region contains $180,502$
voxels, which are partitioned into $95$ regions, ignoring the regions
with fewer than $500$ voxels. The largest region contains $5,014$
voxels and the smallest region contains $665$ voxels. Our
preprocessing stage extracts $948$ representative voxels from these
regions using the $K$-median clustering algorithm. The parameter $K$ is
chosen differently for each region, proportionally to the initial
number of voxels in that region.  More specifically, for each category
of subjects we have an $n \times (d_1 + \ldots + d_{95})$ matrix, where
$n$ is the number of subjects and $d_1 + \ldots + d_{95} = 902,629$ is
the number of voxels. Next we set $K_i = \lceil d_i / \sum_j d_j
\rceil$, the number of representative voxels in region $i$,
$i=1,\ldots,95$. The representative voxels are identified by running
the $K$-median clustering algorithm on a sub-matrix of size $n \times
d_i$ with $K = K_i$.

We inferred three networks, one for each subtype of subjects using the
procedure outlined in $\S$\ref{sec:estimation}.  Note that for
different nodes we have different number of attributes, which
correspond to medians found by the clustering algorithm.  We use the
stability selection~\citep{stability:10} approach to estimate stable
networks.  The stability selection procedure is combined with our
estimation procedure as follows.  We first select the penalty
parameter $\lambda$ in \eqref{eq:max_ll_opt} using cross-validation,
which overselects the number of edges in a network.  Next, we create
$100$ subsampled data sets, each of which contain $80\%$ of the data
points, and estimate one network for each dataset using the selected
$\lambda$.  The final network is composed of stable edges that appear
in at least $95$ of the estimated networks.

We visualize the estimated networks in Figure~\ref{fig:brain}.
Table~\ref{tab:brain_net_stats} provides a few summary statistics for
the estimated networks. Appendix~\ref{appendix:d} contains names of
different regions, as well as the adjacency matrices for networks.
From the summary statistics, we can observe that in normal subjects
there are many more connections between different regions of the
brain.  Loss of connectivity in Alzheimer's \& Dementia has been
widely reported in the literature \citep{greicius2004default,
  hedden2009disruption, andrews2007disruption, wu2011altered}.

Learning functional brain connectivity is potentially valuable for
early identification of signs of Alzheimer's disease.
\citet{huang2009learning} approach this problem using exploratory data
analysis. The framework of Gaussian graphical models is
used to explore functional brain connectivity. Here we point out that
our approach can be used for the same exploratory task, without the
need to reduce the information in the whole brain to one number. For
example, from our estimates, we observe the loss of connectivity in
the cerebellum region of patients with Alzheimer's disease, which has
been reported previously in \citet{sjobeck2001alzheimer}. As another
example, we note increased connectivity between the frontal lobe and
other regions in the patients, which was linked to compensation for
the lost connections in other regions \citep{stern2006cognitive,
  gould2006brain}. 

\begin{table}[t]

 \caption{Summary statistics for protein, gene, and gene/protein
   networks ($p = 91$)}

 \begin{tabular}{lrrr}
 ~ & { Healthy }
 & {Mild Cognitive} 
& {Alzheimer's \& } \\ 

& subjects & Impairment & Dementia \\

Number of edges & 116 & 84  &   59 \\
Density & 0.030 & 0.020 & 0.014 \\
Largest connected component & 48 & 27 & 25 \\
Avg Node Degree  & 2.40 & 1.73 &  1.2 \\
Avg Clustering Coefficient  & 0.001  & 0.0023 &    0.0007 \\
 \end{tabular}

\label{tab:brain_net_stats}
\end{table}

\subsection*{Acknowledgments}

{\small
We thank Eric D. Kolaczyk and Natallia V. Katenka for sharing
preprocessed data used in their study with us. Eric P. Xing is
partially supported through the grants NIH R01GM087694 and AFOSR
FA9550010247.  The research of Han Liu is supported by NSF grant
IIS-1116730.
 }

\bibliography{biblio}

\newpage

\appendix

\section{Efficient Updates of the Precision and Covariance Matrices}
\label{appendix:a}

Our algorithm consists of updating the precision matrix by solving
optimization problems \eqref{eq:update_aa} and \eqref{eq:update_ab}
and then updating the estimate of the covariance matrix. Both steps
can be performed efficiently.

Solutions to \eqref{eq:update_aa} and \eqref{eq:update_ab} can
be computed in a closed form as
\begin{align}
  \label{eq:update_aa_final}
  \hat \Omegab_{aa} &=
  (
    1 - t\lambda/\norm{\tilde\Omegab_{aa} + t(\tilde\Sigmab_{aa}-\Sbb_{aa})
    }_F
  )_+
  (\tilde\Omegab_{aa} + t(\tilde\Sigmab_{aa}-\Sbb_{aa})),
  \quad\text{ and} \\
  \label{eq:update_ab_final}
  \hat \Omegab_{ab} & =
  (
    1 - t\lambda/
      \norm{\tilde\Omegab_{ab} + t(\tilde\Sigmab_{ab}-\Sbb_{ab})}_F
  )_+
  (\tilde\Omegab_{ab} + t(\tilde\Sigmab_{ab}-\Sbb_{ab})),
  \quad \forall b \in \ia,
\end{align}
where $(x)_+ = \max(0, x)$.

The estimate of the covariance matrix can be updated efficiently,
without inverting the whole $\hat\Omegab$ matrix, using the matrix
inversion lemma as follows
\begin{equation}
  \label{eq:update_sigma}
  \begin{aligned}
    \hat\Sigmab_\iaia
      &= (\tilde\Omegab_\iaia)^{-1} +
          (\tilde\Omegab_\iaia)^{-1}
            \hat\Omegab_\iaa
            (\hat\Omegab_{aa} - \hat\Omegab_\aia
                    (\tilde\Omegab_\iaia)^{-1}
                    \hat\Omegab_\iaa
            )^{-1}
            \hat\Omegab_\aia
          (\tilde\Omegab_\iaia)^{-1}, \\
    \hat\Sigmab_\aia &=
    - \hat\Omegab_{aa}\hat\Omegab_\aia\hat\Sigmab_\iaia,\\
    \hat\Sigmab_{aa} &=
            (\hat\Omegab_{aa} - \hat\Omegab_\aia
                    (\tilde\Omegab_\iaia)^{-1}
                    \hat\Omegab_\iaa
            )^{-1},
  \end{aligned}
\end{equation}
with $
  (\tilde\Omegab_\iaia)^{-1} =
  \tilde\Sigmab_\iaia
  - \tilde\Sigmab_\iaa
      \tilde\Sigmab_{aa}^{-1}
      \tilde\Sigmab_\aia$.

\section{Complexity Analysis of Multi-attribute Estimation}
\label{appendix:b}

Step 2 of the estimation algorithm updates portions of the precision
and covariance matrices corresponding to one node at a time.  From
$\S$\ref{appendix:a}, we observe that the computational complexity of
updating the precision matrix is $\Ocal\rbr{pk^2}$. Updating the
covariance matrix requires computing $(\tilde \Omegab_\iaia)^{-1}$,
which can be efficiently done in $\Ocal\rbr{p^2k^2 + pk^2 + k^3} =
\Ocal\rbr{p^2k^2}$ operations, assuming that $k \ll p$. With this, the
covariance matrix can be updated in $\Ocal\rbr{p^2k^2}$
operations. Therefore the total cost of updating the covariance and
precision matrices is $\Ocal\rbr{p^2k^2}$ operations. Since step 2
needs to be performed for each node $a \in V$, the total complexity is
$\Ocal\rbr{p^3k^2}$. Let $T$ denote the total number of times step 2
is executed. This leads to the overall complexity of the algorithm as
$\Ocal\rbr{Tp^3k^2}$. In practice, we observe that $T \approx 10
\text{ to } 20$ for sparse graphs. Furthermore, when the whole
solution path is computed, we can use warm starts to further speed up
computation, leading to $T < 5$ for each $\lambda$.

\section{Technical Proofs}
\label{appendix:c}

In this appendix, we collect proofs of the results presented in the
main part of the paper.

\subsection{Proof of Lemma~\ref{lem:convergence}}

We start the proof by giving to technical results needed later.  The
following lemma states that the minimizer of \eqref{eq:max_ll_opt} is
unique and has bounded minimum and maximum eigenvalues, denoted as
$\Lambda_{\min}$ and $\Lambda_{\max}$.
\begin{lemma}
  \label{lem:bound_eigen_value}
  For every value of $\lambda > 0$, the optimization problem in
  Eq.~\eqref{eq:max_ll_opt} has a unique minimizer $\hat\Omegab$,
  which satisfies
  $\Lambda_{\min}(\hat\Omegab) \geq (\Lambda_{\max}(\Sbb) +  \lambda p)^{-1} > 0$
  and
  $\Lambda_{\max}(\hat\Omegab) \leq \lambda^{-1}\sum_{j\in V}k_j$.
\end{lemma}
\begin{proof}
The optimization objective given in \eqref{eq:max_ll_opt} can be
written in the equivalent constrained form as
\begin{equation*}
  \min_{\Omegab \succ \zero}\ \tr \Sbb \Omegab - \log|\Omegab|
  \quad
  \text{subject to}
  \quad
  \sum_{a, b} \norm{\Omegab_{ab}}_F \leq C(\lambda).
\end{equation*}
The procedure involves minimizing a continuous
objective over a compact set, and so by Weierstrass’ theorem, the
minimum is always achieved. Furthermore, the objective is strongly
convex and therefore the minimum is unique.

The solution $\hat\Omegab$ to the optimization problem
\eqref{eq:max_ll_opt} satisfies
\begin{equation}
  \Sbb - \hat\Omegab^{-1} + \lambda\Zb = \zero
\end{equation}
where $\Zb \in \partial \sum_{a,b} \norm{\hat\Omegab_{ab}}_F$ is the
element of the sub-differential and satisfies $\norm{\Zb_{ab}}_F \leq
1$ for all $(a,b) \in V^2$. Therefore,
\begin{equation*}
  \Lambda_{\max}(\hat\Omegab^{-1})
  \leq \Lambda_{\max}(\Sbb) +  \lambda\Lambda_{\max}(\Zb)
  \leq \Lambda_{\max}(\Sbb) +  \lambda p.
\end{equation*}

Next, we prove an upper bound on $\Lambda_{\max}(\hat\Omegab)$.  At
optimum, the primal-dual gap is zero, which gives that
\[
\sum_{a,b}\norm{\hat\Omega_{ab}}_F \leq
\lambda^{-1}(\sum_{j\in V}k_j - \tr\Sbb\hat\Omegab) \leq
\lambda^{-1}\sum_{j\in V}k_j,
\] as $\Sbb \succeq \zero$ and
$\hat\Omegab \succ \zero$.  Since $\Lambda_{\max}(\hat\Omegab) \leq
\sum_{a,b} \norm{\hat\Omegab_{ab}}_F$, the proof is done.  
\end{proof}

The next results states that the objective function has a Lipschitz continuous
gradient, which will be used to show that the generalized gradient
descent can be used to find $\hat\Omegab$.
\begin{lemma}
  \label{lem:lip_grad}
  The function $f(\Ab) = \tr\Sbb\Ab - \log|\Ab|$ has a Lipschitz
  continuous gradient on the set $\{\Ab \in \Scal^p\ :\
  \Lambda_{\min}(\Ab) \geq \gamma \}$, with the Lipschitz constant $L
  = \gamma^{-2}$.
\end{lemma}
\begin{proof}
    We have that $\nabla f(\Ab) = \Sbb-\Ab^{-1}$. Then
  \begin{equation*}
  \begin{aligned}
    \norm{\nabla f(\Ab) - \nabla f(\Ab')}_F
    &= \norm{\Ab^{-1} - (\Ab')^{-1}}_F \\
    &\leq \Lambda_{\max}{\Ab^{-1}} \norm{\Ab - \Ab'}_F \Lambda_{\max}{\Ab^{-1}} \\
    &\leq \gamma^{-2} \norm{\Ab - \Ab'}_F,
  \end{aligned}
  \end{equation*}
which completes the proof.
\end{proof}

Now, we provide the proof of Lemma~\ref{lem:convergence}.

By construction, the sequence of estimates
$(\tilde\Omegab^{(t)})_{t\geq1}$ decrease the objective value and are
positive definite.

To prove the convergence, we first introduce some additional notation.
Let $f(\Omegab) = \tr\Sbb\Omegab - \log|\Omegab|$ and $F(\Omegab) =
f(\Omegab) + \sum_{ab}\norm{\Omegab_{ab}}_F$. For any $L > 0$,
let
\[
 Q_L(\Omegab; \bar \Omegab) :=
   f(\bar \Omegab) + \tr[(\Omegab-\bar\Omegab)\nabla f(\bar \Omegab)]
   + \frac{L}{2} \norm{\Omegab - \bar \Omegab}_F^2 +
   \sum_{ab}\norm{\Omegab_{ab}}_F
\]
be a quadratic approximation of $F(\Omegab)$ at a given point
$\bar\Omegab$, which has a unique minimizer
\[
p_L(\bar\Omegab) := \arg\min_{\Omegab} Q_L(\Omegab; \bar \Omegab).
\]
From Lemma 2.3.~in \citet{beck09fast}, we have that
\begin{equation}
  \label{eq:proof_convergence:1}
  F(\bar\Omegab) - F(p_L(\bar\Omegab)) \geq \frac{L}{2}
  \norm{p_L(\bar\Omegab) - \bar\Omegab}_F^2
\end{equation}
if $F(p_L(\bar\Omegab)) \leq Q_L(p_L(\bar\Omegab);\bar\Omegab)$. Note
that $F(p_L(\bar\Omegab)) \leq Q_L(p_L(\bar\Omegab);\bar\Omegab)$
always holds if $L$ is as large as the Lipschitz constant of $\nabla
F$.

Let $\tilde\Omegab^{(t-1)}$ and $\tilde\Omegab^{(t)}$ denote two
successive iterates obtained by the procedure. Without loss of
generality, we can assume that $\tilde\Omegab^{(t)}$ is obtained by
updating the rows/columns corresponding to the node $a$. From
\eqref{eq:proof_convergence:1}, it follows that
\begin{equation}
  \label{eq:proof_convergence:2}
   \frac{2}{L_k}(F(\tilde\Omegab^{(t-1)}) - F(\tilde\Omegab^{(t)}))
  \geq \norm{\tilde\Omegab^{(t-1)}_{aa} - \tilde\Omegab^{(t)}_{aa} }_F +
     2 \sum_{b\neq a} \norm{\tilde\Omegab^{(t-1)}_{ab} - \tilde\Omegab^{(t)}_{ab} }_F
\end{equation}
where $L_k$ is a current estimate of the Lipschitz constant. Recall
that in our procedure the scalar $t$ serves as a local approximation
of $1/L$. Since eigenvalues of $\hat\Omegab$ are bounded according to
Lemma~\ref{lem:bound_eigen_value}, we can conclude that the
eigenvalues of $\tilde\Omegab^{(t-1)}$ are bounded as well. Therefore
the current Lipschitz constant is bounded away from zero, using
Lemma~\ref{lem:lip_grad}. Combining the results, we observe that the
right hand side of \eqref{eq:proof_convergence:2} converges to
zero as $t\rightarrow\infty$, since the optimization procedure
produces iterates that decrease the objective value. This shows that
$\norm{\tilde\Omegab^{(t-1)}_{aa} - \tilde\Omegab^{(t)}_{aa} }_F + 2
\sum_{b\neq a} \norm{\tilde\Omegab^{(t-1)}_{ab} -
  \tilde\Omegab^{(t)}_{ab} }_F $ converges to zero, for any $a \in
V$. Since $(\tilde\Omegab^{(t)}$ is a bounded sequence, it has a
limit point, which we denote $\hat\Omegab$. It is easy to see, from
the stationary conditions for the optimization problem given in
\eqref{eq:partial_min}, that the limit point $\hat\Omegab$ also
satisfies the global KKT conditions to the optimization problem in
\eqref{eq:max_ll_opt}.

\subsection{Proof of Lemma~\ref{lem:nec_suff_block_diag}}
  Suppose that the solution $\hat\Omegab$ to \eqref{eq:max_ll_opt}
  is block diagonal with blocks $P_1, P_2, \ldots, P_l$. For two
  nodes $a,b$ in different blocks, we have that
  $(\hat\Omegab)^{-1}_{ab} = 0$ as the inverse of the block diagonal
  matrix is block diagonal. From the KKT conditions, it follows that
  $\norm{\Sbb_{ab}}_F \leq \lambda$.

  Now suppose that $\norm{\Sbb_{ab}}_F \leq \lambda$ for all $a\in
  P_j, b\in P_{j'}, j\neq j'$. For every $l'=1,\ldots,l$ construct
  \begin{equation*}
    \tilde\Omegab_{l'} = \arg
    \min_{\Omegab_{l'} \succ \zero}\ \tr \Sbb_{l'} \Omegab_{l'} - \log|\Omegab_{l'}|
    + \lambda \sum_{a, b} \norm{\Omegab_{ab}}_F.
  \end{equation*}
  Then $\hat\Omegab = \diag(\hat\Omegab_1, \hat\Omegab_2, \ldots,
  \hat\Omegab_l)$ is the solution of \eqref{eq:max_ll_opt} as it
  satisfies the KKT conditions.

\subsection{Proof of Eq.~\eqref{eq:comp_cov}}
\label{sec:appendix_eq_3}

First, we note that 
\[
\Var\left( (\Xb_a^T,\Xb_b^T)^T \mid \Xb_{\bar{ab}}\right) =
\Sigmab_{ab,ab} -
\Sigmab_{ab,\bar{ab}}\Sigmab^{-1}_{\bar{ab},\bar{ab}}\Sigmab_{\bar{ab},ab}
\]
is the conditional covariance matrix of $(\Xb_a^T, \Xb_b^T)^T$ given the
remaining nodes $\Xb_{\bar{ab}}$ (see Proposition C.5 in
\citet{lauritzen96graphical}). Define $\bar\Sigmab =
\Sigmab_{ab,ab} -
\Sigmab_{ab,\bar{ab}}\Sigmab^{-1}_{\bar{ab},\bar{ab}}\Sigmab_{\bar{ab},ab}$.
Partial canonical correlation between $\Xb_a$ and $\Xb_b$ is equal to
zero if and only if $\bar\Sigmab_{ab} = \zero$. On the other hand, the
matrix inversion lemma gives that $\Omegab_{ab,ab} = \bar
\Sigmab^{-1}$. Now, $\Omegab_{ab} = \zero$ if and only if
$\bar\Sigmab_{ab}=  0$. This shows the equivalence relationship in
Eq.~\eqref{eq:comp_cov}.

\subsection{Proof of Proposition~\ref{prop:sparsistency}}
\label{sec:appendix_a}

We provide sufficient conditions for consistent network
estimation. Proposition~\ref{prop:sparsistency} given in
$\S$\ref{sec:theory} is then a simple consequence. To provide
sufficient conditions, we extend the work of \citet{Ravikumar11} to
our setting, where we observe multiple attributes for each node. In
particular, we extend their Theorem~1.

For simplicity of presentation, we assume that $k_a = k$, for all $a
\in V$, that is, we assume that the same number of attributes is
observed for each node. Our assumptions involve the Hessian of the
function $f(\Ab) = \tr\Sbb\Ab - \log|\Ab|$ evaluated at the true
$\Omegab^*$, 
\begin{equation}
  \label{eq:fisher}
\Hcal = \Hcal(\Omegabt) =
(\Omegabt)^{-1}\otimes(\Omegabt)^{-1} \in \RR^{(pk)^2\times(pk)^2},
\end{equation}
and the true covariance matrix $\Sigmab^*$. The Hessian and the
covariance matrix can be thought of block matrices with blocks of size
$k^2 \times k^2$ and $k \times k$, respectively. We will make use of
the operator $\Ccal(\cdot)$ that operates on these block matrices and
outputs a smaller matrix with elements that equal to the Frobenius
norm of the original blocks, 
\[
  \left(
    \begin{array}{cccc}
     \Ab_{11} & \Ab_{12} & \cdots & \Ab_{1p} \\
     \Ab_{21} & \Ab_{22} & \cdots & \Ab_{2p} \\
     \vdots &  & \ddots & \vdots \\
     \Ab_{p1} & \cdots &  & \Ab_{pp} \\
    \end{array}
  \right)
  \quad
  \xrightarrow{\ \ \ \Ccal(\cdot)\ \ \ }
  \quad
  \left(
    \begin{array}{cccc}
     \norm{\Ab_{11}}_F & \norm{\Ab_{12}}_F & \cdots & \norm{\Ab_{1p}}_F \\
     \norm{\Ab_{21}}_F & \norm{\Ab_{22}}_F & \cdots & \norm{\Ab_{2p}}_F \\
     \vdots &  & \ddots & \vdots \\
     \norm{\Ab_{p1}}_F & \cdots &  & \norm{\Ab_{pp}}_F \\
    \end{array}
  \right)  
\]
In particular, $\Ccal(\Sigmab^*) \in \RR^{p\times p}$ and
$\Ccal(\Hcal) \in \RR^{p^2 \times p^2}$.

We denote the index set of the non-zero blocks of the precision matrix as
\[
\Tcal := \{(a,b) \in V \times V \ :\ \norm{\Omegabt_{ab}}_2 \neq
0 \} \cup \{(a,a)\ :\ a \in V \}
\] and let $\Ncal$ denote its complement in $V \times V$, that is,
\[
\Ncal = \{(a,b)\ :\ \norm{\Omegab_{ab}}_F = 0\}.
\]

As mentioned earlier, we need to make an assumption on the Hessian
matrix, which takes the standard { irrepresentable}-like form.  There
exists a constant $\alpha \in [0, 1)$ such that
\begin{equation}
\label{eq:assum:irrepresentable:1}
\opnorm{
\Ccal\left(
  \Hcal_{\Ncal\Tcal}(\Hcal_{\Tcal\Tcal})^{-1}
\right)}{\infty} \leq 1 - \alpha.
\end{equation}
These condition extends the irrepresentable condition given in
\citet{Ravikumar11}, which was needed for estimation of networks from
single attribute observations. It is worth noting, that the condition
given in Eq.~\eqref{eq:assum:irrepresentable:1} can be much weaker
than the irrepresentable condition of \citet{Ravikumar11} applied
directly to the full Hessian matrix. This can be observed in
simulations done in $\S$\ref{sec:simulation}, where a chain network is
not consistently estimated even with a large number of samples.

We will also need the following two quantities to specify the results
\begin{equation}
  \label{eq:kappa_sigma}
  \kappa_{\Sigmab^*} = \opnorm{\Ccal(\Sigmab^*)}{\infty}
\end{equation}
and
\begin{equation}
  \label{eq:kappa_fisher}
  \kappa_{\Hcal} = \opnorm{\Ccal(\Hcal_{\Tcal\Tcal}^{-1})}{\infty}.
\end{equation}

Finally, the results are going to depend on the tail bounds for the
elements of the matrix $\Ccal(\Sbb - \Sigmab^*)$. We will assume that
there is a constant $v_* \in (0, \infty]$ and a function $f: \NN
\times (0, \infty) \mapsto (0,\infty)$ such that for any $(a,b)\in V
\times V$
\begin{equation}
  \label{eq:tail_func}
  \PP\rbr{\Ccal(\Sbb - \Sigmab^*)_{ab} \geq \delta} \leq
  \frac{1}{f(n,\delta)}
  \qquad \delta \in (0, v_*^{-1}].
\end{equation}
The function $f(n,\delta)$ will be monotonically increasing in both
$n$ and $\delta$. Therefore, we define the following two inverse
functions
\begin{equation}
  \label{eq:tail_inf:1}
  \bar n_f(\delta; r) = \arg\max\{n\ :\ f(n, \delta) \leq r\}
\end{equation}
and
\begin{equation}
  \label{eq:tail_inf:2}
  \bar \delta_f(r; n) = \arg\max\{\delta\ :\ f(n, \delta) \leq r\}
\end{equation}
for $r \in [1, \infty)$.

With the notation introduced, we have the following result.
\begin{theorem}
  \label{thm:linf}
  Assume that the irrepresentable condition in
  Eq.~\eqref{eq:assum:irrepresentable:1} is satisfied and that
  there exists a constant $v_* \in (0, \infty]$ and a function
  $f(n,\delta)$ so that Eq.~\eqref{eq:tail_func} is satisfied for any
  $(a,b) \in V \times V$. Let
  \begin{equation*}
    \lambda = \frac{8}{\alpha}\bar\delta_f(n, p^\tau)
  \end{equation*}
  for some $\tau > 2$. If
  \begin{equation}
    \label{eq:sample_size_general}
    n > \bar n_f\left(\frac{1}{\max(v_*,
         6(1+8\alpha^{-1})s\max(\kappa_{\Sigmab^*}\kappa_{\Hcal},
                                \kappa_{\Sigmab^*}^3\kappa_{\Hcal}^2))
       }, p^\tau\right)
  \end{equation}
  then
  \begin{equation}
    \label{eq:linf_bound}
    \norm{\Ccal(\hat\Omegab-\Omegab)}_\infty
      \leq 2(1+8\alpha^{-1})\kappa_\Hcal
           \bar\delta_f(n, p^\tau)
  \end{equation}
  with probability at least $1-p^{2-\tau}$.
\end{theorem}
Theorem~\ref{thm:linf} is of the same form as Theorem~1 in
\citet{Ravikumar11}, but the $\ell_\infty$ element-wise convergence is
established for $\Ccal(\hat\Omegab - \Omegab)$, which will guarantee
successful recovery of non-zero partial canonical correlations if the
blocks of the true precision matrix are sufficiently large.

Theorem~\ref{thm:linf} is proven as Theorem~1 in
\citet{Ravikumar11}. We provide technical results in
Lemma~\ref{lem:dual_feasibility}, Lemma~\ref{lem:bound_remainder} and
Lemma~\ref{lem:bound_delta}, which can be used to substitute results
of Lemma~4, Lemma~5 and Lemma~6 in \citet{Ravikumar11} under our
setting. The rest of the arguments then go through. Below we provide
some more details.

First, let $\Zcal : \RR^{ pk \times pk} \mapsto \RR^{ pk \times pk}$
be the mapping defined as
\begin{equation}
    \label{eq:subgradient_operator}
    \Zcal(\Ab)_{ab} = \left\{
    \begin{array}{ll}
      \frac{\Ab_{ab}}{\norm{\Ab_{ab}}_F} &
      \text{if } \norm{\Ab_{ab}}_F \neq 0, \\
      \Zb \text{ with } \norm{\Zb}_F \leq 1 &
      \text{if } \norm{\Ab_{ab}}_F = 0,
    \end{array}
    \right.
\end{equation}
Next, define the function
\begin{equation}
  \label{eq:def_G}
  G(\Omegab) = \tr \Omegab \Sbb - \log|\Omegab| +
  \lambda \norm{\Ccal(\Omegab)}_1,
\quad \forall \Omegab \succ 0
\end{equation}
and the following system of equations
\begin{equation} \label{eq:kkt_system}
\left\{
\begin{array}{rcll}
\Sbb_{ab} - (\Omegab^{-1})_{ab} &=& -\lambda
  \Zcal(\Omegab)_{ab},
    & \text{if } \Omegab_{ab} \neq 0 \\
\norm{\Sbb_{ab} - (\Omegab^{-1})_{ab}}_F &\leq& \lambda,
& \text{if } \Omegab_{ab} = 0. \\
\end{array}
\right.
\end{equation}
It is known that $\Omegab \in \mathbb{R}^{\tilde p \times \tilde p}$
is the minimizer of optimization problem in Eq.~\eqref{eq:max_ll_opt}
if and only if it satisfies the system of equations given in
Eq.~\eqref{eq:kkt_system}. We have already shown in
Lemma~\ref{lem:bound_eigen_value} that the minimizer is unique.

Let $\tilde \Omegab$ be the solution to the following constrained
optimization problem
\begin{equation}
  \label{eq:const_max_ll_opt}
  \min_{\Omegab \succ \zero}\
  \tr \Sbb \Omegab - \log|\Omegab| + \lambda \norm{\Ccal(\Omegab)}_1
  \text{ subject to } \Ccal(\Omegab)_{ab} = 0,\ \forall (a,b) \in \Ncal.
\end{equation}
Observe that one cannot find $\tilde \Omegab$ in practice, as it
depends on the unknown set $\Ncal$. However, it is a useful
construction in the proof. We will prove that $\tilde \Omegab$ is
solution to the optimization problem given in
Eq.~\eqref{eq:max_ll_opt}, that is, we will show that $\tilde\Omegab$
satisfies the system of equations \eqref{eq:kkt_system}.

Using the first-order Taylor expansion we have that
\begin{equation} \label{eq:linearization_log_barier}
\tilde\Omegab^{-1} = (\Omegabt)^{-1} - (\Omegabt)^{-1}
\Deltab(\Omegabt)^{-1} + R(\Deltab),
\end{equation}
where $\Deltab = \Omegab - \Omegabt$ and $R(\Deltab)$ denotes the
remainder term. With this, we state and prove
Lemma~\ref{lem:dual_feasibility}, Lemma~\ref{lem:bound_remainder} and
Lemma~\ref{lem:bound_delta}. They can be combined as in
\citet{Ravikumar11} to complete the proof of Theorem~\ref{thm:linf}. 

\begin{lemma}
  \label{lem:dual_feasibility}
  Assume that
  \begin{equation}
    \label{eq:assum:dual_feasibility}
    \max_{ab}\norm{\Deltab_{ab}}_F \leq \frac{\alpha\lambda}{8}
    \quad\text{ and }\quad
    \max_{ab}\norm{\Sigmab^*_{ab} - \Sbb_{ab}}_F \leq \frac{\alpha\lambda}{8}.
  \end{equation}
  Then $\tilde \Omegab$ is the solution to the optimization problem in
  Eq.~\eqref{eq:max_ll_opt}.
\end{lemma}
\begin{proof}
  We use $\Rb$ to denote $\Rb(\Delta)$. Recall that $\Delta_\Ncal = 0$
  by construction.  Using \eqref{eq:linearization_log_barier} we can
  rewrite \eqref{eq:kkt_system} as
  \begin{align}
    \Hcal_{ab,\Tcal}\bar \Deltab_\Tcal - \bar \Rb_{ab} + \bar\Sbb_{ab} -
    \bar\Sigmab_{ab}^* + \lambda \bar \Zcal(\tilde\Omegab)_{ab} & = 0
    & \text{if } (a,b) \in \Tcal \label{eq:kkt_system:1} \\
    \norm{\Hcal_{ab,\Tcal}\bar\Deltab_\Tcal - \bar \Rb_{ab} + \bar\Sbb_{ab} -
    \bar\Sigmab_{ab}^*}_2 & \leq \lambda
    & \text{if } (a,b) \in \Ncal \label{eq:kkt_system:2}.
  \end{align}
  By construction, the solution $\tilde\Omegab$ satisfy
  \eqref{eq:kkt_system:1}. Under the assumptions, we show that
  \eqref{eq:kkt_system:2} is also satisfied with inequality.

  From \eqref{eq:kkt_system:1}, we can solve for $\Delta_\Tcal$,
  \begin{equation*}
    \Delta_{\Tcal} = \Hcal_{\Tcal, \Tcal}^{-1}
     [\bar\Rb_\Tcal - \bar\Sigmab_\Tcal + \bar\Sbb_\Tcal
      - \lambda\bar\Zcal(\tilde\Omegab)_\Tcal].
  \end{equation*}
  Then
  \begin{equation*}
  \begin{aligned}
  &\norm{\Hcal_{ab,\Tcal}\Hcal_{\Tcal, \Tcal}^{-1}
     [\bar\Rb_\Tcal - \bar\Sigmab_\Tcal + \bar\Sbb_\Tcal
      - \lambda\bar\Zcal(\tilde\Omegab)_\Tcal]
      - \bar \Rb_{ab} + \bar\Sbb_{ab} -
        \bar\Sigmab_{ab}^*}_2 \\
  & \qquad \leq
  \lambda
    \norm{\Hcal_{ab,\Tcal}\Hcal_{\Tcal, \Tcal}^{-1}\bar\Zcal(\tilde\Omegab)_\Tcal}_2
  +
  \norm{\Hcal_{ab,\Tcal}\Hcal_{\Tcal, \Tcal}^{-1}
     [\bar\Rb_\Tcal - \bar\Sigmab_\Tcal + \bar\Sbb_\Tcal]}_2
  + \norm{\bar \Rb_{ab} + \bar\Sbb_{ab} - \bar\Sigmab_{ab}^*}_2 \\
  & \qquad \leq
  \lambda(1-\alpha) + (2-\alpha)\frac{\alpha\lambda}{4} \\
  & \qquad < \lambda
  \end{aligned}
  \end{equation*}
  using assumption on $\Hcal$ in \eqref{eq:assum:irrepresentable:1} and
  \eqref{eq:assum:dual_feasibility}. This shows that $\tilde \Omegab$
  satisfies \eqref{eq:kkt_system}.
\end{proof}

\begin{lemma}
  \label{lem:bound_remainder}
  Assume that
  \begin{equation}
    \label{eq:assum:remainder}
    \norm{\Ccal(\Deltab)}_\infty \leq \frac{1}{3\kappa_{\Sigmab^*}s}.
  \end{equation}
  Then
  \begin{equation}
    \label{eq:bound_remainder}
    \norm{\Ccal(\Rb(\Deltab))}_\infty \leq
      \frac{3s}{2}\kappa_{\Sigmab^*}^3\norm{\Ccal(\Deltab)}_\infty^2.
  \end{equation}
\end{lemma}
\begin{proof}
  Remainder term can be written as
  \begin{equation*}
    \Rb(\Deltab) =
       (\Omegabt + \Deltab)^{-1} - (\Omegabt)^{-1} + (\Omegabt)^{-1}\Deltab(\Omegabt)^{-1}.
  \end{equation*}
  Using \eqref{eq:submultiplicity_fro_linf_opnorm}, we have that
  \begin{equation*}
  \begin{aligned}
    \opnorm{\Ccal((\Omegabt)^{-1}\Deltab)}{\infty} & \leq
      \opnorm{\Ccal((\Omegabt)^{-1})}{\infty}
      \opnorm{\Ccal(\Deltab)}{\infty} \\
    & \leq
      s\opnorm{\Ccal((\Omegabt)^{-1})}{\infty}
      \norm{\Ccal(\Deltab)}_{\infty} \\
   & \leq \frac{1}{3}
  \end{aligned}
  \end{equation*}
  which gives us the following expansion
  \begin{equation*}
    (\Omegabt + \Deltab)^{-1} =
    (\Omegabt)^{-1} - (\Omegabt)^{-1}\Deltab(\Omegabt)^{-1} +
    (\Omegabt)^{-1}\Deltab(\Omegabt)^{-1}\Deltab\Jb(\Omegabt)^{-1},
  \end{equation*}
  with $\Jb = \sum_{k \geq 0}(-1)^k((\Omegabt)^{-1}\Deltab)^k$. Using
  \eqref{eq:max_fro_norm} and
  \eqref{eq:submultiplicity_fro_linf_opnorm}, we have that
  \begin{equation*}
  \begin{aligned}
      \norm{\Ccal(\Rb)}_\infty &\leq
       \norm{\Ccal((\Omegabt)^{-1}\Deltab)}_\infty
       \opnorm{\Ccal((\Omegabt)^{-1}\Deltab\Jb(\Omegabt)^{-1})^T}{\infty}
       \\
    &\leq
      \opnorm{\Ccal((\Omegabt)^{-1})}{\infty}^3\norm{\Ccal(\Deltab)}_\infty
      \opnorm{\Ccal(\Jb^T)}{\infty}
      \opnorm{\Ccal(\Deltab)}{\infty} \\
    & \leq
      s\opnorm{\Ccal((\Omegabt)^{-1})}{\infty}^3\norm{\Ccal(\Deltab)}_\infty^2
      \opnorm{\Ccal(\Jb^T)}{\infty}.
  \end{aligned}
  \end{equation*}
  Next, we have that
  \begin{equation*}
  \begin{aligned}
    \opnorm{\Ccal(\Jb^T)}{\infty} & \leq
    \sum_{k > 0} \opnorm{\Ccal(\Deltab(\Omegabt)^{-1})}{\infty}^k \\
    &\leq \frac{1}{1-\opnorm{\Ccal(\Deltab(\Omegabt)^{-1})}{\infty}} \\
    &\leq \frac{3}{2},
  \end{aligned}
  \end{equation*}
  which gives us
  \begin{equation*}
    \norm{\Ccal(\Rb)}_\infty \leq
      \frac{3s}{2}\kappa_{\Sigmab^*}^3\norm{\Ccal(\Deltab)}_\infty^2
  \end{equation*}
  as claimed.
\end{proof}

\begin{lemma}
\label{lem:bound_delta}
Assume that
\begin{equation}
  \label{eq:assum:bound_delta}
  r := 2\kappa_\Hcal(\norm{\Ccal(\Sbb - \Sigmab^*)}_\infty + \lambda)
    \leq \min\left( \frac{1}{3\kappa_{\Sigmab^*}s},
               \frac{1}{3\kappa_{\Hcal}\kappa_{\Sigmab^*}^3s} \right).
\end{equation}
Then
\begin{equation}
  \label{eq:linf_bound_restrict}
  \norm{\Ccal(\Deltab)}_\infty \leq r.
\end{equation}
\end{lemma}
\begin{proof}
  The proof follows the proof of Lemma 6 in \citet{Ravikumar11}.
  Define the ball
  \begin{equation*}
    \Bcal(r) := \{ \Ab\ :\ \Ccal(\Ab)_{ab} \leq r, \forall (a,b) \in \Tcal \},
  \end{equation*}
  the gradient mapping
  \begin{equation*}
    G(\Omegab_\Tcal) = -(\Omegab^{-1})_\Tcal + \Sbb_{\Tcal} + \lambda \Zcal(\Omegab)_{\Tcal}
  \end{equation*}
  and
  \begin{equation*}
    F(\bar\Deltab_\Tcal) =
     -\Hcal_{\Tcal\Tcal}^{-1}\bar
     G(\Omegab^{*}_{\Tcal}+\Deltab_{\Tcal})
     + \bar\Deltab_\Tcal.
  \end{equation*}
  We need to show that $F(\Bcal(r))\subseteq\Bcal(r)$, which implies
  that $\norm{\Ccal(\Deltab_\Tcal)}_\infty \leq r$.

  Under the assumptions of the lemma, for any $\Deltab_S \in \Bcal(r)$,
  we have the following decomposition
  \begin{equation*}
    F(\bar\Deltab_{\Tcal}) =
    \Hcal_{\Tcal\Tcal}^{-1}\bar\Rb(\Deltab)_\Tcal
    + \Hcal_{\Tcal\Tcal}^{-1}(\bar\Sbb_{\Tcal}-\bar\Sigmab_\Tcal^* +
    \lambda \bar\Zcal(\Omegab^*+\Deltab)_{\Tcal}).
  \end{equation*}
  Using Lemma~\ref{lem:bound_remainder}, the first term can be bounded
  as
  \begin{equation*}
  \begin{aligned}
    \norm{\Ccal(\Hcal_{\Tcal\Tcal}^{-1}\bar\Rb(\Deltab)_\Tcal)}_\infty
    & \leq \opnorm{\Ccal(\Hcal_{\Tcal\Tcal}^{-1})}{\infty}
            \norm{\Ccal(\Rb(\Deltab)}_\infty \\
    & \leq \frac{3s}{2}\kappa_{\Hcal}\kappa_{\Sigmab^*}^3\norm{\Ccal(\Deltab)}_\infty^2 \\
    & \leq \frac{3s}{2}\kappa_{\Hcal}\kappa_{\Sigmab^*}^3r^2 \\
    & \leq r/2
  \end{aligned}
  \end{equation*}
  where the last inequality follows under the assumptions. Similarly
  \begin{equation*}
  \begin{aligned}
    &\norm{\Ccal(\Hcal_{\Tcal\Tcal}^{-1}(\bar\Sbb_{\Tcal}-\bar\Sigmab_\Tcal^* +
    \lambda \bar\Zcal(\Omegab^*+\Deltab)_{\Tcal})}_\infty \\
    & \qquad\leq \opnorm{\Ccal(\Hcal_{\Tcal\Tcal}^{-1})}{\infty}
           (\norm{\Ccal(\Sbb-\Sigmab^*)}_\infty +
            \lambda\norm{\Ccal(\Zcal(\Omegab^*+\Deltab))}_\infty) \\
    & \qquad\leq \kappa_{\Hcal}(\norm{\Ccal(\Sbb-\bar\Sigmab^*)}_\infty +
    \lambda) \\
    & \qquad\leq r/2.
  \end{aligned}
  \end{equation*}
  This shows that $F(\Bcal(r))\subseteq\Bcal(r)$.
\end{proof}

The following result is a corollary of Theorem~\ref{thm:linf}, which
shows that the graph structure can be estimated consistently under
some assumptions.
\begin{corollary}
  Assume that the conditions of Theorem~\ref{thm:linf} are satisfied. 
  Furthermore, suppose that
  \begin{equation*}
    \min_{(a,b)\in\Tcal,\ a\neq b}
    \norm{\Omegab}_F > 
       2(1+8\alpha^{-1})\kappa_\Hcal
           \bar\delta_f(n, p^\tau)
  \end{equation*}
  then Algorithm~1 estimates a graph $\hat G$ which satisfies 
  \[
     \PP\rbr{\hat G \neq G} \geq 1-p^{2-\tau}.
  \]  
\end{corollary}

Next, we specialize the result of Theorem~\ref{thm:linf} to a case
where $\Xb$ has sub-Gaussian tails. That is, the random vector $\Xb =
(X_1, \ldots, X_{pk})^T$ is zero-mean with covariance $\Sigmab^*$. Each
$(\sigma_{aa}^*)^{-1/2}X_a$ is sub-Gaussian with parameter $\gamma$.

\begin{proposition}
\label{prop:inf_conv_subgauss}
Set the penalty parameter in $\lambda$ in Eq.~\eqref{eq:max_ll_opt} as 
\[
\lambda = 8k\alpha^{-1}
\rbr{128(1+4\gamma^2)^2(\max_a(\sigma_{aa}^*)^2)n^{-1}(2\log(2k) +
  \tau\log(p))}^{1/2}.
\]
If
\begin{equation*}
  n > C_1s^2k^2(1+8\alpha^{-1})^2(\tau\log p + \log 4 + 2\log k)
\end{equation*}
where $C_1 = (48\sqrt{2}(1+4\gamma^2)(\max_a \sigma_{aa}^*)
                \max(\kappa_{\Sigmab^*}\kappa_{\Hcal},
                \kappa_{\Sigmab^*}^3\kappa_{\Hcal}^2))^2$ 
then
\begin{equation*}
    \norm{\Ccal(\hat\Omegab-\Omegab)}_\infty
      \leq 16\sqrt{2}(1+4\gamma^2)\max_i \sigma_{ii}^*
          (1+8\alpha^{-1})\kappa_\Hcal k
          \rbr{\frac{\tau\log p + \log 4 + 2\log k}{n}}^{1/2}
\end{equation*}
with probability $1 - p^{2-\tau}$.
\end{proposition}
The proof simply follows by observing that, for any $(a,b)$,
\begin{equation}
\begin{aligned}
\PP\rbr{\Ccal(\Sbb - \Sigmab^*)_{ab} > \delta}
&\leq \PP\rbr{\max_{(c,d) \in (a,b)} (\sigma_{cd} - \sigma_{cd}^*)^2 > \delta^2/k^2}\\
&\leq k^2\PP\rbr{|\sigma_{cd} - \sigma_{cd}^*| > \delta/k}\\
&\leq 4k^2\exp\left(-\frac{n\delta^2}{c_*k^2}\right)
\end{aligned}
\end{equation}
for all $\delta\in(0, 8(1+4\gamma^2)(\max_a \sigma_{aa}^*))$ with
$c_* = 128(1+4\gamma^2)^2(\max_a(\sigma_{aa}^*)^2)$.
Therefore,
\begin{align*}
  f(n,\delta) &= \frac{1}{4k^2}\exp(c_*\frac{n\delta^2}{k^2})\\
  \bar n_f(\delta;r) &= \frac{k^2\log(4k^2r)}{c_*\delta^2} \\
  \bar \delta_f(r;n) &= \rbr{ \frac{k^2\log(4k^2r)}{c_*n} }^{1/2}.
\end{align*}
Theorem~\ref{thm:linf} and some simple algebra complete the proof.

Proposition~\ref{prop:sparsistency} is a simple conseqeuence of
Proposition~\ref{prop:inf_conv_subgauss}.

\subsection{Some Results on Norms of Block Matrices}

Let $\Tcal$ be a partition of $V$. Throughout this section, we
assume that matrices $\Ab, \Bb \in \RR^{p \times p}$ and a vector $\bb
\in \RR^{p}$ are partitioned into blocks according to $\Tcal$.
\begin{lemma}
  \begin{equation}
    \label{eq:martix_fro_linf_opnorm}
    \max_{a \in \Tcal} \norm{\Ab_{a\cdot}\bb}_2
      \leq
      \max_{a\in\Tcal}\sum_{b \in \Tcal}\norm{\Ab_{ab}}_F
      \max_{c\in\Tcal}\norm{\bb_c}_2.
  \end{equation}
\end{lemma}
\begin{proof}
  For any $a \in \Tcal$,
  \begin{equation*}
    \begin{aligned}
      \norm{\Ab_{a\cdot}\bb}_2
      & \leq \sum_{b \in \Tcal} \norm{\Ab_{ab}\bb_b}_2 \\
      & = \sum_{b \in \Tcal} \rbr{ \sum_{i\in a} (\Ab_{ib}\bb_b)^2 }^{1/2}  \\
      & \leq \sum_{b \in \Tcal} \rbr{ \sum_{i\in a} \norm{\Ab_{ib}}_2^2\norm{\bb_b}_2^2 }^{1/2}  \\
      & \leq \sum_{b \in \Tcal} \rbr{ \sum_{i\in a} \norm{\Ab_{ib}}_2^2 }^{1/2}
             \max_{c \in \Tcal} \norm{\bb_c}_2 \\
      & = \sum_{b \in \Tcal} \norm{\Ab_{ab}}_F \max_{c \in \Tcal} \norm{\bb_c}_2.
    \end{aligned}
  \end{equation*}
\end{proof}

\begin{lemma}
  \label{lem:submultiplicity_fro_linf_opnorm}
  \begin{equation}
    \label{eq:submultiplicity_fro_linf_opnorm}
    \opnorm{\Ccal(\Ab\Bb)}{\infty} \leq
      \opnorm{\Ccal(\Bb)}{\infty}
      \opnorm{\Ccal(\Ab)}{\infty}.
  \end{equation}
\end{lemma}
\begin{proof}
  Let $\Cb = \Ab\Bb$ and let $\Tcal$ be a partition of $V$.
  \begin{equation*}
  \begin{aligned}
    \opnorm{\Ccal(\Ab\Bb)}{\infty}
    &= \max_{a\in\Tcal} \sum_{b\in\Tcal}\norm{\Cb_{ab}}_F \\
    &\leq \max_{a \in \Tcal} \sum_b\sum_c\norm{\Ab_{ac}}_F\norm{\Bb_{cb}}_F \\
    &\leq \{\max_{a \in \Tcal} \sum_c\norm{\Ab_{ac}}_F\}
          \{\max_{c \in \Tcal} \sum_b\norm{\Bb_{cb}}_F\} \\
    &=       \opnorm{\Ccal(\Ab)}{\infty} \opnorm{\Ccal(\Bb)}{\infty}.
  \end{aligned}
  \end{equation*}
\end{proof}

\begin{lemma}
\begin{equation}
  \label{eq:max_fro_norm}
  \norm{\Ccal(\Ab\Bb)}_\infty
    \leq \norm{\Ccal(\Ab)}_\infty \opnorm{\Ccal(\Bb)^T}{\infty}.
\end{equation}
\end{lemma}
\begin{proof}
  For a fixed $a$ and $b$,
  \begin{equation*}
  \begin{aligned}
   \Ccal(\Ab\Bb)_{ab} & = \norm{\sum_c \Ab_{ac}\Bb_{cb}}_F \\
   & \leq \sum_c \norm{\Ab_{ac}}_F\norm{\Bb_cb}_F \\
   & \leq \max_{c} \norm{\Ab_{ac}} \sum_c \norm{\Bb_{cb}}_F.
  \end{aligned}
  \end{equation*}
  Maximizing over $a$ and $b$ gives the result.
\end{proof}

\section{Additional Information About Functional Brain Networks}
\label{appendix:d}

Table~\ref{tab:avoi} contains list of the names of the brain
regions. The number before each region is used to index the node in
the connectivity models. Figures~\ref{fig:normal_adj},
\ref{fig:mci_adj} and \ref{fig:ad_adj} contain adjacency
matrices for the estimated graph structures.

\begin{figure}[h]
  \centering
  \includegraphics[width=0.65\textwidth]{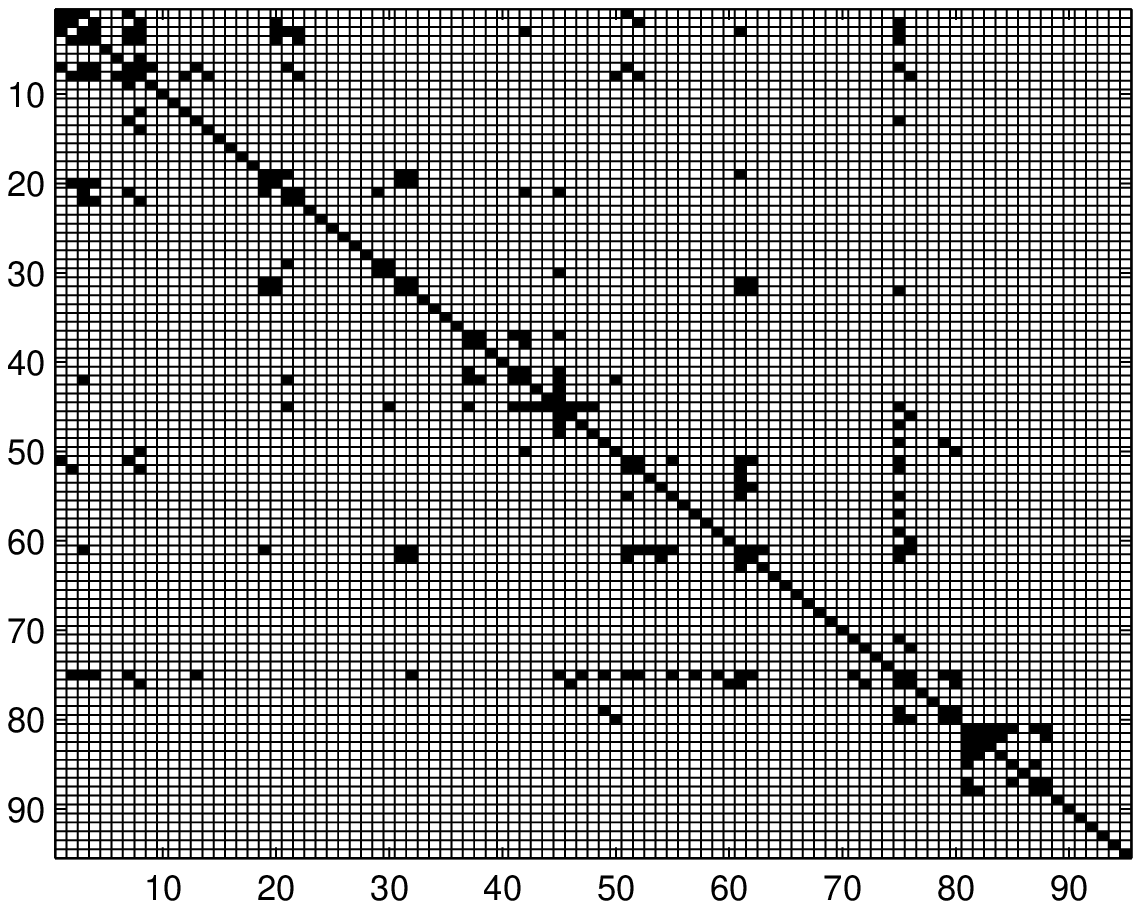}  
  \caption{Adjacency matrix for the brain connectivity network:
        healthy subjects}
  \label{fig:normal_adj}
\end{figure}

\begin{figure}[h]
  \centering
  \includegraphics[width=0.65\textwidth]{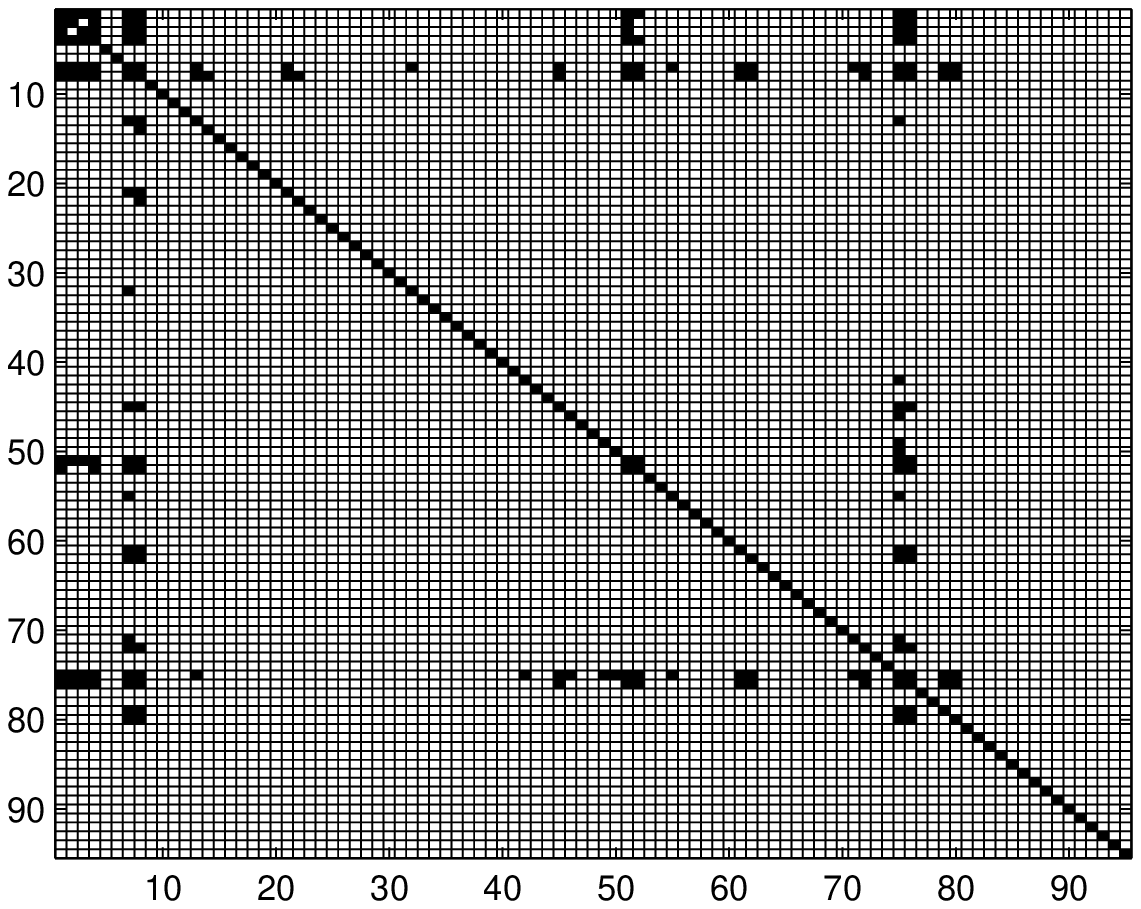}  
  \caption{Adjacency matrix for the brain connectivity network:
        Mild Cognitive Impairment}
  \label{fig:mci_adj}
\end{figure}

\begin{figure}[h]
  \centering
  \includegraphics[width=0.65\textwidth]{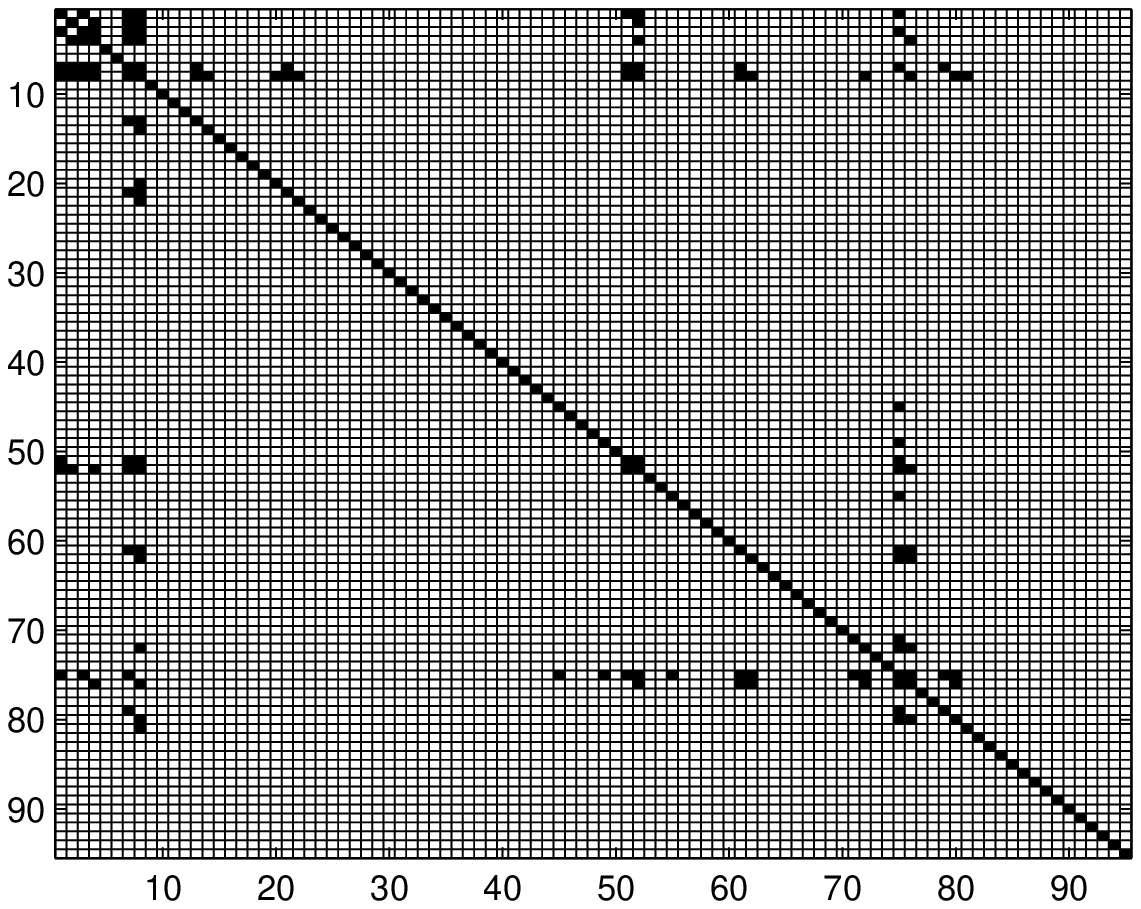}  
  \caption{Adjacency matrix for the brain connectivity network:
        Alzheimer's \& Dementia}
  \label{fig:ad_adj}
\end{figure}

\begin{table}[p]

 \caption{Names of the brain regions. “L” means that the brain region 
is located at the left hemisphere; “R” means right hemisphere.}
{\footnotesize
 \begin{tabular}{cc@{\hskip 2cm}cc}
$1$ & \verb|Precentral_L| & $49$  & \verb|Fusiform_L| \\ 
$2$ & \verb|Precentral_R| & $50$  & \verb|Fusiform_R| \\ 
$3$ & \verb|Frontal_Sup_L| & $51$  & \verb|Postcentral_L| \\ 
$4$ & \verb|Frontal_Sup_R| & $52$  & \verb|Postcentral_R| \\ 
$5$ & \verb|Frontal_Sup_Orb_L| & $53$  & \verb|Parietal_Sup_L| \\ 
$6$ & \verb|Frontal_Sup_Orb_R| & $54$  & \verb|Parietal_Sup_R| \\ 
$7$ & \verb|Frontal_Mid_L| & $55$  & \verb|Parietal_Inf_L| \\ 
$8$ & \verb|Frontal_Mid_R| & $56$  & \verb|Parietal_Inf_R| \\ 
$9$ & \verb|Frontal_Mid_Orb_L| & $57$  & \verb|SupraMarginal_L| \\ 
$10$ & \verb|Frontal_Mid_Orb_R| & $58$  & \verb|SupraMarginal_R| \\ 
$11$ & \verb|Frontal_Inf_Oper_L| & $59$  & \verb|Angular_L| \\ 
$12$ & \verb|Frontal_Inf_Oper_R| & $60$  & \verb|Angular_R| \\ 
$13$ & \verb|Frontal_Inf_Tri_L| & $61$  & \verb|Precuneus_L| \\ 
$14$ & \verb|Frontal_Inf_Tri_R| & $62$  & \verb|Precuneus_R| \\ 
$15$ & \verb|Frontal_Inf_Orb_L| & $63$  & \verb|Paracentral_Lobule_L| \\ 
$16$ & \verb|Frontal_Inf_Orb_R| & $64$  & \verb|Paracentral_Lobule_R| \\ 
$17$ & \verb|Rolandic_Oper_L| & $65$  & \verb|Caudate_L| \\ 
$18$ & \verb|Rolandic_Oper_R| & $66$  & \verb|Caudate_R| \\ 
$19$ & \verb|Supp_Motor_Area_L| & $67$  & \verb|Putamen_L| \\ 
$20$ & \verb|Supp_Motor_Area_R| & $68$  & \verb|Putamen_R| \\ 
$21$ & \verb|Frontal_Sup_Medial_L| & $69$  & \verb|Thalamus_L| \\ 
$22$ & \verb|Frontal_Sup_Medial_R| & $70$  & \verb|Thalamus_R| \\ 
$23$ & \verb|Frontal_Med_Orb_L| & $71$  & \verb|Temporal_Sup_L| \\ 
$24$ & \verb|Frontal_Med_Orb_R| & $72$  & \verb|Temporal_Sup_R| \\ 
$25$ & \verb|Rectus_L| & $73$  & \verb|Temporal_Pole_Sup_L| \\ 
$26$ & \verb|Rectus_R| & $74$  & \verb|Temporal_Pole_Sup_R| \\ 
$27$ & \verb|Insula_L| & $75$  & \verb|Temporal_Mid_L| \\ 
$28$ & \verb|Insula_R| & $76$  & \verb|Temporal_Mid_R| \\ 
$29$ & \verb|Cingulum_Ant_L| & $77$  & \verb|Temporal_Pole_Mid_L| \\ 
$30$ & \verb|Cingulum_Ant_R| & $78$  & \verb|Temporal_Pole_Mid_R| \\ 
$31$ & \verb|Cingulum_Mid_L| & $79$  & \verb|Temporal_Inf_L| \\ 
$32$ & \verb|Cingulum_Mid_R| & $80$  & \verb|Temporal_Inf_R| \\ 
$33$ & \verb|Hippocampus_L| & $81$  & \verb|Cerebelum_Crus1_L| \\ 
$34$ & \verb|Hippocampus_R| & $82$  & \verb|Cerebelum_Crus1_R| \\ 
$35$ & \verb|ParaHippocampal_L| & $83$  & \verb|Cerebelum_Crus2_L| \\ 
$36$ & \verb|ParaHippocampal_R| & $84$  & \verb|Cerebelum_Crus2_R| \\ 
$37$ & \verb|Calcarine_L| & $85$  & \verb|Cerebelum_4_5_L| \\ 
$38$ & \verb|Calcarine_R| & $86$  & \verb|Cerebelum_4_5_R| \\ 
$39$ & \verb|Cuneus_L| & $87$  & \verb|Cerebelum_6_L| \\ 
$40$ & \verb|Cuneus_R| & $88$  & \verb|Cerebelum_6_R| \\ 
$41$ & \verb|Lingual_L| & $89$  & \verb|Cerebelum_7b_L| \\ 
$42$ & \verb|Lingual_R| & $90$  & \verb|Cerebelum_7b_R| \\ 
$43$ & \verb|Occipital_Sup_L| & $91$  & \verb|Cerebelum_8_L| \\ 
$44$ & \verb|Occipital_Sup_R| & $92$  & \verb|Cerebelum_8_R| \\ 
$45$ & \verb|Occipital_Mid_L| & $93$  & \verb|Cerebelum_9_L| \\ 
$46$ & \verb|Occipital_Mid_R| & $94$  & \verb|Cerebelum_9_R| \\ 
$47$ & \verb|Occipital_Inf_L| & $95$  & \verb|Vermis_4_5| \\ 
$48$ & \verb|Occipital_Inf_R| & & \\ 

 \end{tabular}
}
\label{tab:avoi}
\end{table}

\end{document}